\documentclass[11pt]{article}
\PassOptionsToPackage{dvipsnames}{xcolor}

\usepackage[letterpaper,margin=1in]{geometry}
\usepackage[T1]{fontenc}
\usepackage[utf8]{inputenc}
\usepackage{lmodern}        %
\usepackage{amsmath,amsthm,amssymb,amsfonts,mathtools}
\usepackage{authblk}
\usepackage[font=small,labelfont=bf,labelsep=period]{caption}
\usepackage[shortlabels]{enumitem}
\usepackage{graphicx}
\graphicspath{{./}}
\usepackage{float}
\usepackage{booktabs, multirow, makecell}
\usepackage{flafter}
\usepackage{needspace}
\usepackage{placeins}
\usepackage[bottom]{footmisc}  %
\usepackage{algorithm}
\usepackage[noend]{algpseudocode}
\providecommand{\LComment}[1]{\Statex \(\triangleright\) \emph{#1}}
\usepackage{tikz}
\usetikzlibrary{shapes.geometric, arrows.meta, positioning, calc, fit, decorations.pathreplacing}
\usepackage{microtype}
\usepackage{xcolor}
\usepackage[normalem]{ulem} %
\usepackage[round,sort]{natbib}
\usepackage{hyperref}
\usepackage{xurl}           %
\usepackage{cleveref}

\definecolor{linkred}{RGB}{192,39,45}
\definecolor{citeblue}{RGB}{47,111,186}
\hypersetup{
  colorlinks=true,
  linkcolor=linkred,
  citecolor=citeblue,
  urlcolor=citeblue,
  pdftitle={MoRE: Scaling mixture of experts with hardware-aware low-rank routing},
  pdfauthor={Honam Wong, Surbhi Goel, Enric Boix-Adsera}
}

\newtheorem{theorem}{Theorem}

\newtheorem{proposition}{Proposition}
\newtheorem{corollary}{Corollary}
\newtheorem{lemma}{Lemma}

\newtheorem{definition}{Definition}

\renewenvironment{abstract}{%
  \par\small
  \begin{center}\textbf{Abstract}\end{center}\vspace{-0.6em}%
  \hspace*{1.5em}\ignorespaces
}{\par\vspace{1.2em}}

\title{MoRE: Scaling mixture of experts with\\hardware-aware low-rank routing}

\renewcommand{\thefootnote}{\fnsymbol{footnote}}
\author[1]{Honam Wong}
\author[1,*]{Surbhi Goel}
\author[2,*]{Enric Boix-Adser\`{a}}
\affil[1]{University of Pennsylvania}
\affil[2]{The Wharton School, University of Pennsylvania}
\date{}

\begin{document}

\addtocontents{toc}{\protect\setcounter{tocdepth}{-1}}

\maketitle
\footnotetext[1]{Equal advising.}
\renewcommand{\thefootnote}{\arabic{footnote}}
\setcounter{footnote}{0}

\begin{abstract}
    Mixture-of-Experts (MoE) layers are central to frontier language models, and recent architectures push toward more and smaller experts. In this regime, the standard linear router becomes a bottleneck: with $M$ experts and hidden dimension $h$, its per-token cost $\Theta(Mh)$ dominates the MoE layer once $M$ is large. We introduce MoRE (\textbf{M}ixture \textbf{o}f \textbf{R}ank-reduced-routed \textbf{E}xperts), which factorizes the router weight matrix at rank $r$ and reduces the routing cost to $O((h + M)r)$. We prove that rank logarithmic in $M$ suffices for routing expressivity when the number of active experts is fixed, and is necessary up to precision factors. We also prove that logarithmic rank preserves load balance in a Gaussian memorization model, and training on a synthetic phonebook task shows that low rank does not hurt memorization. At matched active FLOPs, the factorization allows a factor of $\Theta(h/r)$ more experts. To realize this gain in wall-clock time, we design a fused Triton kernel at inference that avoids expensive memory operations on HBM. Empirically, MoRE improves memorization on the phonebook task and performance on knowledge-intensive Q\&A benchmarks after pretraining, while matching reasoning ability.\footnote{Code available at \url{https://github.com/Matheart/MoRE_code}.}
\end{abstract}

\section{Introduction}
\label{sec:introduction}

Mixture-of-experts (MoE) is now the dominant architectural paradigm for frontier language models~\citep{deepseek2024deepseekv2, deepseek2024deepseekv3, jiang2024mixtral, grok2024, yan2026scalabletrainingmixtureofexpertsmodels}. In these architectures, the standard dense MLP is replaced by a pool of independent experts and a router layer that activates only a small subset of experts per token. Recent advances in model scaling rely heavily on increasing the number of experts in MoE layers~\citep{abnar2025parameters}. This approach expands the total parameter count, allowing models to benefit from neural scaling laws without proportionally increasing active compute \citep{kaplan2020scaling,abnar2025parameters}. In particular, at a fixed active compute budget, MoE layers achieve significantly better \textit{memorization and factual recall} than dense MLPs \citep{jelassi2025mixtureparrotsexpertsimprove}. Recent frontier MoE models are also increasingly operating in the \textit{fine-grained regime} (Appendix~\ref{appendix:frontier_models}), in which the layer is divided into a larger number of smaller experts, each with intermediate width below the hidden dimension~\citep{krajewski2024scalinglawsfinegrainedmixture}.

\begin{figure}[!htbp]
  \centering
  \includegraphics[width=0.7\linewidth]{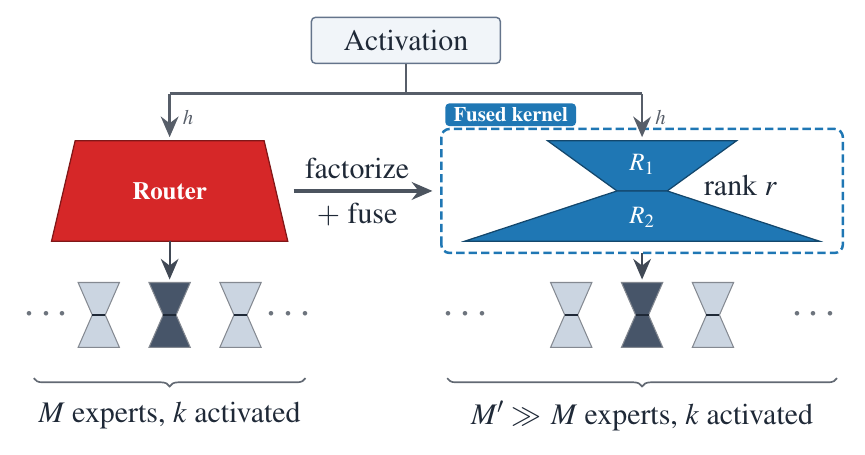}\\[0.5em]
  \begin{minipage}{0.46\linewidth}
    \centering
    \includegraphics[height=0.5\linewidth]{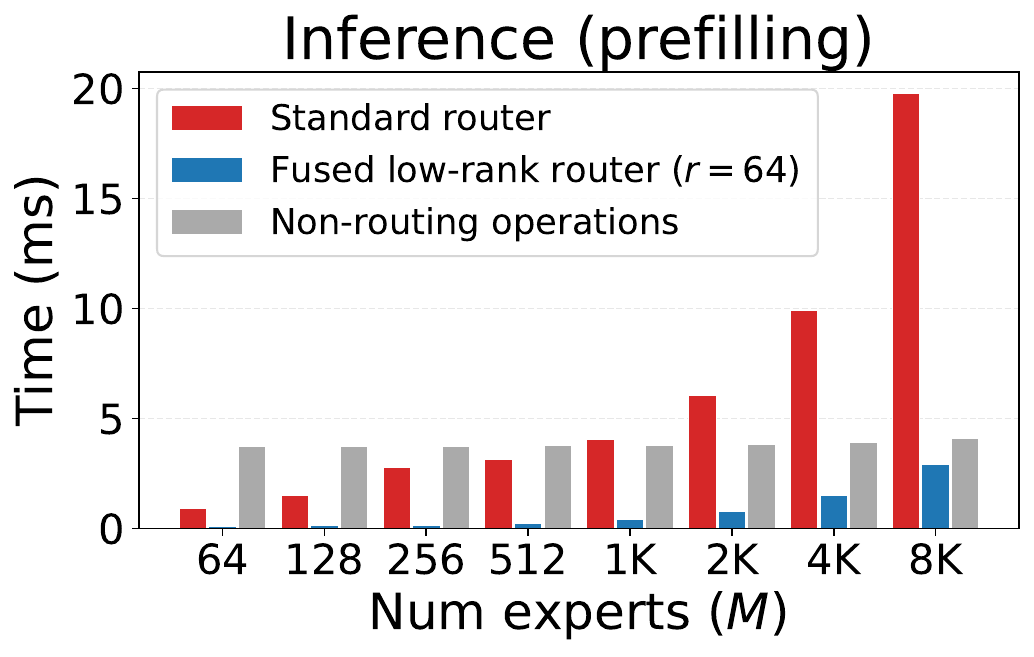}
  \end{minipage}\hspace{0.04\linewidth}%
  \begin{minipage}{0.46\linewidth}
    \centering
    \includegraphics[height=0.5\linewidth]{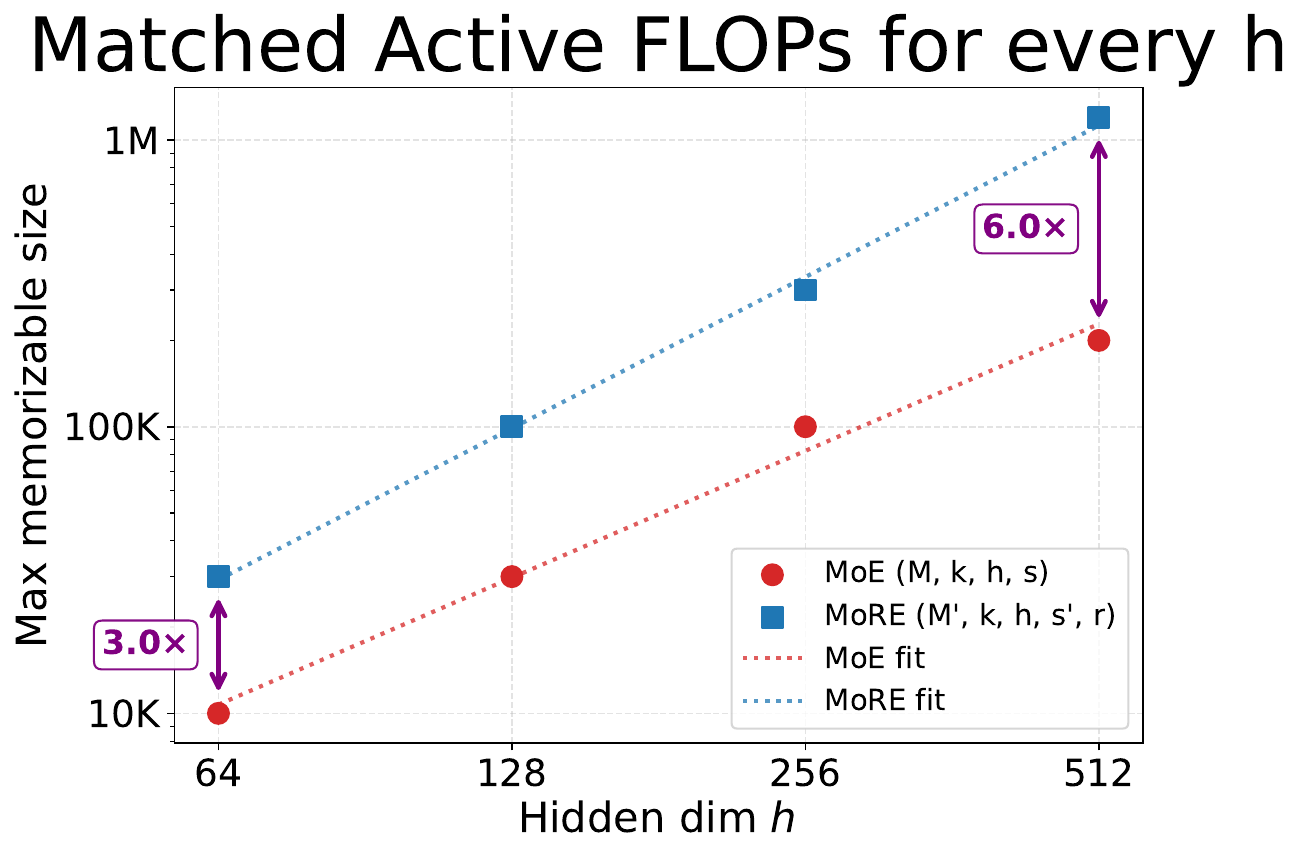}
  \end{minipage}
  \caption{\textbf{MoRE unlocks more scaling with the same active FLOPs.}
  \textbf{(Top)} Replacing the full-rank $h \times M$ router with a low-rank factorization $R_1, R_2$ inside a fused kernel cuts routing cost from $O(hM)$ to $O(r(h+M'))$, enabling $M' \gg M$ experts at matched active FLOPs.
  \textbf{(Bottom Left)} Inference (prefilling) wall-clock time of one layer at $h{=}512$, split into the router and the rest (setup as in Figure~\ref{fig:prefill_heatmap}). The standard router's cost grows linearly with $M$ and dominates at large $M$, while MoRE's fused low-rank router grows far more slowly.
  \textbf{(Bottom Right)} At matched FLOPs, MoRE memorizes more on the synthetic phonebook task, and the gap widens with $h$, as predicted by the $\Theta(h/r)$ scaling of Section~\ref{sec:scaling}.}
  \label{fig:teaser}
\end{figure}

However, further scaling the number of experts in MoE layers hits a bottleneck: routing costs begin to dominate the expert cost as the expert pool grows. This is driven by two distinct factors: the increased active compute required to calculate the routing function, and the heavy memory-bandwidth cost of writing and reading the routing scores. These combined costs can offset the efficiency gains of sparsity in the many-expert regime (see Figure~\ref{fig:teaser}). This motivates our main question:

\begin{center}
    \textit{Can we scale the number of experts in MoEs
    efficiently?}
\end{center}

\subsection{Our contributions}
\noindent In this work, we propose \textbf{M}ixture \textbf{o}f \textbf{R}ank-reduced-routed \textbf{E}xperts (MoRE), an architectural design that allows for scaling to more experts with an asymptotically lower computational cost. \textbf{MoRE replaces the full-rank linear router with a low-rank router}. This factorization reduces the number of parameters in the routing function and the FLOPs required for a forward pass, allowing scaling to more experts. Our router design is inspired by the low-rank parameterization of key-query (KQ) weights in recent attention architectures \citep{deepseek2024deepseekv2} as well as the low-rank design of LoRA adapters \citep{hu2021loralowrankadaptationlarge}.

\medskip\noindent\textbf{Expressivity guarantees.} To determine whether the rank bottleneck limits MoRE's routing capacity, we study which expert assignments it can realize and how evenly it can use the experts. In Section~\ref{sec:theory}, we first prove general upper and lower bounds on the rank needed to realize any top-$k$ assignment. Then, we prove that \textbf{low-rank routing in MoRE layers preserves load-balancing ability}, by constructing a router with $r = \Theta(\log M)$ that balances expert assignment in a Gaussian memorization model extending that of \citet{jelassi2025mixtureparrotsexpertsimprove}. By spreading inputs across experts, balanced routing helps use the memorization capacity of the whole expert pool. We also show empirically on a synthetic memorization task, Phonebook, that low rank does not hurt load balancing and memorization (Section~\ref{subsec:expert-utilization-at-training-time}).

\medskip\noindent\textbf{Scaling experts efficiently.} At matched active FLOPs, a MoRE layer with rank-$r$ routing and hidden dimension $h$ supports a factor of $\Theta(h/r)$ more experts than an MoE layer with the standard router (Section~\ref{sec:scaling}). However, realizing this gain in wall-clock time requires more than a FLOP reduction: materializing the score tensor in HBM (High Bandwidth Memory) becomes expensive when $M$ is large. In this paper, we focus on inference, since a model is trained once but runs inference on every request after deployment and during RL post-training.  We therefore design a fused kernel that keeps scores in on-chip SRAM (Section~\ref{sec:efficiency}). The kernel computes scores in blocks and retains only the running top-$k$ values and expert indices. Notably, the \textbf{inference time} (prefilling) of MoRE matches that of MoE at matched active FLOPs, with a much larger number of experts. %

\medskip\noindent\textbf{Experimental validation.} Finally, we validate our architecture through pretraining experiments on both the phonebook task and real-world Internet data. We show that \textbf{MoRE models consistently improve performance in memorization for phonebook task and improve on knowledge-intensive} tasks such as TriviaQA \citep{joshi2017triviaqalargescaledistantly} and Natural Questions \citep{kwiatkowski2019natural}, compared with standard MoEs at equal active FLOPs. Additionally, we find that MoRE matches performance on reasoning tasks.

\FloatBarrier

\subsection{Related Work}

\noindent\textbf{Routing in MoE layers.}
The standard router in modern MoE is a learned linear layer that scores all $M$ experts per token~\citep{shazeer2017outrageously, lepikhin2020gshard, fedus2022switch}. \emph{Hash routing} fixes the assignment to a deterministic hash~\citep{roller2021hashlayerslargesparse}, but \citet{clark2022unified, dikkala2023on} show that deterministic assignment underperforms learned routing at scale, motivating the search for \textit{parameter-efficient} learned routers rather than discarding learning. \citet{boix2025secret} suggests low-rank routing but evaluates it only in a limited distillation setting. In the $M \ll h$ regime, X-MoE~\citep{chi2022representationcollapse} and L2R~\citep{yang2026l2rlowranklipschitzcontrolledrouting} also introduce low-rank projections, but to address issues in token representations and routing geometry. MoRE instead targets the fine-grained regime where $M$ is large and $M \ll h$ no longer holds, using low-rank routing to reduce the cost of scoring many experts. \citet{he2024mixturemillionexperts} introduces \emph{PEER}, a product-key retrieval layer with cost $O(h\sqrt{M})$. At matched active FLOPs, MoRE takes much less wall-clock time than PEER (see Appendix~\ref{appendix:additional_related_work}).

\medskip\noindent\textbf{Fine-grained MoE and scaling.}
Recent architectures are increasingly designed in the \emph{fine-grained} regime, pairing a small expert intermediate dimension $s$ with a large expert count $M$. \citet{krajewski2024scalinglawsfinegrainedmixture} introduce fine-grained MoE and show that it outperforms coarse-grained MoE at a fixed training budget, and that higher granularity is out of reach because of the routing cost. \citet{li2026slicingdicingconfiguringoptimal} confirm through extensive experiments that larger granularity and expert counts result in better performance. Since MoE memorization capacity scales with \emph{total} rather than active parameters~\citep{jelassi2025mixtureparrotsexpertsimprove}, designs with large $M$ are directly motivated. Nevertheless, these works do not address the routing bottleneck that emerges as $M$ grows, and we introduce MoRE to reduce the routing cost.

\medskip\noindent\textbf{Knowledge storage in language models.}
Two threads study how language models store factual knowledge. The first measures memorization in standard transformer and MoE architectures: \citet{geva2021transformer} interpret FFN layers as key-value memories, \citet{allenzhu2024physicslanguagemodels33} establish knowledge capacity scaling laws for dense LMs, and \citet{jelassi2025mixtureparrotsexpertsimprove} show that for MoE, memorization capacity scales with \emph{total} parameters rather than active parameters, directly motivating designs with large $M$. A second thread replaces the FFN with a learnable key-value lookup: Memory Layers at Scale~\citep{berges2024memorylayersscale} scales product-key retrieval to billions of stored embeddings, Ultra-Sparse Memory~\citep{huang2025ultrasparsememorynetwork} pushes to ultra-sparse activation patterns, and \citet{cheng2026conditionalmemoryscalablelookup} extend retrieval-style storage to conditional access. MoRE is complementary to this thread: it preserves the MLP-expert structure of MoE and inherits the memorization analysis of \citet{jelassi2025mixtureparrotsexpertsimprove} while reducing only the routing cost.

\medskip\noindent\textbf{Efficient MoE systems.}
Existing fused MoE kernels target \emph{expert dispatch}: ScatterMoE~\citep{tan2024scatteredmixtureofexpertsimplementation}, MegaBlocks~\citep{gale2023megablocks}, and follow-up IO-aware kernels~\citep{costin2025momoe, guo2025sonicmoeacceleratingmoeio} fuse scatter-gather with the per-expert grouped matmul, but leave open the question of how to implement the router efficiently in the many-expert regime. vLLM~\citep{kwon2023vllm} provides a \texttt{FusedTopKRouter} that fuses top-$k$ selection with softmax normalization, but does not fuse the router computation with top-$k$. At inference, expert offloading~\citep{hwang2024pregatedmoe, xue2024moeinfinity, he2024expertflow} reduces serving cost by moving inactive experts to CPU memory, but does not reduce the router computation. MoRE instead fuses both low-rank matmuls and top-$k$ selection into one router kernel (Section~\ref{subsec:fused-kernel}).

\section{Preliminaries}
\label{sec:preliminaries}

\noindent\textbf{Notation.} We write $h$ for the hidden dimension, $s$ for the expert intermediate dimension, $M$ for the number of experts per layer, $k$ for the number of active experts per token (with $k < M$), $r$ for the rank of the low-rank router, and $L$ for the number of transformer layers. With batch size $B$ and sequence length $T$, each iteration processes $N = B T$ tokens. We write $[m] = \{1, \ldots, m\}$, and use $P$ for the total parameter count and $\tau$ for the per-token active FLOPs.

\medskip\noindent\textbf{Mixture-of-experts layer.} An MoE layer has $M$ experts $\{f_i\}_{i=1}^{M}$ and a router $g : \mathbb{R}^h \to \mathbb{R}^M$. Each expert is usually a SwiGLU MLP~\citep{shazeer2020glu} with three weight matrices $W^{\mathrm{gate}}_i, W^{\mathrm{up}}_i \in \mathbb{R}^{s \times h}$ and $W^{\mathrm{down}}_i \in \mathbb{R}^{h \times s}$:
\[
  f_i(x) \;=\; W^{\mathrm{down}}_i \bigl( \sigma(W^{\mathrm{gate}}_i x) \odot W^{\mathrm{up}}_i x \bigr), \qquad f_{\mathrm{MoE}}(x) \;=\; \sum_{i=1}^{M} g_i(x) f_i(x),
\]
with $\sigma$ the SiLU activation and $\odot$ the elementwise product. Each expert thus has $3hs$ parameters. The \textbf{standard router}~\citep{shazeer2017outrageously, fedus2022switch} is parametrized by a linear projection $R \in \mathbb{R}^{M \times h}$ and applies softmax over the $k$ largest scores: $g(x; R) = \mathrm{softmax}\bigl(\mathrm{TopK}(R x)\bigr)$, where $\mathrm{TopK}(x)$ denotes the $k$ largest values of $x \in \mathbb{R}^M$. We write its configuration as $\mathrm{MoE}(M, k, h, s)$. 

\medskip\noindent\textbf{MoRE layer.} In a MoRE layer, we replace $R$ with $R_1 R_2$ where $R_1 \in \mathbb{R}^{M \times r}$ and $R_2 \in \mathbb{R}^{r \times h}$, scoring through an $r$-dimensional bottleneck: 
$$g(x; R_1, R_2) = \mathrm{softmax}\bigl(\mathrm{TopK}(R_1 R_2 x)\bigr).$$ Routing parameters drop from $M h$ to $(M+h) r$. We write the low-rank configuration as $\mathrm{MoRE}(M, k, h, s, r)$.

\medskip\noindent\textbf{Total parameters and active cost.} The total MLP parameter count per
MoE/MoRE layer is $3Mhs$, and the per-token active FLOPs per layer are
\begin{equation}
  \tau \;=\;
  \begin{cases}
    C_e\, k h s + C_r\, M h & \text{(standard router)},\\
    C_e\, k h s + C_r\, (M+h)\,r & \text{(low-rank router)},
  \end{cases}
\end{equation}
where $C_e, C_r$ collect per-operation FLOP factors. Following
\citet[App.~E]{krajewski2024scalinglawsfinegrainedmixture}, at inference we have $(C_e, C_r) = (6, 6)$. Training gives a similar ratio, $(C_e, C_r) = (18, 14)$.

\section{Low-rank routing preserves memorization capabilities}
\label{sec:theory}

MoRE replaces the standard router with a rank-$r$ factorization, which raises two related theoretical questions. First, can the router still express the needed expert assignments through the rank bottleneck? Second, when there are many inputs to memorize, can the low-rank router still use the experts' parameters efficiently by balancing the inputs assigned to them? We answer both questions positively in Sections~\ref{subsec:routing-expressivity} and~\ref{subsec:memorization-implications}, and confirm the answers with experiments on the synthetic Phonebook memorization task.

\subsection{Routing expressivity}
\label{subsec:routing-expressivity}

Here an assignment is a subset of $k$ experts. We show that, for fixed $k$, a \emph{single} router can select every assignment in a given family for some input, with rank growing only logarithmically in the number of experts and assignments. This dependence on the number of assignments is necessary up to a factor depending on numerical precision.

\medskip\noindent\textbf{Setup.} The rank-$r$ router has a simple geometric form: the rows of $R_1 \in \mathbb{R}^{M \times r}$ define the $M$ expert vectors $u_i \in \mathbb{R}^r$, and each input $x \in \mathbb{R}^h$ reaches $R_1$ only through its projected representation $z(x) := R_2 x \in \mathbb{R}^r$ where $R_2 \in \mathbb{R}^{r \times h}$. The router then selects the $k$ experts with the largest routing scores $\langle u_i, z(x) \rangle$, $1 \leq i \leq M$. For exposition, throughout this section we normalize the scale by assuming $\|u_i\|_2 \leq 1$ for every expert vector and $\|z(x)\|_2 \leq 1$ for every input $x$ under consideration.\footnote{We state the general result in Appendix~\ref{appendix:theoretical_results}, which only requires them to be bounded by some constants $B_u$ and $B_z$.}

\medskip\noindent\textit{Margin and precision.} Because top-$k$ is implemented at finite precision, selecting a subset robustly requires a positive gap between selected and unselected scores. We formalize this as follows.
\begin{definition}
\label{def:realized-assignment}
Fix a rank-$r$ router $R = R_1 R_2$. A subset $S \in \binom{[M]}{k}$ is \emph{realized at input $x$ with margin} $\eta > 0$ if $S = \mathrm{TopK}(Rx)$ and
\[
\min_{i \in S} (Rx)_i - \max_{j \notin S} (Rx)_j \geq \eta.
\]
Moreover, we say a family $\mathcal{F} \subseteq \binom{[M]}{k}$ is \emph{realizable at margin $\eta$} by $R$ if every $S \in \mathcal{F}$ has some input $x_S \in \mathbb{R}^h$ that realizes it at margin at least $\eta$.
\end{definition}
Note that for a $p$-bit implementation, a margin $\eta$ is resolvable when $\eta = \Omega(2^{-p})$, which is equivalent to $p = \Omega(\log (1/\eta))$. We first present an upper bound based on the Johnson--Lindenstrauss lemma~\citep{johnson1984extensions} by providing a router construction:

\begin{theorem}[Upper bound on required rank]
\label{thm:upper}
There exists a universal constant $C > 0$ such that the following holds. For $1 \leq k \leq M-1$, any collection of top-$k$ assignments $\mathcal{F} \subseteq \binom{[M]}{k}$, and any rank
\[
r \geq C  k \log(M + |\mathcal{F}|),\qquad h \geq r,\]
there exists a rank-$r$ router $R = R_1 R_2$ with $\|u_i\|_2 \leq 1$ that realizes every $S \in \mathcal{F}$ at margin $\Omega(1/\sqrt{k})$, using inputs $\{x_S\}_{S \in \mathcal{F}}$ with $\|R_2 x_S\|_2 \leq 1$.
\end{theorem}

\noindent\textit{Proof idea.} In the uncompressed case $r = M$, take the expert vectors to be the standard basis vectors $e_1,\ldots,e_M \in \mathbb{R}^M$. The representation $z_S \propto \tfrac{\mathbf{1}_S}{k} - \tfrac{\mathbf{1}_{[M]\setminus S}}{M-k}$ is positive exactly on $S$, so it routes to $S$ with margin $\Theta(1/\sqrt{k})$. The Johnson--Lindenstrauss lemma maps these $M+|\mathcal{F}|$ vectors to $r = \Theta(\log(M+|\mathcal{F}|) / \varepsilon^2)$ dimensions while preserving inner products up to additive error $\varepsilon$, and $\varepsilon = \Theta(1/\sqrt{k})$ keeps a constant fraction of the margin, which gives $r = O(k \log(M+|\mathcal{F}|))$. The full proof is in Appendix~\ref{appendix:theoretical_results}.

Next, we use a volume-packing argument to prove a lower bound on the rank. The intuition is that if two different assignments $S,T \in \mathcal{F}$ are both realized with margin $\eta$, their projected representations must be separated.

\begin{theorem}[Lower bound on required rank]
\label{thm:lower}
Let $\mathcal{F} \subseteq \binom{[M]}{k}$. If a rank-$r$ router $R = R_1 R_2$ with $\|u_i\|_2 \leq 1$ and $\|R_2 x_S\|_2 \leq 1$ realizes every top-$k$ assignment $S \in \mathcal{F}$ at margin at least $\eta$, then
\[
|\mathcal{F}| \leq (1 + 2/\eta)^r,
\qquad
r \geq \frac{\log |\mathcal{F}|}{\log(1 + 2/\eta)}.
\]
\end{theorem}

\noindent\textit{Proof idea.} If two assignments $S \neq T$ are realized at $x_S$ and $x_T$, pick $i \in S\setminus T$ and $j \in T\setminus S$. Expert $i$ beats expert $j$ by at least $\eta$ at $x_S$ and loses by at least $\eta$ at $x_T$, so $\|u_i-u_j\|_2 \leq 2$ forces $\|R_2 x_S - R_2 x_T\|_2 \geq \eta$. Packing $|\mathcal{F}|$ such points into the unit ball of $\mathbb{R}^r$ gives $|\mathcal{F}| \leq (1+2/\eta)^r$.

Finally, taking $\mathcal{F}$ to be the collection of all top-$k$ subsets gives the following consequence:
\begin{corollary}[Full expressivity]
\label{cor:full-expressivity}
At precision $p = \Omega(\log k)$, rank $r = O(k^2 \log(M/k))$ suffices to realize every top-$k$ assignment over $M$ experts. Conversely, any rank-$r$ router that realizes every top-$k$ assignment at precision scale $p$ must have $r = \Omega(k \log(M/k)/p)$.
\end{corollary}
\textit{Remark.} The extra factor of $k$ comes from preserving the routing margin in the Johnson--Lindenstrauss projection. Our bounds do not establish whether this factor is necessary.

\subsection{Load balancing and memorization capacity}\label{subsec:memorization-implications}
Memorization, a core function of MoE layers \citep{du2022glam,geva2021transformer,jelassi2025mixtureparrotsexpertsimprove}, concerns how many facts can be stored, not which expert configurations can be addressed. Extending the Gaussian model of \citet{jelassi2025mixtureparrotsexpertsimprove}, we prove that logarithmic rank does not affect load balancing, and then explain how load balancing relates to memorization capacity.

\begin{proposition}[Gaussian load balancing with logarithmic rank]
\label{prop:gaussian-load-balancing}
Let $x_1,\ldots,x_n \sim \mathcal{N}(0,I_h)$ be independent, and suppose $M=2^r$ ($r = \Theta(\log_2 M)$) with $h \geq r$. For any $1 \leq k \leq M$, there exists a rank-$r$ top-$k$ router with $M$ experts that selects every expert with probability exactly $k/M$ on each example. Consequently, let $N_a$ be the number of examples routed to expert $a$, and write $L := \log\frac{2M}{\delta}$. For any $\delta \in (0,1)$, the following holds with probability at least $1-\delta$:
\[
\max_{a \in [M]} \left|N_a - \frac{nk}{M}\right| \leq \sqrt{\frac{2nkL}{M}} \;+\; \frac{2}{3}L .
\]
\end{proposition}

The key intuition is that, under this construction, a standard Gaussian input has no preference for one expert over another in expectation. We defer the full proof to Appendix~\ref{appendix:theoretical_results}. Thus, the low-rank factorization preserves the load-balanced mechanism of standard MoE, while reducing routing complexity from $O(Mh)$ to $O((M+h)\log M)$.

\medskip\noindent\textbf{Why balanced load maximizes the number of stored facts.} The number of facts an MLP expert can memorize is determined by its parameter count, as shown for two-layer networks \citep{daniely2020memorizinggaussians} and for language models \citep{allenzhu2024physicslanguagemodels33,allenzhu2023physics}. An MoE layer reaches the total capacity of its experts only under balanced load: an overloaded expert receives more facts than it can memorize, and an underutilized expert leaves its parameters unused.

\medskip\noindent\textbf{Experimental validation with trained models.}
\label{subsec:expert-utilization-at-training-time}
We verify in training that sufficient rank does not affect memorization, using the phonebook task of \citet{jelassi2025mixtureparrotsexpertsimprove}. Each example is a five-letter name followed by an eight-digit number, such as ``alice|12345678''. We pretrain on these pairs and report the largest number of training examples for which the model attains $\geq 90\%$ exact-match accuracy on the numbers given the names, evaluated on $1{,}000$ training examples sampled uniformly at random. To approach the model's full memorization capacity, we train for up to $1{,}500$ epochs, in line with \citet{allenzhu2024physicslanguagemodels33}.

Figure~\ref{fig:memorization_in_training}(Left) varies the router rank $r$ while keeping the expert configuration fixed: MoRE matches the memorization capacity of standard MoE once $r$ is sufficiently large. To compare load balancing, Figure~\ref{fig:memorization_in_training}(Middle, Right) tracks normalized routing entropy, $H_{\mathrm{norm}} = -\sum_{a=1}^{M} p_a \log p_a / \log M$, where $p_a$ is expert $a$'s share of routed tokens. Both models reach high entropy by the end of training.

\begin{figure}[!t]
    \centering
    \includegraphics[width=\linewidth]{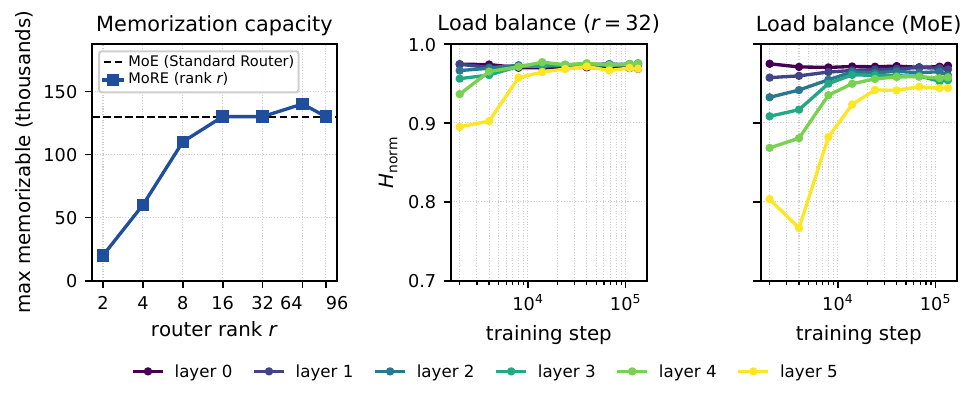}
    \caption{\textbf{(Left) Sufficient rank does not affect memorization. (Middle, Right) MoRE with $r=32$ and MoE have similar load balancing at the end of training.} On the Phonebook task, at $M{=}192$, $h{=}128$, $s{=}8$, $k{=}4$, $L{=}6$. \textbf{Left:} maximum phonebook size MoRE memorizes against the router rank. The dashed line is MoE (standard router) at the same configuration, which MoRE already matches at $r=16$. \textbf{Middle, Right:} routing entropy $H_{\mathrm{norm}}$ per layer against training step for MoRE ($r{=}32$) and MoE. 1 means equal load balancing and 0 means collapse onto one expert. We define $H_{\mathrm{norm}}$ and show similar results for other setups in Appendix~\ref{appendix:hnorm_definition}. }
    \label{fig:memorization_in_training}
\end{figure}

\section{Scaling experts efficiently with fused implementation of MoRE}
\label{sec:efficiency}

Since logarithmic rank keeps the load balanced and does not hurt memorization (Section~\ref{sec:theory}), we ask: \textit{at a fixed active-FLOP budget, how much additional capacity does the low-rank router unlock?} Section~\ref{sec:scaling} shows a $\Theta(h/r)$ gain in the expert count.\footnote{In Appendix~\ref{appendix:more_experts}, we find that adding experts monotonically reduces MoRE's validation perplexity.} The resulting large-$M$ regime creates a memory-bandwidth bottleneck on the routing score tensor, which we resolve with a fused Triton~\citep{tillet2019triton} kernel for the router (Section~\ref{subsec:fused-kernel}).

\subsection{Scaling under fixed active-FLOPs}
\label{sec:scaling}

Fix the active-FLOP budget $\tau$, the hidden dimension $h$, and the active expert count $k$, so that all non-MoE components have identical parameter counts across the two settings. We compare $\mathrm{MoE}(M, k, h, s)$ and $\mathrm{MoRE}(M', k, h, s', r)$ on total MLP parameters $P = 3 M h s$. The full calculation is in Proposition~\ref{thm:optimal_scaling} (Appendix~\ref{appendix:scaling}).

The active budget $\tau$ decomposes into an expert cost $C_e k h s$ and a router cost. For MoE, the router cost is $C_r M h$, i.e.\ $C_r h$ FLOPs per expert. For MoRE, the router cost is $C_r M' r + C_r h r$, where the per-expert term $C_r r$ is an $h/r$ reduction over MoE, and $C_r h r$ is a fixed projection overhead independent of $M'$. Since $P = 3 M h s \propto M s$ is a product, and the budget trades off linearly between scaling $M$ (via the router) and scaling $s$ (via the experts), $P$ is maximized when the budget is split equally between the router and the experts. This gives
\[
M = \frac{\tau}{2 C_r h} \quad \text{(MoE)}, \qquad
M' = \frac{\tau - C_r h r}{2 C_r r} \quad \text{(MoRE)}.
\]

In the regime $\tau \gg C_r h r$ (equivalently $M, M' \gg r$), the projection overhead is negligible, and $M'/M \approx h/r$: each MoRE expert costs only a fraction $r/h$ of the routing cost of an MoE expert. For example, at $h = 2048$ and $r = 16$, MoRE admits roughly $128\times$ more experts than MoE at the same active cost. In Section~\ref{sec:experiments}, we choose configurations according to this derivation, with $C_e = C_r$.\footnote{Since $C_e = C_r = 6$ for inference and $C_e \approx C_r$ in training, the active FLOPs roughly match in both.} We find that the same $\Theta(h/r)$ conclusion also holds under multi-GPU expert parallelism (Appendix~\ref{subsec:scaling-ep}).

\subsection{Realizing the speedup at large $M$}
\label{subsec:fused-kernel}

When the expert count $M$ and the number of tokens $N$ are large, running a top-$k$ over the full $(N, M)$ routing-score tensor requires writing the tensor to HBM (High Bandwidth Memory, see Appendix~\ref{appendix:gpu_memory}) and reading it back, and dominates the routing time (Figure~\ref{fig:fuse_benefit}, Left). This requires \textit{fusing operations} in addition to the low-rank factorization.

 \begin{figure}[!t]
    \centering
    \includegraphics[width=\linewidth, trim=0 5 0 4, clip]{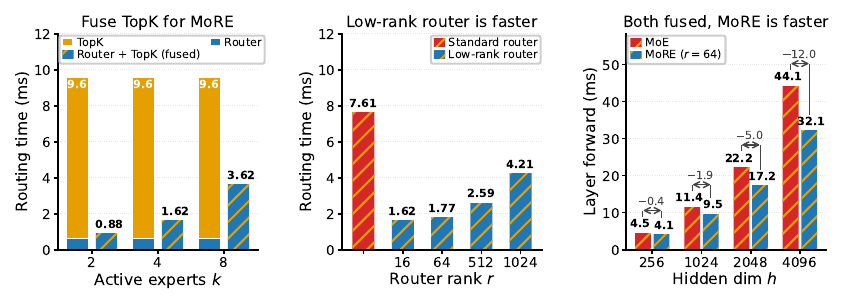}
    \caption{
\textbf{(Left)} Fusing with top-$k$ eliminates the need to materialize the
scores for MoRE.
\textbf{(Middle)} The low-rank router is faster than the standard router, while its cost grows with $r$.
\textbf{(Right)} The gap between MoRE and MoE widens as $h$ grows.
$M{=}4{,}096$, $N = 262{,}144$ ($B{=}256$, $T{=}1024$), single NVIDIA B200, mean of $200$ trials after $50$ warm-ups. Left and middle time the router alone at $h{=}2048$ (left: $r{=}16$, middle: $k{=}4$). Right times the whole layer at $s{=}64$, $k{=}4$, $r{=}64$.}
    \label{fig:fuse_benefit}
\end{figure}

Below, we sketch how MoRE can be fused and leave the full details to
Appendix~\ref{appendix:inference_pseudocode}. Instead of writing the routing
scores for every token to HBM and reading them back to compute the top-$k$
scores, we compute $R_1R_2 x$ and
maintain a running top-$k$ buffer (Figure~\ref{fig:fuse_benefit}, Left)\footnote{Our algorithm works well for common choices of $k$. For a discussion of our
top-$k$ algorithm and a comparison with other implementations, see
Appendix~\ref{appendix:topk_analysis}.}. We find that the low-rank router in MoRE is faster than the standard router across different ranks (Figure~\ref{fig:fuse_benefit}, Middle). MoE can be fused in a similar way, but we find that its speedup from fusing becomes small when $h$ is large, while the speedup for MoRE remains visible, because for MoE the cost of computing $Rx$ and maintaining the top-$k$ buffer scales with $h$ instead of $r$ (Figure~\ref{fig:fuse_benefit}, Right). We also show similar speedups on multiple GPUs with expert and data parallelism, where communication costs are involved (Appendix~\ref{appendix:ep}).%

\medskip\noindent\textbf{Wall-clock time at inference in the large-$M$ regime.} As shown in Figure~\ref{fig:prefill_heatmap}, the fused MoRE layer outperforms the standard MoE layer across the whole grid, and the speedup grows with $M$ along every row, reaching $1.83$ to $4.54\times$ at $M = 8{,}192$. Equivalently, at matched wall-clock time, MoRE supports about eight times more experts: at $h = 512$, the standard layer reaches $M = 512$ in $6.90$\,ms, while MoRE reaches $M = 4{,}096$ in $5.37$\,ms, realizing the scaling advantage of Section~\ref{sec:scaling} in practice. Fusing also reduces peak memory by up to $14\times$ for the whole layer (Figure~\ref{fig:switch_memory}). We find that MoE and MoRE have similar decoding speed (see Appendix~\ref{appendix:decoding}).

\begin{figure}[!tbp]
    \centering
    \includegraphics[width=\linewidth]{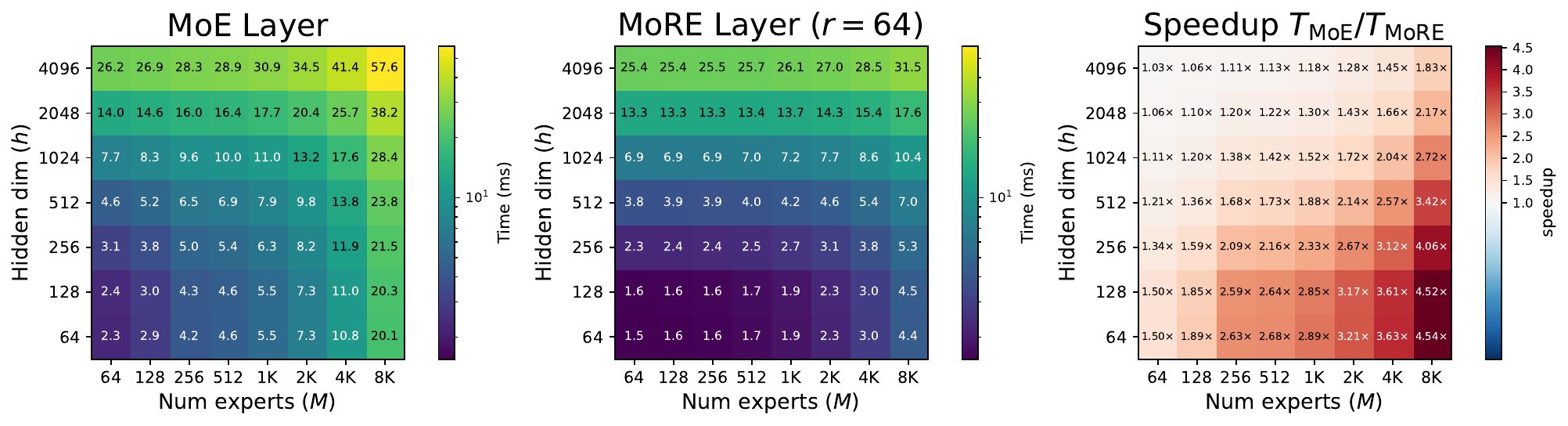}
    \caption{\textbf{Inference (prefilling): At a fixed configuration the MoRE layer
is $1.83$ to $4.54\times$ faster at $M{=}8{,}192$, and at matched active FLOPs it
has comparable speed to the standard MoE layer while holding $h/r$ times more
experts.} Forward pass of a single layer.
\textbf{Left:} MoE layer. \textbf{Middle:} MoRE layer ($r{=}64$). \textbf{Right:} speedup
$T_{\text{MoE}} / T_{\text{MoRE}}$.
$s{=}64$, $k{=}4$, $N = 262{,}144$ ($B{=}256$, $T{=}1024$), mean of
$200$ trials after $50$ warm-ups on a single NVIDIA B200. The comparison where both layers are fused is in Figure~\ref{fig:lowrank_bothfused}. We provide additional measurements on an A100, with similar speedups, in Appendix~\ref{appendix:a100_grid}.}
    \label{fig:prefill_heatmap}
\end{figure}

\section{Experiments: MoRE improves performance at matched active FLOPs}
\label{sec:experiments}

In this section, we show that replacing MoE layers with MoRE layers in transformers can improve performance on memorization and knowledge-intensive tasks while keeping active FLOPs fixed.

\subsection{Synthetic task: increased memorization capabilities at matched active FLOPs}
\label{subsec:phonebook}

We first test on the phonebook task (Section~\ref{subsec:expert-utilization-at-training-time}) whether MoRE's parameter advantage at matched active FLOPs translates into memorization capacity. MoRE memorizes substantially more pairs than MoE at every scale we sweep.
In Figure~\ref{fig:phonebook_capacity}, we sweep the hidden dimension $h \in \{64,128,256,512\}$ at expert width $s=8$, and the expert width $s \in \{8,16,32,64\}$ at $h=64$. At each point, we choose the optimal MoE and MoRE configurations by Proposition~\ref{thm:optimal_scaling} so that they have matched active FLOPs. Notably, the gap widens with larger $h$, in agreement with the $\Theta(h/r)$ trend of Proposition~\ref{thm:optimal_scaling}.\footnote{We stop at $h = 512$ and $s = 64$ because larger models memorize more than 2 million examples, and training on such datasets for up to 1,500 epochs exceeds our compute budget.}

\begin{figure}[!htbp]
    \centering
    \begin{minipage}{0.49\linewidth}
        \centering
        \includegraphics[width=\linewidth]{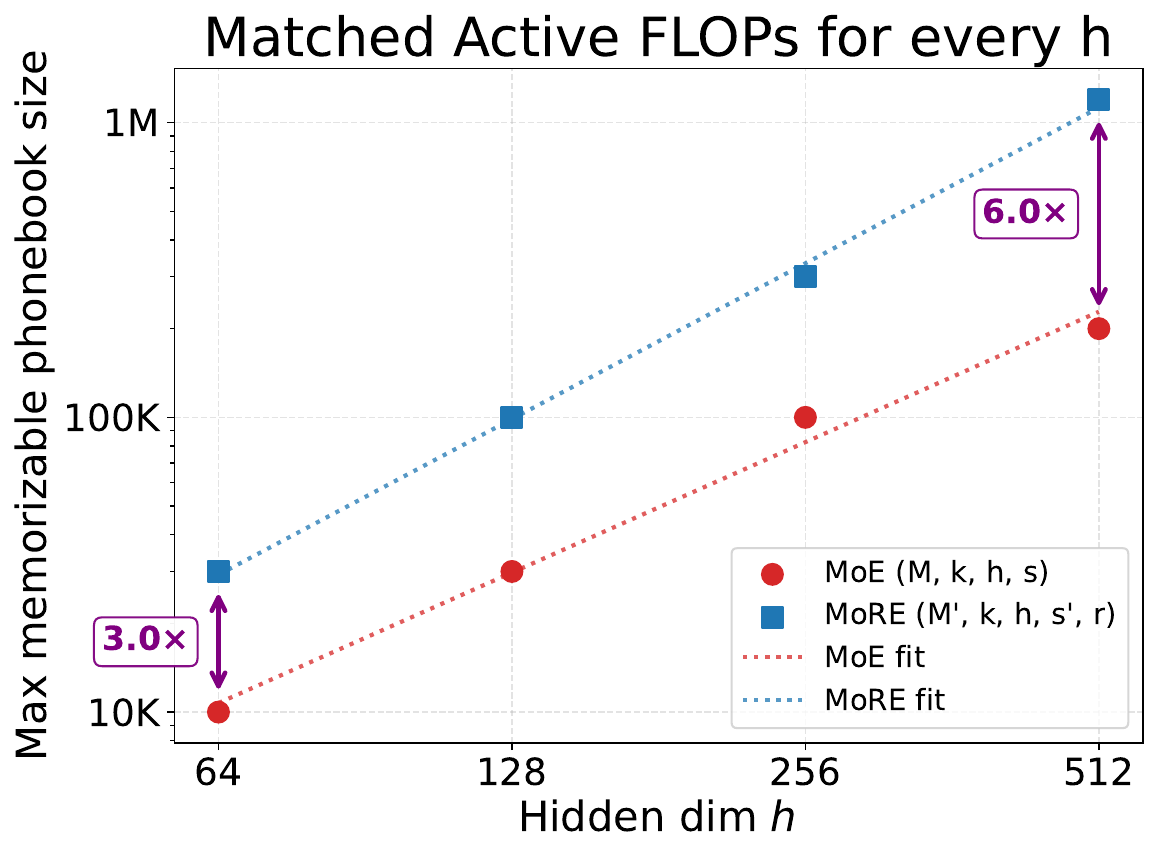}
    \end{minipage}\hfill
    \begin{minipage}{0.49\linewidth}
        \centering
        \includegraphics[width=\linewidth]{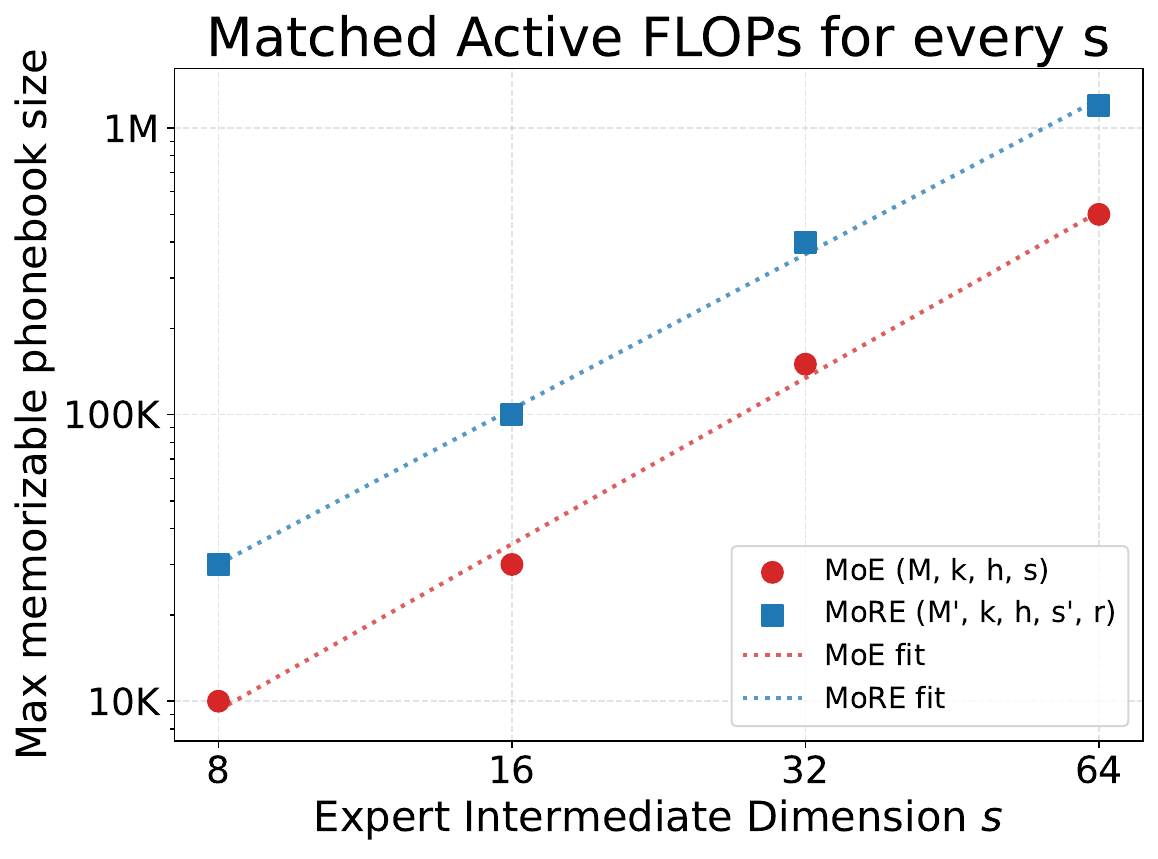}
    \end{minipage}
    \caption{\textbf{At fixed active FLOPs~$\tau$, MoRE memorizes more phonebook entries than MoE, and the gap widens with~$h$ as Proposition~\ref{thm:optimal_scaling} predicts.} \textbf{Left:} sweep over hidden dimension $h \in \{64,128,256,512\}$ at expert width $s=8$. \textbf{Right:} sweep over expert width $s \in \{8,16,32,64\}$ at $h=64$. Dotted lines are log-log linear fits. Configurations and sizes are in Tables~\ref{tab:phonebook_configs_h} and~\ref{tab:phonebook_configs_s} (Appendix~\ref{appendix:phonebook_details}).}
    \label{fig:phonebook_capacity}
\end{figure}

\subsection{Real-world pretraining: improved performance on knowledge-intensive tasks}
\label{subsec:real_world_pretraining}

For the real-world pretraining, we construct a 40B-token mixture dataset following the natural-language pretraining recipe of \citet{jelassi2025mixtureparrotsexpertsimprove}: a mixture of FineWeb-edu \citep{huggingfacefw_2024}, Cosmopedia \citep{benallal2024cosmopedia}, and Wikipedia, together with the training sets of the downstream benchmarks, so that the model is familiar with the downstream tasks at test time \citep{jelassi2025mixtureparrotsexpertsimprove, allenzhu2023physics}. 

We select two MoE/MoRE pairs with matched active FLOPs by Section~\ref{sec:scaling}, a third pair with wider experts ($s$ and $s'$), and a larger pair with $M = 1024$ and $M' = \frac{h}{r} \cdot M = 8192$, trained with expert parallelism. By construction, MoRE has 5 to 6$\times$ more total parameters. On the first three pairs, MoRE reduces validation perplexity by 7\% to 15\% (Table~\ref{tab:fineweb_p3_results}).\footnote{The gap is smaller for the last pair. A possible reason is that a larger model needs $>$40B tokens, which is beyond our budget. In Appendix~\ref{appendix:gap_vs_tokens}, we find that the perplexity gap increases with the token budget.}

\begin{figure}[tbp]
    \centering
    \includegraphics[width=0.55\textwidth]{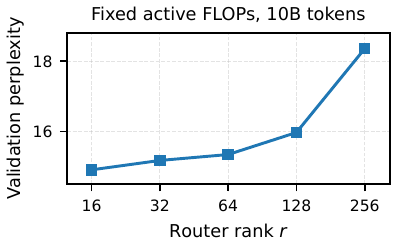}
    \caption{\textbf{With active FLOPs fixed, a smaller rank allows more experts, and validation perplexity falls as the rank shrinks.} The number of experts $M$ grows from 64 to 8704 as $r$ falls from 256 to 16, at $h = 512$, $k = 4$, $L = 12$ and 10B training tokens.}
    \label{fig:rank_sweeps}
\end{figure}
We evaluate 0-shot via lm-evaluation-harness~\citep{eval-harness} on \textbf{knowledge-intensive} tasks (TriviaQA~\citep{joshi2017triviaqalargescaledistantly}, Natural Questions~\citep{kwiatkowski2019natural}) and \textbf{commonsense and science reasoning} (HellaSwag~\citep{zellers2019hellaswagcanmachinereally}, WinoGrande~\citep{sakaguchi2020winogrande}, ARC-Easy~\citep{clark2018thinkyousolvedquestion}, SciQ~\citep{welbl2017crowdsourcingmultiplechoice}). MoRE matches or improves knowledge-intensive performance while preserving commonsense and reasoning scores.

\begin{table}[!tbp]
\centering
\setlength{\tabcolsep}{4pt}
\resizebox{\textwidth}{!}{%
\newcommand{\pmval}[1]{{\scriptsize$\pm$#1}}

\begin{tabular}{llcc|c|cc|cccc}
\toprule
\textbf{$s,h$} & \textbf{Mo(R)E-$M$} & \makecell{\textbf{Active MoE}\\\textbf{FLOPs / layer}} & \makecell{\textbf{Total}\\\textbf{Param}} & \makecell{\textbf{Val}\\\textbf{PPL}$\downarrow$} & \makecell{\textbf{TriviaQA}\\F1$\uparrow$} & \makecell{\textbf{NQ-Open}\\F1$\uparrow$} & \makecell{\textbf{Hella.}\\acc\_norm$\uparrow$} & \makecell{\textbf{ARC-e}\\acc$\uparrow$} & \makecell{\textbf{SciQ}\\acc$\uparrow$} & \makecell{\textbf{Wino.}\\acc$\uparrow$} \\
\midrule
\multirow{2}{*}{32, 1024} & MoE-128    & 1.57M & 287.0M & 14.01          & 16.57\,\pmval{0.24} & 6.88\,\pmval{0.32} & 36.91\,\pmval{0.48} & 56.31\,\pmval{1.02} & 79.60\,\pmval{1.27} & 50.20\,\pmval{1.41} \\
                          & MoRE-1536  & 1.57M & 1.50B  & \textbf{11.96} & \textbf{20.21\,\pmval{0.25}} & \textbf{7.63\,\pmval{0.34}} & \textbf{42.28\,\pmval{0.49}} & \textbf{63.30\,\pmval{0.99}} & \textbf{80.20\,\pmval{1.26}} & \textbf{53.51\,\pmval{1.40}} \\
\midrule
\multirow{2}{*}{64, 512}  & MoE-256    & 1.57M & 362.9M & 14.01          & 16.29\,\pmval{0.24} & 5.72\,\pmval{0.31} & 36.95\,\pmval{0.48} & 57.32\,\pmval{1.01} & 75.80\,\pmval{1.36} & \textbf{51.46\,\pmval{1.40}} \\
                          & MoRE-1792  & 1.57M & 1.91B  & \textbf{12.66} & \textbf{17.12\,\pmval{0.23}} & \textbf{6.88\,\pmval{0.32}} & \textbf{40.15\,\pmval{0.49}} & \textbf{60.44\,\pmval{1.00}} & \textbf{77.70\,\pmval{1.32}} & 50.28\,\pmval{1.41} \\
\midrule
\multirow{2}{*}{128, 512} & MoE-256    & 2.36M & 664.9M & 12.81          & 18.09\,\pmval{0.24} & 5.99\,\pmval{0.29} & 40.00\,\pmval{0.49} & 59.68\,\pmval{1.01} & 76.30\,\pmval{1.35} & 51.22\,\pmval{1.40} \\
                          & MoRE-1792  & 2.26M & 3.76B  & \textbf{11.92} & \textbf{20.79\,\pmval{0.26}} & \textbf{7.62\,\pmval{0.34}} & \textbf{42.07\,\pmval{0.49}} & \textbf{60.27\,\pmval{1.00}} & \textbf{78.70\,\pmval{1.30}} & \textbf{54.70\,\pmval{1.40}} \\
\midrule
\multirow{2}{*}{{64, 512}}  & {MoE-1024}   & {3.93M} & {1.27B}  & {12.43}          & {21.67\,\pmval{0.29}} & {\textbf{8.75\,\pmval{0.40}}} & {40.82\,\pmval{0.49}} & {60.86\,\pmval{1.00}} & {77.10\,\pmval{1.33}} & {\textbf{53.67\,\pmval{1.40}}} \\
                          & {MoRE-8192}  & {3.98M} & {7.92B}  & {\textbf{11.92}} & {\textbf{22.15\,\pmval{0.29}}} & {8.68\,\pmval{0.39}} & {\textbf{42.33\,\pmval{0.49}}} & {\textbf{61.74\,\pmval{1.00}}} & {\textbf{81.50\,\pmval{1.23}}} & {51.70\,\pmval{1.40}} \\

\bottomrule
\end{tabular}

}

\caption{Pretraining results for MoE Transformers on 40B tokens. Each MoRE configuration ($r = 64$) matches the active FLOPs of its MoE counterpart, while total parameters are scaled up. All tasks are evaluated 0-shot, and $\pm$ denotes the standard error over test samples. Active MoE FLOPs shown are at inference and we have matched FLOPs at training too. See Appendices~\ref{appendix:fineweb_details} and~\ref{appendix:flops_param_accounting}
for details.}
\label{tab:fineweb_p3_results}
\label{table:real-training}
\end{table}

\medskip\noindent\textbf{End-to-end wall-clock cost comparison.} At fixed active FLOPs (Table~\ref{table:real-training} and Table~\ref{tab:fineweb_flops_train_inf}), MoE and MoRE have similar inference time in both prefilling and decoding (Tables~\ref{tab:prefill_time} and~\ref{tab:decode_table1}). With the same configuration, training MoRE is $1.10$--$1.23\times$ faster than training MoE.\footnote{At matched active FLOPs, training MoRE can be slower because of work that scales with total parameters, such as the AdamW update. A more careful implementation may shorten this gap (Appendix~\ref{appendix:train_time}).}

\begin{table}[!tbp]
    \centering
    \small
    \begin{tabular}{llrrr}
        \toprule
        Scale & Model & \#Experts & Total params & Prefill (ms) \\
        \midrule
        $s{=}32$, $h{=}1024$ & MoE & 128 & 0.29\,B & 115.4 \\
         & MoRE ($r{=}64$) & 1,536 & 1.50\,B & 115.1 \\
        \midrule
        $s{=}64$, $h{=}512$ & MoE & 256 & 0.36\,B & 101.4 \\
         & MoRE ($r{=}64$) & 1,792 & 1.91\,B & 95.3 \\
        \midrule
        $s{=}128$, $h{=}512$ & MoE & 256 & 0.66\,B & 102.0 \\
         & MoRE ($r{=}64$) & 1,792 & 3.76\,B & 98.0 \\
        \midrule
        $s{=}64$, $h{=}512$ & MoE & 1,024 & 1.27\,B & 105.4 \\
         & MoRE ($r{=}64$) & 8,192 & 7.92\,B & 104.0 \\
        \bottomrule
    \end{tabular}
    \caption{\textbf{MoRE prefills as fast as the MoE it matches in active FLOPs, while holding $5$ to $6$ times as many parameters.} Prefill wall-clock (ms) for the pairs of Table~\ref{tab:fineweb_p3_results}: the forward pass of the whole model on one NVIDIA B200 at the per-device batch the training runs use, $T{=}1024$, mean of $200$ iterations after $100$ warm-ups. Training time for the same pairs is in Table~\ref{tab:train_time}.}
    \label{tab:prefill_time}
\end{table}

\medskip\noindent\textbf{Ablation of the router rank at matched active FLOPs.} Reducing $r$ from $256$ to $16$ at fixed active FLOPs raises MoRE's expert count from $64$ to $8{,}704$ and lowers validation perplexity monotonically, from $18.35$ to $14.90$ (Figure~\ref{fig:rank_sweeps}). As the benefit plateaus at $r = 64$, we keep $r = 64$ throughout our pretraining experiments.

\section{Conclusion and Discussion}
We introduce \textit{MoRE}, a principled solution to the MoE routing bottleneck that scales experts by a factor of $\Theta(\frac{h}{r})$ while matching active FLOPs. We provide theoretical guarantees for its expressivity, design a hardware-efficient implementation, and demonstrate improved performance on knowledge-intensive benchmarks such as TriviaQA \citep{joshi2017triviaqalargescaledistantly} and Natural Questions \citep{kwiatkowski2019natural}, without degrading reasoning. 

\medskip\noindent\textbf{Limitations and future work.} Our real-world training experiments are still limited in scale ($h \leq 2048$, fewer than 10B total parameters) due to our limited computational budget. We also mainly compare MoE and MoRE under the principle of Section~\ref{sec:scaling} for improving memorization, and future work could study the scaling laws of the two architectures. While we find that matching the standard router's initialization variance can already improve upon the default initialization of $R_1, R_2$ (Appendix~\ref{appendix:initialization}), a LoRA-like \citep{hu2021loralowrankadaptationlarge} initialization might improve performance further.

\section{Acknowledgements}
We thank Wentao Guo for helpful discussions and feedback. HW is supported by the gift from AWS to Penn Engineering's ASSET Center for Trustworthy AI.

\bibliography{references}
\bibliographystyle{references}

\appendix
\addtocontents{toc}{\protect\setcounter{tocdepth}{2}}
\renewcommand{\contentsname}{Appendix}
\clearpage
\tableofcontents
\newpage

\section{Frontier MoE Model Configurations}
\label{appendix:frontier_models}

For reference, Table~\ref{table:frontier_model_comparison} summarizes the configurations of recent frontier MoE models. The trend toward higher sparsity ($M \gg k$) and higher granularity ($s < h$) motivates the fine-grained regime studied throughout this paper.

\begin{table}[h!]

    \centering

    \footnotesize

    \begin{tabular}{lccccccc}

        \toprule

        \textbf{Model} & \textbf{Total} & \textbf{Active} & $L$ & $M$ & $k$ & $h$ & $s$ \\

        \midrule

        \textbf{Mixtral 8x22B}~\cite{mixtral8x22b2024}       & 141B  & 39B   & 56 & 8   & 2  & 6144 & 16384 \\

        \textbf{Grok-1}~\cite{grok2024}                      & 314B  & 79B & 64 & 8 & 2 & 6144 & 32768 \\

        \textbf{DeepSeek-V2}~\cite{deepseek2024deepseekv2} & 236B  & 21B   & 60 & 160 & 6  & \textbf{5120} & \textbf{1536} \\

        \textbf{DeepSeek-V3}~\cite{deepseek2024deepseekv3} & 671B  & 37B   & 61 & 256 & 8  & \textbf{7168} & \textbf{2048} \\

        \textbf{Llama 4 Maverick}~\cite{meta2025llama4}      & 400B  & 17B   & 48 & 128 & 1  & 5120 & 8192 \\

        \textbf{Qwen3 235B-A22B}~\cite{yang2025qwen3}        & 235B  & 22B   & 94 & 128 & 8  & \textbf{4096} & \textbf{1536} \\

        \textbf{Kimi K2}~\cite{kimiteam2025kimik2}           & 1.04T & 32B   & 61 & 384 & 8  & \textbf{7168} & \textbf{2048} \\

        \textbf{gpt-oss-120b}~\cite{openai2025gptoss}        & 117B  & 5.1B  & 36 & 128 & 4  & 2880 & 2880 \\

        \textbf{Qwen3-Next-80B-A3B}~\cite{qwen2025qwen3next} & 80B   & 3B    & 48 & 512 & 10 & \textbf{2048} & \textbf{512} \\

        \bottomrule

    \end{tabular}

    \caption{Configurations of frontier MoE models. $M$: total experts; $k$: active experts per token; $h$: hidden dimension; $s$: expert FFN intermediate dimension. The trend shows greater sparsity and granularity is preferred.}

    \label{table:frontier_model_comparison}

\end{table}

\section{Additional Discussions in Related Work}
\label{appendix:additional_related_work}

\noindent\textbf{Low-rank routing in the $M \ll h$ regime.}
X-MoE~\citep{chi2022representationcollapse} scores experts by the dot product between a hidden vector $x \in \mathbb{R}^h$ and $M$ learned expert embeddings. The gradient pulls $x$ toward the span of those embeddings, so after training the hidden states entering a routed layer concentrate in a subspace of $\mathbb{R}^h$ of dimension at most $M$. When $M \ll h$, that subspace is much smaller than $\mathbb{R}^h$, a phenomenon the authors call \emph{representation collapse}. To mitigate this, \citet{chi2022representationcollapse} project both expert embeddings and the hidden vector to a lower dimension. L2R~\citep{yang2026l2rlowranklipschitzcontrolledrouting} identifies a related but distinct issue: when $M \ll h$, token representations and expert embeddings become nearly orthogonal, which narrows the margin between routing logits and produces ambiguous expert selection. L2R introduces low-rank routing to improve discrimination among experts.

Our work also uses low-rank routing and derives upper and lower bounds on the rank based on the margin argument, but it serves a different purpose. The mechanisms above require $M \ll h$ (e.g., $h = 4096$ with $M \leq 64$, as in~\citet{yang2026l2rlowranklipschitzcontrolledrouting} and Mixtral~\citep{jiang2024mixtral}). MoRE instead targets the fine-grained regime where $M$ is scaled up and $M \ll h$ does not hold. In this regime, the low-rank router reduces routing cost rather than correcting the geometry of token representations. 

PEER \citep{he2024mixturemillionexperts} reaches a very large expert pool by shrinking each expert
to a single neuron and replacing the linear router with product-key retrieval, which has complexity $O(h \sqrt{M})$ and is asymptotically better than our $O((M+h)r)$ where $r = O_k(\log M)$.  However, in the following comparison, we find that MoRE improves upon PEER based on the wall-clock speed in practice, with matched active FLOPs. 

\subsection{Wall-clock Comparison with PEER}
\label{appendix:peer}

A PEER layer holds
$M$ experts, each a single neuron $e_i(x) = \sigma(u_i^\top x)\, v_i$ with
$u_i, v_i \in \mathbb{R}^h$, so an expert has width $1$ and $2h$ parameters. Routing uses $H$ query
networks, each a matrix in $\mathbb{R}^{d_q \times h}$, together with two sub-key tables of shape
$\sqrt{M} \times d_q/2$. Each head splits its query in half, matches the halves against the two
tables, keeps the top $\kappa$ on each axis, and takes the top $\kappa$ of the $\kappa^2$ candidate
sums, so the layer activates $k = H\kappa$ experts per token. The
per-layer parameter counts and per-token active FLOPs (we count the query projections and the
sub-key tables as active FLOPs too) are
\begin{equation}
  P_{\text{experts}} = 2 M h, \qquad
  P_{\text{router}} = H h d_q + \sqrt{M}\, d_q, \qquad
  \tau_{\text{PEER}} = 12\, k h + 6\, H d_q\big(h + \sqrt{M}\big),
  \label{eq:peer_accounting}
\end{equation}
to be read against $3Mhs$, $(M+h)r$ and $C_e k h s + C_r (M+h) r$ for MoRE. %

For comparison, we implement PEER with reference to the public PEER-pytorch
repository,\footnote{\url{https://github.com/lucidrains/PEER-pytorch}} and change the code to make the heads share one set of product keys, following \citet{he2024mixturemillionexperts}. We compare MoRE against it at matched total parameters and matched active
FLOPs.

We match against the $(s{=}64,\, h{=}512)$ row of Table~\ref{tab:fineweb_p3_results}, where
MoRE-1792 uses $M' = 1792$, $s' = 56$, $k = 4$ and $r = 64$ over $L = 12$ layers. Matching expert
parameters gives $M = \tfrac{3}{2}M's' = 150{,}528$, and since product-key retrieval needs $M$ to be
a perfect square we take $M = 388^2 = 150{,}544$. Matching active FLOPs at $H = 8$ then gives
$d_q = 32$ and $\kappa = 56$, so $k = 448$ experts are active per token. We choose $d_q = 32$ rather
than the $d_q = h$ of the default implementation in order to match the router FLOPs.

Table~\ref{tab:peer_wallclock} measures one layer at the matched configuration. \textbf{At matched active FLOPs, a MoRE layer takes substantially less wall-clock time than a PEER layer.} Because $M$ is large for PEER, we additionally implement a fused kernel for the expert application that never materializes the gathered weights. This fused kernel runs $1.90\times$ faster on that stage, but the wall-clock ratio remains large.

\begin{table}[H]
\centering
\small
\begin{tabular}{rrrrr}
\toprule
& \multicolumn{2}{c}{Forward} & \multicolumn{2}{c}{Forward $+$ backward} \\
\cmidrule(lr){2-3}\cmidrule(lr){4-5}
Tokens per step & MoRE & PEER & MoRE & PEER \\
\midrule
$4{,}096$   & $1.06$ & $5.56$ \ ($5.2\times$)    & $2.94$  & $16.30$ \ ($5.5\times$) \\
$16{,}384$  & $1.23$ & $22.06$ \ ($18.0\times$)  & $3.01$  & $60.96$ \ ($20.3\times$) \\
$65{,}536$  & $1.93$ & $87.79$ \ ($45.6\times$)  & $4.75$  & $244.01$ \ ($51.4\times$) \\
$262{,}144$ & $4.92$ & $350.86$ \ ($71.4\times$) & $14.19$ & $974.41$ \ ($68.7\times$) \\
\bottomrule
\end{tabular}
\caption{\textbf{Single-layer wall-clock time at matched active FLOPs. The ratio grows with batch size.} Single-layer wall-clock in milliseconds at the matched configuration, on one NVIDIA B200 in
bf16, median of $20$ iterations after $5$ warm-ups.}
\label{tab:peer_wallclock}
\end{table}

Table~\ref{tab:peer_quality} reports validation perplexity for the three architectures trained for $1$B tokens on mixture dataset specified in Appendix~\ref{appendix:experimental_details}. MoE and MoRE use the
Switch load-balancing loss~\citep{fedus2022switch}. 

\begin{table}[H]
\centering
\small
\begin{tabular}{lrrr}
\toprule
Architecture & Validation PPL & Seconds per step & Training time \\
\midrule
MoE-256   & $24.26$ & $0.321$ & $20.4$ min \\
MoRE-1792 & $24.13$ & $0.405$ & $25.8$ min \\
PEER      & $24.17$ & $6.135$ & $390.1$ min \\
\bottomrule
\end{tabular}
\caption{Validation perplexity at $1$B tokens, $h = 512$, $L = 12$, four B200s,
$262{,}144$ tokens per optimizer step.}
\label{tab:peer_quality}
\end{table}

PEER and MoRE reach similar perplexity under the $1$B-token budget. The wall-clock cost differs by $15\times$: \textbf{MoRE trains in $25.8$ minutes compared to $390.1$ minutes for PEER} (Table~\ref{tab:peer_quality}).

\section{Experimental Details}
\label{appendix:experimental_details}
\paragraph{Router initialization.}
We initialize each entry of the low-rank router factors as $\mathcal{N}(0, \sigma'^2)$ with $\sigma' = \sqrt{\sigma}/r^{1/4}$ and $\sigma = 0.02$, so that the variance of the router logits matches the standard router. See Appendix~\ref{appendix:initialization} for the derivation and an empirical comparison to the default initialization.

To make the speedup by the low-rank router more apparent, we seek an efficient implementation of the Mixture-of-Experts architecture and therefore build upon ScatterMoE~\citep{tan2024scatteredmixtureofexpertsimplementation} (refer to Appendix~\ref{appendix:wallclock} for more details). We believe that further speedups are achievable on specific hardware by applying more careful optimization as in~\cite{guo2025sonicmoeacceleratingmoeio}. 

We use load balancing coefficient as 0.001 for both MoE and MoRE setups and do not have evidence that MoRE is sensitive to its choice.

\subsection{Phonebook}
\label{appendix:phonebook_details}

\paragraph{Tokenization.}
Each example is tokenized as ``\texttt{\$abcde|12345678.}'', with a five-letter lowercase name, the separator ``\texttt{|}'', and an eight-digit number, padded with ``\texttt{*}'' to a fixed sequence length of~$16$. The vocabulary has $V = 40$ symbols: $26$ lowercase letters, $10$ digits, and $4$ specials (``\texttt{\$}'' BOS, ``\texttt{.}'' EOS, ``\texttt{|}'' separator, ``\texttt{*}'' pad).

\paragraph{Architecture and sweep.}
All models are $L=6$-layer Transformers with top-$k$ routing at $k=4$ and low-rank rank $r=16$. The active budget~$\tau$ has two natural scaling axes at fixed~$h$: the hidden dimension itself and the expert width~$s$. We sweep each in turn:
(i) $h \in \{64, 128, 256, 512\}$ at fixed $s = 8$, and
(ii) $s \in \{8, 16, 32, 64\}$ at fixed $h = 64$.
For every $(h, s)$, the standard baseline takes $M = 4s$ experts, and the low-rank counterpart takes the largest~$M'$ that matches the standard active FLOPs at $r=16$, following Proposition~\ref{thm:optimal_scaling}.

\paragraph{Choice of $s'$.}
Under our convention $C_e = C_r$, Proposition~\ref{thm:optimal_scaling} prescribes the optimal MoRE expert width as $s'^{*} = s - r/(2k) = s - 2$ at $r=16, k=4$. For $s \in \{8, 16, 32, 64\}$ this gives $s'^{*} \in \{6, 14, 30, 62\}$, none of which are multiples of~$8$ -- the kernel-friendly width for ScatterMoE. We round each $s'$ to the nearest multiple of~$8$, which collapses to $s$ itself across the entire sweep, and we use the prescribed $M'^{*} = h(2ks - r)/(2r) = h(s-2)/4$ unchanged. The rounding inflates the total MLP parameter count by a factor $s/(s-2)$, which is largest at the smallest scale ($8/6 \approx 1.33\times$ at $s=8$) and shrinks to $\approx 1.03\times$ at $s=64$, it does not change the qualitative trends of Figure~\ref{fig:phonebook_capacity}.

\paragraph{Training and learning-rate selection.}
Each configuration is trained with batch size $1{,}024$ under a constant-with-warmup schedule (warmup ratio $0.05$, weight decay $0.01$, $\beta_2 = 0.95$), with early stopping once exact-string accuracy on the digit numbers crosses $0.9$. The learning rate is selected per configuration by a short calibration phase over $\{1\!\times\!10^{-4},\, 3\!\times\!10^{-4},\, 1\!\times\!10^{-3}\}$, picking the value with the highest validation character accuracy on a brief preliminary run; the calibrated rate is then used for the full $1{,}500$-epoch run.

\paragraph{Per-dot configurations.}
Tables~\ref{tab:phonebook_configs_h} and~\ref{tab:phonebook_configs_s} list,
for every dot in Figure~\ref{fig:phonebook_capacity}, the full $(M, k, h, s)$ and
$(M, k, h, s, r)$ configuration and the maximum memorizable
phonebook size at the calibrated learning rate. Table~\ref{tab:phonebook_configs_h}
covers the vary-$h$ sweep at $s=8$, and Table~\ref{tab:phonebook_configs_s}
covers the vary-$s$ sweep at $h=64$; the $(h{=}64,\,s{=}8)$ corner appears
in both as the shared anchor between the two sweeps. 

The number of experts $M$ are small for Table~\ref{tab:phonebook_configs_h} as the optimal $M = ks$ in our setup, when $s = 8$, $M = ks = 32$. 

\begin{table}[h]
\centering
\small
\begin{tabular}{lrrrrrr}
\toprule
Configuration & $M / M'$ & $k$ & $h$ & $s$ & $r$ & Max memorizable \\
\midrule
MoE  & 32  & 4 & 64  & 8 & --- & 10K  \\
MoE  & 32  & 4 & 128 & 8 & --- & 30K  \\
MoE  & 32  & 4 & 256 & 8 & --- & 100K \\
MoE  & 32  & 4 & 512 & 8 & --- & 200K \\
\midrule
MoRE & 96  & 4 & 64  & 8 & 16 & 30K  \\
MoRE & 192 & 4 & 128 & 8 & 16 & 100K \\
MoRE & 384 & 4 & 256 & 8 & 16 & 300K \\
MoRE & 768 & 4 & 512 & 8 & 16 & 1.2M \\
\bottomrule
\end{tabular}

\caption{Phonebook per-dot configurations for the \textbf{vary-$h$ sweep} at
$s=8$. $M$ ($M'$ for MoRE) is the number of experts; $k$ active per token;
$h$ hidden dim; $s$ expert intermediate dim; $r$ router rank (MoRE only).
``Max memorizable'' is the largest phonebook size at which the calibrated
training run reached the early-stopping threshold.}
\label{tab:phonebook_configs_h}
\end{table}

\begin{table}[h]
\centering
\small
\begin{tabular}{lrrrrrr}
\toprule
Configuration & $M / M'$ & $k$ & $h$ & $s$ & $r$ & Max memorizable \\
\midrule
MoE  & 32  & 4 & 64 & 8  & --- & 10K  \\
MoE  & 64  & 4 & 64 & 16 & --- & 30K  \\
MoE  & 128 & 4 & 64 & 32 & --- & 150K \\
MoE  & 256 & 4 & 64 & 64 & --- & 500K \\
\midrule
MoRE & 96  & 4 & 64 & 8  & 16 & 30K  \\
MoRE & 224 & 4 & 64 & 16 & 16 & 100K \\
MoRE & 480 & 4 & 64 & 32 & 16 & 400K \\
MoRE & 992 & 4 & 64 & 64 & 16 & 1.2M \\
\bottomrule
\end{tabular}

\caption{Phonebook per-dot configurations for the \textbf{vary-$s$ sweep} at
$h=64$. Columns as in Table~\ref{tab:phonebook_configs_h}. The $(h{=}64,\,s{=}8)$
row is shared with Table~\ref{tab:phonebook_configs_h}.}
\label{tab:phonebook_configs_s}
\end{table}

\subsection{Real-world Pretraining}
\label{appendix:fineweb_details}

\paragraph{Data mixture.}
We roughly follow the natural-language pre-training recipe of \citet[Appendix B.2]{jelassi2025mixtureparrotsexpertsimprove}, which consists of a $\approx 60.4$\,B token mixture: (i)~FineWeb-edu deduplicated~\citep{huggingfacefw_2024} (37\,B); (ii)~Cosmopedia synthetic textbooks~\citep{benallal2024cosmopedia} (14\,B); (iii)~Wikipedia (4\,B unique, looped $3\times$ for 12\,B); and (iv)~the training splits of 12 downstream knowledge-intensive and commonsense benchmarks (TriviaQA \citep{joshi2017triviaqalargescaledistantly}, Natural Questions \citep{kwiatkowski2019natural}, HotpotQA \citep{yang2018hotpotqa}, WebQuestions \citep{berant2013semantic}, ComplexWebQuestions \citep{talmor2018web}, ARC-Easy \citep{clark2018thinkyousolvedquestion}, HellaSwag \citep{zellers2019hellaswagcanmachinereally}, OpenBookQA \citep{mihaylov2018canask}, PIQA \citep{bisk2020piqa}, SciQ \citep{welbl2017crowdsourcingmultiplechoice}, SocialIQA \citep{sap2019socialiqa}, and Winogrande \citep{sakaguchi2020winogrande}), each looped $3\times$ to familiarize the model with the evaluation format \citep{allenzhu2023physics}.

Every source uses only its \texttt{train} split. We include the benchmark training splits in line with \citet{jelassi2025mixtureparrotsexpertsimprove}, since \citet{allenzhu2023physics} show that mixed training is required for successful knowledge extraction. We sample disjoint training (40\,B tokens) and 5{,}000-sequence validation sets uniformly from this pool, following the convention of \citet{jelassi2025mixtureparrotsexpertsimprove}. We conduct zero-shot evaluations only, since we did not include few-shot demonstrations in training.

\paragraph{Tokenization.}
We use the GPT-2 byte-pair-encoding tokenizer~\citep{radford2019language} with vocabulary size $V = 50{,}257$. Documents from each mixture source are concatenated and chunked to a fixed sequence length of $T = 1{,}024$ tokens, then materialized as a single \texttt{uint16} memory-mapped NumPy file for efficient reuse across runs.

\paragraph{Architecture.}
All models are $L=12$-layer Transformers with top-$k$ routing at $k=4$, head dimension~$64$, and untied input/output embeddings. We compare three reference points $(h, s) \in \{(1024, 32),\, (512, 64),\, (512, 128)\}$. For the first two, the standard baseline takes $M = 4s$ experts, and the low-rank counterpart fixes router rank $r = 64$ and takes the largest~$M'$ that matches the standard active FLOPs, following Proposition~\ref{thm:optimal_scaling}. The third point $(512, 128)$ is constructed by scaling both $s$ and $s'$ of $(512, 64)$ by~$2$ while keeping $M = 256$ and $M' = 1792$ fixed. Full configurations are listed in Table~\ref{tab:fineweb_p3_configs}.

\paragraph{Training and learning-rate selection.}
Each configuration is trained for $\approx 40$\,B tokens, with sequence length $1{,}024$ on $4$ GPUs in DDP, using AdamW~\citep{loshchilov2019decoupled} ($\beta_1 = 0.9$, $\beta_2 = 0.95$, $\epsilon = 10^{-8}$), weight decay $0.01$, max grad norm $1.0$, and a cosine learning-rate schedule with warmup ratio $0.05$, in bf16 precision. The peak (global) learning rate is selected per configuration from $\{5\!\times\!10^{-4},\, 1\!\times\!10^{-3},\, 3\!\times\!10^{-3},\, 5\!\times\!10^{-3}\}$ by best validation perplexity on a 2B-token run, which selects $1\!\times\!10^{-3}$ for every configuration. We set the router learning rate, for both $R_1$ and $R_2$ of the low-rank router, equal to the global learning rate, as we find that validation loss is not very sensitive to the router learning rate (Figure~\ref{fig:rlr_switch}).

\begin{figure}[h]
    \centering
    \includegraphics[width=\linewidth]{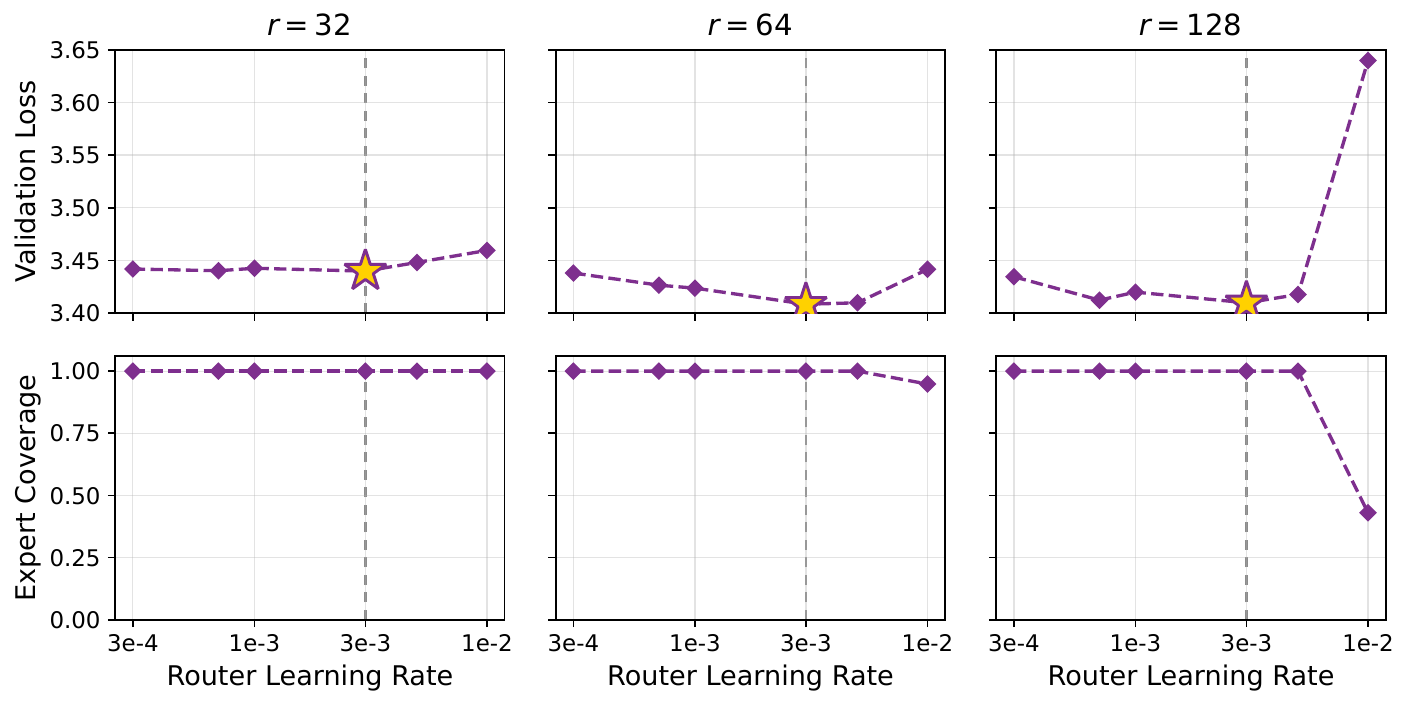}
    \caption{\textbf{Expert coverage is insensitive to the router learning rate, and the
    optimal router learning rate is the global learning rate at every rank. We therefore
    do not tune the router learning rate separately, and set it equal to the global
    learning rate.} Switch balancing loss, $M = 1792$, $s = 56$, $h = 512$, $k = 4$,
    $L = 12$, trained on 500M tokens of the mixture corpus at global learning rate
    $\eta = 3\times10^{-3}$. The dashed vertical line in every panel marks that global
    rate, and stars mark each rank's best cell. \textit{Top:} the optimum coincides with
    the global learning rate at all three ranks. The $5\times10^{-4}$ cell is omitted at
    every rank because its apparent advantage at this budget does not survive to 2B
    tokens. \textit{Bottom:} coverage holds at $1.000$ at every rank up to
    $5\times10^{-3}$, and at $r = 32$ across the entire grid. It falls only at $10^{-2}$,
    and only for $r = 64$ and $r = 128$.}
    \label{fig:rlr_switch}
\end{figure}

\begin{table}[t]
\centering
\begin{tabular}{llccccc}
\toprule
\textbf{Scale} & \textbf{Model} & $M$ & $k$ & $h$ & $s$ & $r$ \\
\midrule
\multirow{2}{*}{$s$=32, $h$=1024} & Standard          & 128  & 4 & 1024 & 32 & -- \\
                                  & Low-rank ($r$=64) & 1536 & 4 & 1024 & 24 & 64  \\
\midrule
\multirow{2}{*}{$s$=64, $h$=512}  & Standard          & 256  & 4 & 512  & 64 & -- \\
                                  & Low-rank ($r$=64) & 1792 & 4 & 512  & 56 & 64  \\
\midrule
\multirow{2}{*}{$s$=128, $h$=512} & Standard          & 256  & 4 & 512  & 128 & -- \\
                                  & Low-rank ($r$=64) & 1792 & 4 & 512  & 112 & 64  \\
\bottomrule
\end{tabular}
\vspace{6pt}
\caption{Architectural configurations for each variant in Table~\ref{tab:fineweb_p3_results}. $M$: number of experts. $k$: experts active per token. $h$: model hidden dimension. $s$: expert FFN hidden dimension. $r$: router rank (Low-rank only). All variants use the GPT-2 tokenizer. Within each Scale group, the Low-rank variant scales $M$ up while reducing $s$ slightly so that active FLOPs approximately match the Standard reference; the $(s{=}128, h{=}512)$ group is the exception, obtained by scaling $s$ and $s'$ of $(s{=}64, h{=}512)$ by~$2$ with $M$ and $M'$ fixed.}
\label{tab:fineweb_p3_configs}
\end{table}

\subsection{FLOPs Calculation and Total parameter accounting}
\label{appendix:flops_param_accounting}

For consistency with Proposition~\ref{thm:optimal_scaling}, we report \emph{active FLOPs per token} (rather than active parameters) in Table~\ref{tab:fineweb_p3_results}, since the matched-active configuration grid is selected to fix this quantity. We use the standard convention that one fused multiply--add counts as $2$~FLOPs. All $h{=}512$ models use $n_h{=}8$ attention heads; $h{=}1024$ uses $n_h{=}16$, $n_{kv}{=}4$. Head dimension is $d{=}h/n_h{=}64$ throughout. Vocabulary size is $V{=}50{,}257$ (GPT-2 BPE), and input/output embeddings are untied.

\paragraph{Active MoE FLOPs in Table~\ref{tab:fineweb_p3_results}.}
Per token and per MoE layer, the active FLOPs in training (forward and backward) are $18khs + 14Mh$ for MoE and $18khs' + 14(M'+h)r$ for MoRE. At inference (forward only), they are $6khs + 6Mh$ for MoE and $6khs' + 6(M'+h)r$ for MoRE. The constants are derived below. In the computation of the configurations, we treat $C_e \approx C_r$. Table~\ref{tab:fineweb_p3_results} reports the inference FLOPs, and Table~\ref{tab:fineweb_flops_train_inf} lists both the training and the inference FLOPs.

\begin{table}[h]
\centering
\begin{tabular}{llrr}
\toprule
$(s, h)$ & Model & Training FLOPs & Inference FLOPs \\
\midrule
\multirow{2}{*}{(32, 1024)} & MoE-128   & 4.19M & 1.57M \\
                            & MoRE-1536 & 4.06M & 1.57M \\
\midrule
\multirow{2}{*}{(64, 512)}  & MoE-256   & 4.19M & 1.57M \\
                            & MoRE-1792 & 4.13M & 1.57M \\
\midrule
\multirow{2}{*}{(128, 512)} & MoE-256   & 6.55M & 2.36M \\
                            & MoRE-1792 & 6.19M & 2.26M \\
\midrule
\multirow{2}{*}{(64, 512)}  & MoE-1024  & 9.70M & 3.93M \\
                            & MoRE-8192 & 9.72M & 3.98M \\
\bottomrule
\end{tabular}
\caption{Active MoE FLOPs per token per layer for the configurations in Table~\ref{tab:fineweb_p3_results}. Training FLOPs are $18khs + 14Mh$ for MoE and $18khs' + 14(M'+h)r$ for MoRE. Inference FLOPs are $6khs + 6Mh$ and $6khs' + 6(M'+h)r$.}
\label{tab:fineweb_flops_train_inf}
\end{table}

\paragraph{FLOPs constants in Proposition~\ref{thm:optimal_scaling}.}
The forward-pass formula above tabulates the empirical FLOPs we report. For the analytical scaling argument, Proposition~\ref{thm:optimal_scaling} writes the active-FLOP budget as $\tau = C_e \cdot k h s + C_r \cdot M h$, with $C_r r (h + M)$ in place of $C_r M h$ for the rank-$r$ low-rank router, abstracting per-operation FLOPs into two scalar constants. We follow the full-training (forward + backward) accounting of \citet[Appendix E]{krajewski2024scalinglawsfinegrainedmixture}, which charges $6$ FLOPs per pair of (active parameter, processed token) for each linear projection (forward $2$ + backward input-grad $2$ + backward weight-grad $2$). A SwiGLU expert performs three projections, giving $C_e = 3 \times 6 = 18$. The standard router charges $6$ for its linear projection plus $2$ FLOPs per pair for each of four token-routing operations (forward and backward dispatches of inputs and outputs), giving $C_r = 6 + 4 \times 2 = 14$ as in \citet{krajewski2024scalinglawsfinegrainedmixture}. The forward-only ratio is similar: $C_e^{\text{fwd}}/C_r^{\text{fwd}} = 6/6 = 1$, compared to $C_e/C_r = 18/14 \approx 1.29$ for full training. We set $C_e = C_r$, which simplifies the configurations (for example, the optimal MoE then has $M = ks$), and use the same $C_r$ for the low-rank router. Only the ratio $C_e/C_r$ enters the matching of an MoE and a MoRE configuration. The token-routing operations cost $O(kh)$ per token in ScatterMoE, so a stricter accounting charges each router only for its matrix multiplications ($C_r = 2$ at inference). Under this accounting, MoRE uses $4.5\%$ to $12.5\%$ fewer active FLOPs than its MoE counterpart in every pair of Table~\ref{tab:fineweb_p3_results}. The $\Theta(h/r)$ ratio in Proposition~\ref{thm:optimal_scaling} depends only on the leading $C_r M h$ term, so it is insensitive to the precise value of $C_r$ within the defensible range $[6, 20]$ noted by \citet{krajewski2024scalinglawsfinegrainedmixture}.

\paragraph{Total parameter count.}
Total parameters include the embedding matrix, all $M$ experts (active and inactive), the router, attention, layer norms, and the unembedding head:
\begin{equation}
P_{\text{total}} \;=\; \underbrace{2\,V\,h}_{\text{embed} + \text{LM head}} \;+\; \underbrace{h}_{\text{final LN}} \;+\; L\cdot\Big[
\underbrace{\boldsymbol{P_{\text{router}}}}_{\text{router}}
\;+\; \underbrace{\boldsymbol{3\,M\,h\,s}}_{\text{all experts}}
\;+\; \underbrace{2h^2 + 2\,h\,n_{kv}\,d}_{\text{attention}}
\;+\; \underbrace{2\,h}_{\text{2 RMSNorm}}
\Big],
\label{eq:total_params}
\end{equation}
where $P_{\text{router}} = h\,M$ for the standard router and $r\,(h+M)$ for the low-rank variant. (RMSNorm has no bias, so each contributes $h$ scale parameters.) The two bold terms are the only ones that depend on the expert configuration: once $h$ is fixed, the embedding, attention, and norm terms are also fixed, and the parameter count varies only through the router and expert terms, which are set by $M$ and $s$. The depth-independent embedding and LM head term is amortized over layers, so as the depth $L$ grows the router and the experts account for a larger fraction of the total parameters.

\section{Initialization}
\label{appendix:initialization}

The Switch Transformer~\citep{fedus2022switch} sets the router weights elementwise to $\mathcal{N}(0, \sigma^2)$ with $\sigma = 0.02$, and prior work has reported that the precise router initialization variance has limited impact on standard-router performance. The natural default for the low-rank router carries the same per-entry standard deviation over to both factors, i.e.\ each entry of $R_1 \in \mathbb{R}^{M \times r}$ and $R_2 \in \mathbb{R}^{r \times h}$ is drawn i.i.d.\ from $\mathcal{N}(0, \sigma^2)$ with the same $\sigma = 0.02$. We refer to this as the \emph{default} (or \emph{standard}) initialization. 

Empirically, however, we find that the default initialization noticeably degrades validation perplexity and shifts the optimal router learning rate significantly. We instead choose $\sigma_1, \sigma_2$ so that the variance of the router logits matches the standard baseline. Lemma~\ref{lemma:variance_matching} makes this matching condition precise.

\subsection{Setup}
We treat $x \in \mathbb{R}^h$ as a fixed nonzero input and condition on it throughout. The standard router $R \in \mathbb{R}^{M \times h}$ has entries i.i.d.\ $\mathcal{N}(0, \sigma^2)$, and the low-rank factors $R_1 \in \mathbb{R}^{M \times r}$, $R_2 \in \mathbb{R}^{r \times h}$ have independent entries i.i.d.\ $\mathcal{N}(0, \sigma_1^2)$ and $\mathcal{N}(0, \sigma_2^2)$ respectively, with $R_1$ independent of $R_2$. We denote the corresponding learning rates by $\eta$ for the standard router, and $\eta_1, \eta_2$ for $R_1, R_2$. Our goal is to derive $\sigma_1, \sigma_2$ (and, in future work, $\eta_1, \eta_2$) by matching the variance and the logit update of the standard router.

\subsection{Matching variance}

\begin{lemma}[Variance matching]
\label{lemma:variance_matching}
Under the setup above, for any fixed $x \neq 0$, $\mathrm{Var}\!\left([R_1 R_2 x]_k\right) = \mathrm{Var}\!\left([Rx]_k\right)$ for every $k \in [M]$ if and only if
\[
\sigma_1 \sigma_2 \;=\; \frac{\sigma}{\sqrt{r}}.
\]
The matching condition does not depend on $x$.
\end{lemma}

\begin{proof}
All variances below are conditional on the fixed input $x$.

For the standard router, $[Rx]_k = \sum_{j=1}^{h} R_{kj}\, x_j$ is a linear combination of i.i.d.\ $\mathcal{N}(0, \sigma^2)$ entries with deterministic coefficients $x_j$, so
\[
\mathrm{Var}\bigl([Rx]_k\bigr) \;=\; \sigma^2 \|x\|_2^2.
\]

For the low-rank router, write $z = R_2 x \in \mathbb{R}^r$, so $[R_1 R_2 x]_k = \sum_{i=1}^{r} (R_1)_{ki}\, z_i$. The same linear-combination argument applied to $R_2$ gives $\mathrm{Var}(z_i) = \sigma_2^2 \|x\|_2^2$. Since $R_1$ is independent of $R_2$ and zero-mean, each summand satisfies
\[
\mathrm{Var}\bigl((R_1)_{ki}\, z_i\bigr)
\;=\; \mathbb{E}\bigl[(R_1)_{ki}^2\bigr]\,\mathbb{E}\bigl[z_i^2\bigr]
\;=\; \sigma_1^2\, \sigma_2^2\, \|x\|_2^2,
\]
and the $r$ summands are mutually independent because they share no entry of $R_1$. Summing,
\[
\mathrm{Var}\bigl([R_1 R_2 x]_k\bigr) \;=\; r\, \sigma_1^2\, \sigma_2^2\, \|x\|_2^2.
\]
Equating the two variances gives $r\, \sigma_1^2\, \sigma_2^2\, \|x\|_2^2 = \sigma^2 \|x\|_2^2$, and since $\|x\|_2^2 > 0$ this is equivalent to $\sigma_1\sigma_2 = \sigma/\sqrt{r}$, independent of $x$.
\end{proof}

In all the experiments in this paper, we set $\sigma_1 = \sigma_2$ and $\eta_1 = \eta_2$. Combined with Lemma~\ref{lemma:variance_matching}, this gives
\[
\sigma_1 \;=\; \sigma_2 \;=\; \frac{\sqrt{\sigma}}{r^{1/4}}, \qquad \sigma = 0.02.
\]

\begin{figure}[h]
    \centering
    \includegraphics[width=\linewidth]{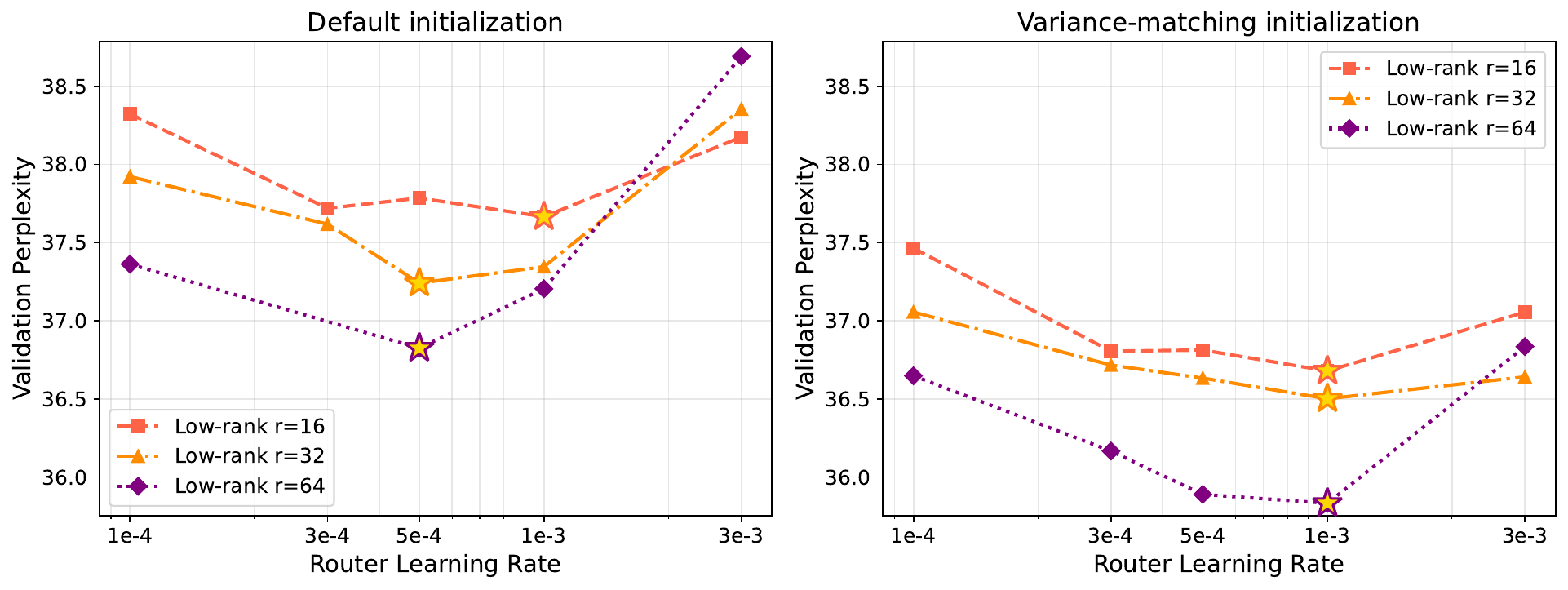}
    \caption{Router learning-rate sweep on FineWeb-edu at the optimal global learning rate $\eta = 10^{-3}$, comparing the default per-entry initialization (\textit{left}) against the variance-matching initialization of Lemma~\ref{lemma:variance_matching} (\textit{right}) for low-rank ranks $r \in \{16, 32, 64\}$. Each curve trains for $400$M FineWeb-edu tokens. The variance-matching initialization achieves lower validation perplexity across ranks, and also has stabler optimum across this range, although it might also shift leftwards when $r$ is larger.}
    \label{fig:initialization_ablation}
\end{figure}

\subsection{Discussion}

Figure~\ref{fig:initialization_ablation} presents our empirical comparison. We hypothesize that a theoretical analysis analogous to those developed for LoRA~\citep{hu2021loralowrankadaptationlarge} and LoRA+~\citep{hayou2024loraplus,thinkingmachines2025loraregret} could provide a principled approach to selecting both the initialization and learning rates for routers in the low-rank setting, potentially eliminating the need to tune $\eta_1$ and $\eta_2$ individually for each rank. Although a key distinction from LoRA is that LoRA initializes its output to zero to preserve finetuning performance, whereas initializing the router's output to zero is undesirable in our context: doing so would impede the ability to break ties when selecting the top $k$ experts.

In a complementary direction, \citet{jiang2026hyperparametertransfermixtureofexpertlayers} derive a hyperparameter-transfer parameterization for MoE routing weights via a dynamical mean-field analysis; we believe their framework can be adapted to the low-rank router setting as well.

\section{Efficiency}
\label{appendix:wallclock}

For the MoE architecture, we build upon ScatterMoE~\citep{tan2024scatteredmixtureofexpertsimplementation} and implement a new fused Triton kernel for the router (refer to Figure~\ref{fig:flowchart} for the detailed operations in the forward pass). We believe that further speedups are achievable on specific hardware by applying more careful optimization as in~\cite{guo2025sonicmoeacceleratingmoeio}.

\paragraph{Why we build on ScatterMoE.}
Figure~\ref{fig:flowchart} illustrates the operations in one MoE layer forward pass. Beyond the router, the expert MLP involves token dispatch (sorting by expert assignment), two grouped linear transformations with a SwiGLU activation, and a scatter-back step that combines expert outputs. Each of these steps is a potential source of memory traffic. For example, a naive grouped MoE implementation pays $O(Nkh)$ memory traffic to materialize the gathered tokens (sorted by expert assignment) before the per-expert grouped GEMM, and pays the same again on the scatter-back. 

ScatterMoE~\citep{tan2024scatteredmixtureofexpertsimplementation} fuses scatter-gather into the per-expert grouped matmul, removing both round-trips and producing a kernel whose only HBM traffic is the input $X$, the weight stack, and the output $Y$. With this dispatch overhead eliminated, the standard router's $O(NhM)$ matmul is the only remaining cost that grows linearly with $M$. Our low-rank router is designed precisely to address this last bottleneck, while reusing the rest of ScatterMoE's machinery unchanged.

\subsection{GPU Memory Hierarchy and Related Analysis}
\label{appendix:gpu_memory}

\paragraph{Bandwidth tiers.} Modern GPU architectures expose a multi-level memory hierarchy with sharply varying bandwidth. Off-chip high-bandwidth memory (HBM) provides the largest capacity but the lowest bandwidth. On-chip, the L2 cache offers an intermediate level, while each streaming multiprocessor (SM) contains a small pool of SRAM (registers and shared memory) with the highest bandwidth. Any data that does not fit in on-chip memory must be materialized to HBM, incurring an order-of-magnitude bandwidth penalty.

Our fusing avoids writing the full $(N, M)$ logits and the intermediate tensors of low-rank router to HBM, and computes the Top-$K$ operation in SRAM, which is much faster than the HBM round-trip.

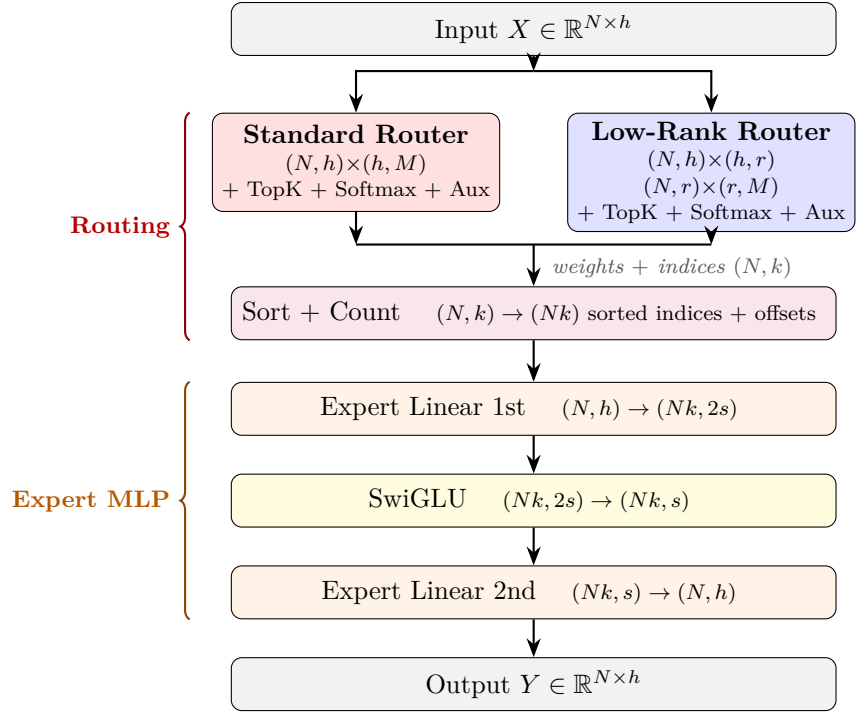
\begin{figure}[t]
    \centering
    \begin{tikzpicture}[
        node distance=0.55cm,
        block/.style={draw, rounded corners, minimum width=8cm, minimum height=0.7cm,
                      align=center, font=\small},
        routerblock/.style={draw, rounded corners, minimum width=3.8cm, minimum height=1.3cm,
                           align=center, font=\small, inner sep=3pt},
        arrow/.style={-{Stealth[length=2.5mm]}, thick},
        lbl/.style={font=\scriptsize\itshape, text=gray!70!black},
        grp/.style={font=\footnotesize\bfseries, anchor=east, inner sep=2pt},
    ]
    \node[block, fill=gray!10] (input) {Input $X \in \mathbb{R}^{N \times h}$};

    \node[routerblock, fill=red!12, below=0.75cm of input, xshift=-2.35cm] (standard) {
        \textbf{Standard Router}\\[-2pt]
        {\scriptsize $(N,h){\times}(h,M)$}\\[-2pt]
        {\scriptsize $+$ TopK $+$ Softmax $+$ Aux}
    };
    \node[routerblock, fill=blue!12, below=0.75cm of input, xshift=2.35cm] (lr) {
        \textbf{Low-Rank Router}\\[-2pt]
        {\scriptsize $(N,h){\times}(h,r)$}\\[-2pt]
        {\scriptsize $(N,r){\times}(r,M)$}\\[-2pt]
        {\scriptsize $+$ TopK $+$ Softmax $+$ Aux}
    };

    \coordinate (routermid) at ($(standard.south)!0.5!(lr.south)$);
    \node[block, fill=purple!10, below=0.85cm of routermid] (sort) {
        Sort $+$ Count \quad {\scriptsize $(N,k) \to (Nk)$ sorted indices $+$ offsets}
    };

    \node[block, fill=orange!10, below=0.55cm of sort] (s2s1) {
        Expert Linear 1st \quad {\scriptsize $(N,h) \to (Nk, 2s)$}
    };
    \node[block, fill=yellow!15, below=0.5cm of s2s1] (glu) {
        SwiGLU \quad {\scriptsize $(Nk, 2s) \to (Nk, s)$}
    };
    \node[block, fill=orange!10, below=0.5cm of glu] (s2s2) {
        Expert Linear 2nd \quad {\scriptsize $(Nk, s) \to (N, h)$}
    };
    \node[block, fill=gray!10, below=0.5cm of s2s2] (output) {Output $Y \in \mathbb{R}^{N \times h}$};

    \coordinate (inbus) at ([yshift=-0.2cm]input.south);
    \coordinate (rtrbus) at ($(routermid)!0.35!(sort.north)$);
    \coordinate (busL) at (standard.south |- rtrbus);
    \coordinate (busR) at (lr.south |- rtrbus);
    \draw[arrow] (input.south) -- (inbus);
    \draw[arrow] (inbus) -| (standard.north);
    \draw[arrow] (inbus) -| (lr.north);
    \draw[arrow] (standard.south) -- (busL);
    \draw[arrow] (lr.south) -- (busR);
    \draw[thick] (busL) -- (busR);
    \draw[arrow] (rtrbus) -- (sort.north);
    \node[lbl, anchor=west] at ([xshift=3pt]$(rtrbus)!0.55!(sort.north)$) {weights $+$ indices $(N,k)$};
    \draw[arrow] (sort) -- (s2s1);
    \draw[arrow] (s2s1) -- (glu);
    \draw[arrow] (glu) -- (s2s2);
    \draw[arrow] (s2s2) -- (output);

    \def\bracex{-4.55}
    \draw[decorate, decoration={brace, amplitude=3.5pt, mirror}, thick, red!55!black]
        (\bracex,0 |- standard.north) -- (\bracex,0 |- sort.south)
        node[midway, grp, text=red!70!black, xshift=-5pt] {Routing};
    \draw[decorate, decoration={brace, amplitude=3.5pt, mirror}, thick, orange!55!black]
        (\bracex,0 |- s2s1.north) -- (\bracex,0 |- s2s2.south)
        node[midway, grp, text=orange!70!black, xshift=-5pt] {Expert MLP};
    \end{tikzpicture}
    \caption{Operations in the forward pass of one MoE layer. The first expert linear projects to $2s$ because SwiGLU uses separate gate and up projections: $\mathrm{SiLU}(Xw_{\mathrm{gate}}) \odot Xw_{\mathrm{up}}$. Expert linear layers use fused scatter-gather kernels~\citep{tan2024scatteredmixtureofexpertsimplementation} that eliminate $O(Nkh)$ memory traffic from token dispatch, making the standard router the dominant remaining overhead and motivating the low-rank factorization.}
    \label{fig:flowchart}
\end{figure}

\paragraph{Comparison to other fused-MoE kernels.}
Existing efficient MoE kernels,  ScatterMoE~\citep{tan2024scatteredmixtureofexpertsimplementation}, MegaBlocks~\citep{gale2023megablocks},  MoMoE~\citep{costin2025momoe} and SonicMoE~\citep{guo2025sonicmoeacceleratingmoeio}, do top-$k$ computation on the materialized $(N, M)$ logits on the HBM and have to pay expensive read and write costs. Our kernel fuses the router with top-$k$ and do sequential top-$k$ operation on each tile, this avoid expensive memory operations between HBM and SRAM, but the cost would also scale with $k$, making large $k$ infeasible in our setup.

\subsection{Top-$K$ Algorithm Analysis}
\label{appendix:topk_analysis}

Our implementation is faster than \texttt{torch.topk} when $K < 16$ and is common across frontier model choices (Table~\ref{table:frontier_model_comparison}). But it is slower than \texttt{torch.topk} for $K \geq 16$.  For our fused kernel, the per-token cost is $\Theta(K \cdot M) + \Theta(K^2 \cdot M / \texttt{BLOCK\_M})$.

For each of the $M / \texttt{BLOCK\_M}$ tiles, we run $K$ iterations to find the $K$ largest values and insert them to the $K$-element running buffer. The $\Theta(K \cdot M)$ term aggregates the $K$ iterations of max/argmax reductions over the $\texttt{BLOCK\_M}$ new candidates, across $M / \texttt{BLOCK\_M}$ tiles. The $\Theta(K^2 \cdot M / \texttt{BLOCK\_M})$ term is for computing the minimum over the $K$-element buffer that happen inside the same $K$ iterations.

We suspect that a more carefully engineered fused topk along the lines of SonicMoE~\citep{guo2025sonicmoeacceleratingmoeio} (Section~D and Appendix~F.4), could outperform \texttt{torch.topk} at larger $K$ and close this gap, which we leave to future work.

\subsection{Expert Parallelism}
\label{appendix:ep}
We explained how MoRE can be implemented efficiently on a single GPU, here we describe how we implement it under expert parallelism, a common setting in large-scale MoE training. We consider there are $E$ GPUs where each GPU contains $M/E$ experts, for simplicity we assume all GPUs are inside one node, and we defer inter-node communication analysis to future work.

Following popular design choices~\citep{lepikhin2020gshard, fedus2022switch, yan2026scalabletrainingmixtureofexpertsmodels}, we replicate non-expert components including router in all GPUs, which our fusing technique still carries over as every GPU has a local router. We could also shard the router across GPUs, i.e.\ each GPU holds $M/E$ experts and a $1/E$ fraction of the router, by computing a top-$k$ locally inside each GPU and then a global top-$k$ across GPUs, which is a hierarchical top-$k$ which we leave to future work. It has the same scaling argument as in the main paper.

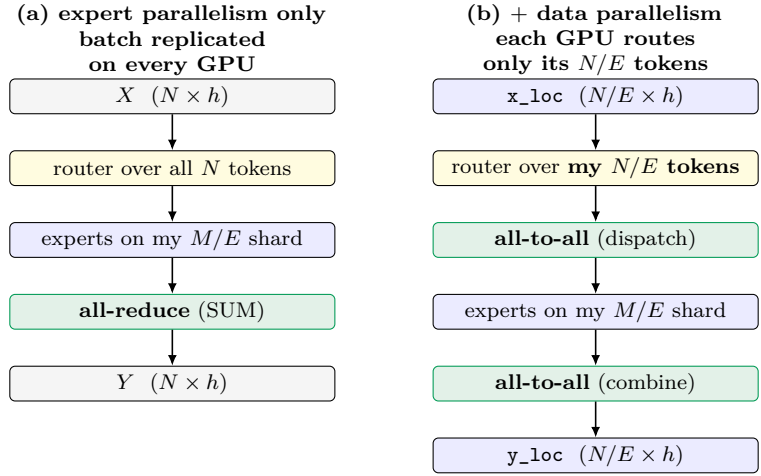
\begin{figure}[h]
\centering
\begin{tikzpicture}[
  font=\scriptsize,
  b/.style={draw, rounded corners=2pt, minimum width=4.3cm, minimum height=0.46cm,
            align=center, inner sep=2pt},
  wire/.style={draw=ForestGreen, fill=ForestGreen!10, rounded corners=2pt,
               minimum width=4.3cm, minimum height=0.46cm, align=center, inner sep=2pt},
  d/.style={-{Latex[length=1.4mm,width=1.2mm]}, semithick}
]
\def\ca{-2.8}
\def\cb{2.8}
\node[font=\scriptsize\bfseries, align=center, text width=4.6cm] at (\ca,0.75)
  {(a) expert parallelism only\\ batch replicated on every GPU};
\node[font=\scriptsize\bfseries, align=center, text width=4.6cm] at (\cb,0.75)
  {(b) $+$ data parallelism\\ each GPU routes only its $N/E$ tokens};

\node[b, fill=gray!8] (a1) at (\ca,0)    {$X$ \ $(N \times h)$};
\node[b, fill=yellow!15] (a2) at (\ca,-0.95) {router over all $N$ tokens};
\node[b, fill=blue!8] (a3) at (\ca,-1.90) {experts on my $M/E$ shard};
\node[wire] (a4) at (\ca,-2.85) {\textbf{all-reduce} (SUM)};
\node[b, fill=gray!8] (a5) at (\ca,-3.80) {$Y$ \ $(N \times h)$};
\foreach \i/\j in {1/2,2/3,3/4,4/5} {\draw[d] (a\i) -- (a\j);}

\node[b, fill=blue!8] (b1) at (\cb,0)    {\texttt{x\_loc} \ $(N/E \times h)$};
\node[b, fill=yellow!15] (b2) at (\cb,-0.95) {router over \textbf{my $N/E$ tokens}};
\node[wire] (b3) at (\cb,-1.90) {\textbf{all-to-all} (dispatch)};
\node[b, fill=blue!8] (b4) at (\cb,-2.85) {experts on my $M/E$ shard};
\node[wire] (b5) at (\cb,-3.80) {\textbf{all-to-all} (combine)};
\node[b, fill=blue!8] (b6) at (\cb,-4.75) {\texttt{y\_loc} \ $(N/E \times h)$};
\foreach \i/\j in {1/2,2/3,3/4,4/5,5/6} {\draw[d] (b\i) -- (b\j);}
\end{tikzpicture}
\caption{The two implementations measured in Table~\ref{tab:ep_layer_time_fwd}. Both shard the experts
$M/E$ per GPU and hold a full copy of the router weights. In (a) the batch is replicated, so every
GPU scores all $N$ tokens and one all-reduce sums the partial outputs; the routing work is therefore
duplicated $E$ times. In (b) each GPU keeps only its own $N/E$ tokens and scores only those, so the
routing work is divided by $E$, at the cost of shipping each token to the GPU owning its selected
expert and back with two all-to-alls.}
\label{fig:ep_dataflow}
\end{figure}

\subsubsection{Scaling under Expert Parallelism}
\label{subsec:scaling-ep}

Here we show that the $\Theta(h/r)$ advantage of Section~\ref{sec:scaling} still holds under expert
parallelism, and that replicating the router makes the routing bottleneck more severe.

\paragraph{Active cost with a replicated router.} Because the router is replicated, every device
scores all tokens against all experts and computes the global top-$k$ locally. Each token-expert
pair is therefore still computed exactly once across the system, while the router runs $E$ times.
Per-token active FLOPs across GPUs become
\begin{equation}
  \tau \;=\;
  \begin{cases}
    C_e\, k h s + E\, C_r\, M h & \text{(standard router)},\\
    C_e\, k h s + E\, C_r\, (M+h)\,r & \text{(low-rank router)}.
  \end{cases}
  \label{eq:tau_ep}
\end{equation}
Equivalently, since each device holds $M/E$ experts and therefore computes $k/E$ of the
token-expert pairs in expectation, the active FLOPs \emph{per GPU} are $C_e (k/E) h s + C_r M h$ for
the standard router. Either way the factor $E$ falls on the router term and not the expert term, so
routing occupies a larger share of the budget than on a single device: replicating the router
sharpens exactly the bottleneck that MoRE is designed to remove.

\paragraph{The ratio is unchanged.} Applying the budget allocation of
Proposition~\ref{thm:optimal_scaling} to \eqref{eq:tau_ep} gives
\[
M^{*} \;=\; \frac{\tau}{2 C_r E h}, \qquad
M' \;=\; \frac{\tau - E\, C_r h r}{2 C_r E r}, \qquad
\frac{M'}{M^{*}} \;=\; \frac{h}{r}\left(1 - \frac{E\, C_r h r}{\tau}\right).
\]
The factor $E$ enters both configurations identically and cancels, so the $\Theta(h/r)$ advantage
carries over. The correction term is close to $1$ whenever $\tau > E\, C_r M h \gg E\, C_r r h$,
that is whenever $M \gg r$, which is the same condition as on a single device.

Table~\ref{tab:ep_layer_time_fwd} measures both implementations of Figure~\ref{fig:ep_dataflow} while sweeping $M$ from $2{,}048$ to $16{,}384$. The standard MoE router
dominates the layer time under expert parallelism, and communication does not, so replacing it
with the low-rank router of MoRE speeds the layer up by up to $3.1\times$.

\begin{table}[h]
\centering\small
\begin{tabular}{rr rr rr rr r}
\toprule
& & \multicolumn{2}{c}{Router} & \multicolumn{2}{c}{Shared phases} & \multicolumn{2}{c}{Layer total} & \\
\cmidrule(lr){3-4}\cmidrule(lr){5-6}\cmidrule(lr){7-8}
$E$ & $M$ & \textbf{MoE} & \textbf{MoRE} & experts & comm. & MoE & MoRE & speedup \\
\midrule
\multicolumn{9}{l}{\itshape (a) expert parallelism only} \\
2 & 2{,}048 & \textbf{8.20} & \textbf{1.68} & 3.86 & 0.61 & 12.68 & 6.15 & $2.1\times$ \\
2 & 8{,}192 & \textbf{27.90} & \textbf{6.09} & 4.15 & 0.61 & 32.66 & 10.85 & $3.0\times$ \\
2 & 16{,}384 & \textbf{45.30} & \textbf{11.95} & 4.41 & 0.62 & 50.33 & 16.98 & $3.0\times$ \\
4 & 2{,}048 & \textbf{9.01} & \textbf{1.68} & 3.81 & 0.96 & 13.77 & 6.45 & $2.1\times$ \\
4 & 8{,}192 & \textbf{29.01} & \textbf{6.09} & 4.05 & 0.86 & 33.93 & 11.00 & $3.1\times$ \\
4 & 16{,}384 & \textbf{45.34} & \textbf{11.95} & 4.25 & 0.79 & 50.38 & 17.00 & $3.0\times$ \\
\midrule
\multicolumn{9}{l}{\itshape (b) $+$ data parallelism, each GPU routes its own $N/E$ tokens} \\
2 & 2{,}048 & \textbf{4.25} & \textbf{0.97} & 8.59 & 4.12 & 16.96 & 13.68 & $1.2\times$ \\
2 & 8{,}192 & \textbf{13.04} & \textbf{3.30} & 8.70 & 4.42 & 26.17 & 16.43 & $1.6\times$ \\
2 & 16{,}384 & \textbf{22.81} & \textbf{6.60} & 9.08 & 4.43 & 36.32 & 20.11 & $1.8\times$ \\
4 & 2{,}048 & \textbf{2.31} & \textbf{0.65} & 4.73 & 2.23 & 9.27 & 7.61 & $1.2\times$ \\
4 & 8{,}192 & \textbf{6.73} & \textbf{1.93} & 4.75 & 2.22 & 13.70 & 8.89 & $1.5\times$ \\
4 & 16{,}384 & \textbf{10.55} & \textbf{3.62} & 4.85 & 2.18 & 17.59 & 10.66 & $1.7\times$ \\
\bottomrule
\end{tabular}
\caption{\textbf{The standard MoE router is the bottleneck of one layer under expert parallelism and
dominates communication, and MoRE achieves a notable speedup by replacing it with a low-rank
router.} Forward pass of one MoE layer. The \emph{experts} column includes the local dispatch
bookkeeping, which moves nothing between devices: it turns expert ids into a sorted pair list, and
in (b) it also gathers the rows that each of the other GPUs will receive. \emph{comm.} is the
collectives themselves, a single all-reduce of the $(N, h)$ output in (a) against two all-to-alls in
(b), the second of which also carries the $(M)$ balancing statistic. All runs use the Switch
load-balancing loss. $h = 512$, $s = 64$, $k = 4$, $r = 64$, \texttt{bf16}, $N = 262{,}144$, mean of
$30$ trials on B200 GPUs in one node.}
\label{tab:ep_layer_time_fwd}
\end{table}

{
\subsection{Decoding}
\label{appendix:decoding}
In terms of decoding, it is memory-bound rather than compute-bound, since each step processes a single token per sequence and the cost is dominated by reading the weights and the KV cache ~\citep{shazeer2019fast, pope2023efficiently}. 

We initially found the decoding speedup to degrade sharply as the expert count $M$ grows, which we found the reason is the \texttt{scatter2scatter} kernel of ScatterMoE~\citep{tan2024scatteredmixtureofexpertsimplementation}: it loops over the entire contiguous range of expert ids the block spans from the smallest to the largest. A decoding step carries so few tokens whose expert ids are scattered across the pool, and the loop then walks almost the entire expert range, which makes decoding significantly slower at large $M$. 

We therefore modify the kernel to step directly from one expert present in the block to the next and make it significantly faster. Table~\ref{tab:decode_table1} and Figure~\ref{fig:decode_heatmap_fixed} show the resulting decode times.

\begin{figure}[p]
    \centering
    \includegraphics[width=\linewidth]{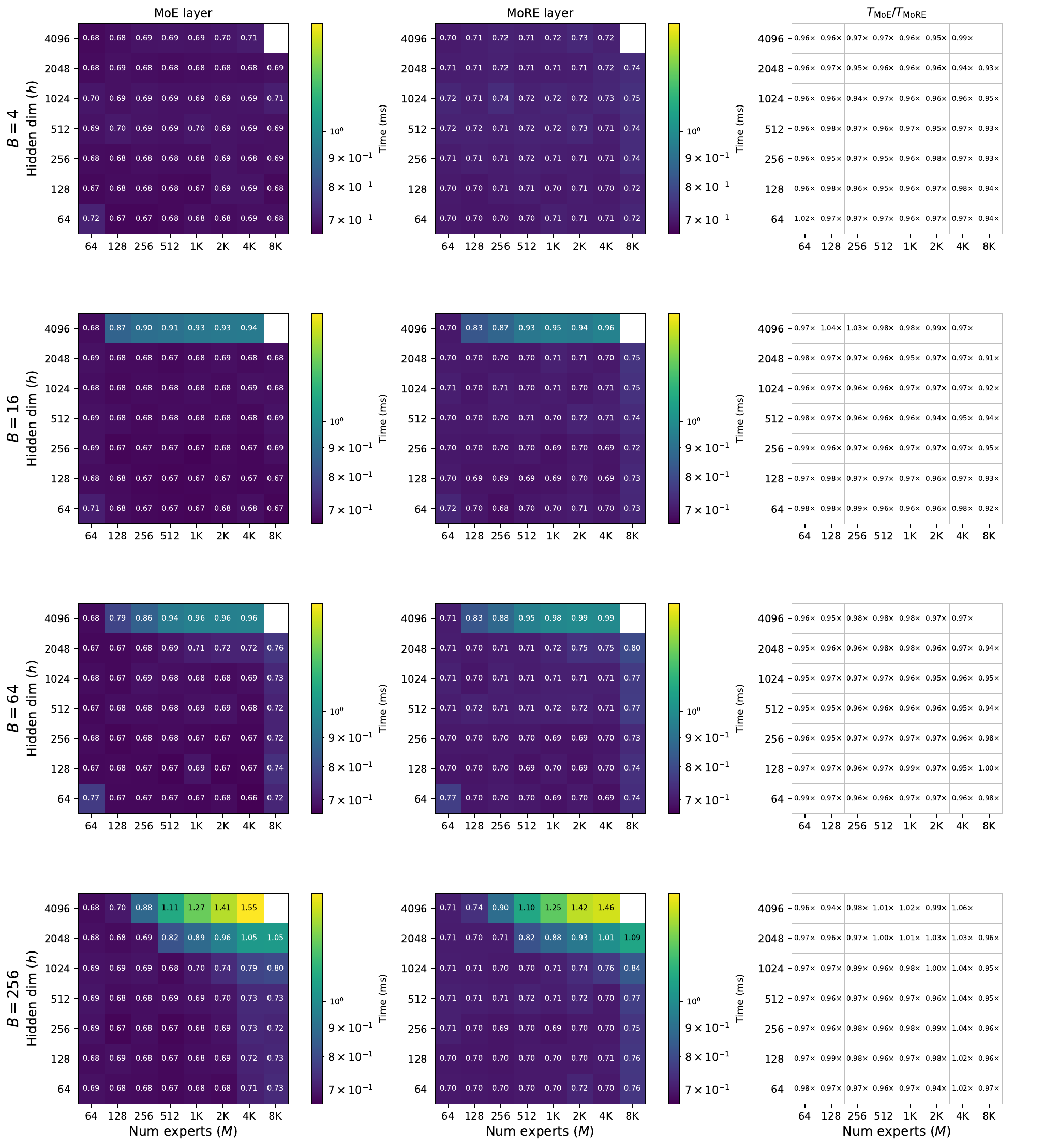}
    \caption{\textbf{MoRE has comparable decoding time with MoE for the same configuration.}
    Decode wall-clock of one MoE layer on the $(h, M)$ grid of
    Figure~\ref{fig:prefill_heatmap}, one row per decode batch size $B$.
    \textbf{Left:}~MoE layer, standard router. \textbf{Middle:}~MoRE layer, fused low-rank
    router. \textbf{Right:}~the ratio of the two. Both arms use the dispatch of this section.
    One NVIDIA B200 in \texttt{bf16}, $s{=}64$, $k{=}4$, $r{=}64$, one token per sequence,
    median of ten timed calls after five warm-ups. The $h{=}4096$, $M{=}8192$ cell is blank
    because the expert stride overflows the \texttt{int32} offsets ScatterMoE uses.}
    \label{fig:decode_heatmap_fixed}
\end{figure}

\begin{table}[!t]
    \centering
    \small
    \begin{tabular}{llrrrrr}
        \toprule
        Scale & Model & \#Experts & $B{=}4$ & $B{=}16$ & $B{=}64$ & $B{=}256$ \\
        \midrule
        $s{=}32$, $h{=}1024$ & MoE & 128 & 17.9 & 17.9 & 17.9 & 18.2 \\
         & MoRE ($r{=}64$) & 1,536 & 18.5 & 18.3 & 18.4 & 18.6 \\
        \midrule
        $s{=}64$, $h{=}512$ & MoE & 256 & 18.0 & 17.8 & 17.9 & 18.1 \\
         & MoRE ($r{=}64$) & 1,792 & 18.5 & 18.5 & 18.4 & 18.6 \\
        \midrule
        $s{=}128$, $h{=}512$ & MoE & 256 & 17.9 & 17.8 & 17.9 & 18.0 \\
         & MoRE ($r{=}64$) & 1,792 & 18.5 & 18.3 & 18.5 & 18.5 \\
        \bottomrule
    \end{tabular}
    \caption{\textbf{Decoding: MoRE matches MoE in wall-clock time even with a much larger number of experts.} Per-step decode wall-clock (ms) for MoE/MoRE pairs used in our pretraining experiments with matched active FLOPs (Table~\ref{tab:fineweb_p3_results}), across decode batch size $B$. Each MoRE row holds $6$--$12\times$ the experts of the MoE it is matched against, and stays within $3\%$ of it at every scale and batch size. One NVIDIA B200 in \texttt{bf16}, $12$ layers, $k{=}4$, prefill $128$, median of the last six of twelve decode steps. Both use the expert dispatch of this section.}
    \label{tab:decode_table1}
\end{table}

}

\FloatBarrier
\subsection{Wall-clock measurements for configurations in Table~\ref{tab:fineweb_p3_results}}
\label{appendix:train_time}

At inference, which runs only the forward pass, MoRE prefills as fast as its MoE
counterpart or faster (Table~\ref{tab:prefill_time}). Training is the one place where the
extra parameters cost time. Table~\ref{tab:train_time} reports the training step time for
the same pairs, measured over the $40$\,B-token runs: MoRE costs $15$ to $42\%$ more per
step, although it has $5$ to $6$ times as many parameters and $7$ to $12$ times as many
experts. 

Since the two models have the same active FLOPs, we attribute the extra time to
work whose cost grows with the number of parameters and experts rather than with active
FLOPs. First, the AdamW optimizer updates every parameter and its first and second
moments at every step, and gradient clipping and the data-parallel gradient all-reduce
also process every parameter. Second, the backward pass computes a weight gradient for
every expert by summing over the tokens routed to that expert. We expect that storing the
AdamW states and the backward gradients in lower precision could help close this gap.

\begin{table}[H]
    \centering
    \small
    \begin{tabular}{llrrrr}
        \toprule
        Scale & Model & \#Experts & Total params & Training (ms) & \makecell{Training time\\ratio} \\
        \midrule
        $s{=}32$, $h{=}1024$ & MoE & 128 & 0.29\,B & 241.0 & \\
         & MoRE ($r{=}64$) & 1,536 & 1.50\,B & 312.5 & 1.30$\times$ \\
        \midrule
        $s{=}64$, $h{=}512$ & MoE & 256 & 0.36\,B & 165.0 & \\
         & MoRE ($r{=}64$) & 1,792 & 1.91\,B & 221.1 & 1.34$\times$ \\
        \midrule
        $s{=}128$, $h{=}512$ & MoE & 256 & 0.66\,B & 182.7 & \\
         & MoRE ($r{=}64$) & 1,792 & 3.76\,B & 259.3 & 1.42$\times$ \\
        \midrule
        $s{=}64$, $h{=}512$ & MoE & 1,024 & 1.27\,B & 313.8 & \\
         & MoRE ($r{=}64$) & 8,192 & 7.92\,B & 360.6 & 1.15$\times$ \\
        \bottomrule
    \end{tabular}
    \caption{\textbf{MoRE costs $15$ to $42\%$ more per training step than the MoE it matches in active FLOPs.} Training wall-clock (ms) for the pairs of Table~\ref{tab:fineweb_p3_results}, measured over the $40$\,B-token runs themselves, on $4{\times}$NVIDIA B200. Prefill time for the same pairs is in Table~\ref{tab:prefill_time}.}
    \label{tab:train_time}
\end{table}

In Table~\ref{tab:same_config}, we also compare MoRE with a standard MoE of the same
configuration, that is, with the same number and width of experts and therefore about
the same total parameters. We find MoE trains $1.10$ to $1.23\times$ slower than MoRE. This implies that the extra training time of MoRE in
Table~\ref{tab:train_time} comes from its larger total number of parameters, not from the low rank.

\begin{table}[H]
    \centering
    \small
    \begin{tabular}{llrrr}
        \toprule
        Configuration & Model & Total params & Training (ms) & Ratio \\
        \midrule
        $M{=}1{,}536$, $s{=}24$, $h{=}1024$ & MoRE ($r{=}64$) & 1.50\,B & 312.5 & \\
         & MoE & 1.51\,B & 344.2 & 1.10$\times$ \\
        \midrule
        $M{=}1{,}792$, $s{=}56$, $h{=}512$ & MoRE ($r{=}64$) & 1.91\,B & 221.1 & \\
         & MoE & 1.92\,B & 249.6 & 1.13$\times$ \\
        \midrule
        $M{=}1{,}792$, $s{=}112$, $h{=}512$ & MoRE ($r{=}64$) & 3.76\,B & 259.3 & \\
         & MoE & 3.77\,B & 289.6 & 1.12$\times$ \\
        \midrule
        $M{=}8{,}192$, $s{=}52$, $h{=}512$, EP & MoRE ($r{=}64$) & 7.92\,B & 383.8 & \\
         & MoE & 7.96\,B & 472.2 & 1.23$\times$ \\
        \bottomrule
    \end{tabular}
    \caption{\textbf{MoRE and MoE at the same configuration: MoE is $1.10$ to $1.23\times$ slower than MoRE.} Training step time (ms) of each MoRE model of Table~\ref{tab:train_time} and of a standard MoE  with the same number $M$ and width $s$ of experts, on $4{\times}$NVIDIA B200. The last pair is trained with expert parallelism (EP).}
    \label{tab:same_config}
\end{table}

\FloatBarrier

\section{Pseudocode}
\label{appendix:pseudocode}

\providecommand{\ch}[1]{\textcolor{purple}{#1}}

We give Triton-style pseudocode for the two fused router kernels. Neither
materializes the $(N, M)$ logits tensor in HBM. The inference kernel returns the
top-$k$ selection alone. The training kernel also returns the statistic that the
load-balancing loss needs, and pays a second sweep over the expert dimension for
it.

\paragraph{Notation.} Write the router logits as
\begin{equation}
  \ell \;=\; X R_2^\top R_1^\top \;\in\; \mathbb{R}^{N \times M},
  \qquad
  \ell_{n,j} \in \mathbb{R}
  \label{eq:app_logits}
\end{equation}
for token $n \in [N]$ and expert $j \in [M]$, where $[M] = \{1, \dots, M\}$, with
$R_2 \in \mathbb{R}^{r \times h}$ and $R_1 \in \mathbb{R}^{M \times r}$ as in
Section~\ref{sec:preliminaries}. Both kernels evaluate \eqref{eq:app_logits} one
$B_N \times B_M$ tile at a time, so the $(N, M)$ logits exist only in on-chip
SRAM. An overbar denotes the gradient of the
total loss with respect to that tensor and always has the same shape as the
tensor: $\bar A = \partial\mathcal{L}/\partial A$.

\subsection{Inference (prefilling)}
\label{appendix:inference_pseudocode}

At inference the router must name the $k$ experts each token will use and supply
the $k$ logits that produce the gate weights. There is no loss, so nothing reads
a softmax over all $M$ experts. Algorithm~\ref{alg:lr_inf}
therefore walks the expert dimension once, keeping a running top-$k$ buffer in
SRAM, and writes $2Nk$ numbers instead of $NM$. In Figure~\ref{fig:fuse_standard_moe}, we fuse the standard MoE using a similar way, but does not have intermediate $z$ which has size in terms of $r$. This makes the complexity of fusing MoE scale with $h$ instead with $r$, and makes the fusing benefit for MoRE diminishes at large $h$. 

\begin{figure}[H]
\centering
\resizebox{0.92\textwidth}{!}{%
\begin{tikzpicture}[
  font=\small,
  bx/.style={draw, rounded corners=2pt, align=center, inner sep=4pt, minimum height=8mm},
  ker/.style={bx, fill=blue!6, draw=blue!45},
  dat/.style={bx, fill=black!4},
  ar/.style={-{Latex[length=2mm]}, thick},
  node distance=10mm,
]
\node[dat] (X) {$X$\\$(N,h)$};
\node[ker, right=of X, minimum height=18mm, text width=42mm] (tile)
  {$z = X R_2^\top$ in SRAM\\[1pt] tile $Y = z\,R_1[\mathrm{blk}]^\top$\\[1pt]
   one sweep: running top-$k$};
\node[dat, right=10mm of tile] (tv) {TopVals, TopIdxs\\$(N,k)$};
\node[dat, right=10mm of tv, text width=26mm] (gw) {softmax over\\the $k$ scores\\gate weights $(N,k)$};
\node[dat, right=10mm of gw, text width=24mm] (out) {combine the $k$\\expert outputs};
\draw[ar] (X) -- (tile);
\draw[ar] (tile) -- (tv);
\draw[ar] (tv) -- (gw);
\draw[ar] (gw) -- (out);
\node[below=3mm of tile, font=\footnotesize, text=blue!55!black] {one Triton kernel};
\end{tikzpicture}}
\caption{\textbf{Flowchart for inference (prefilling) for low-rank router.} The
blue block is one Triton kernel, holding one $B_N \times B_M$ tile of $\ell$ at a
time, so neither the $(N, r)$ intermediate $z$ nor the $(N, M)$ logits reaches
HBM. TopIdxs names the $k$ experts to run for each token; a softmax over the $k$
entries of TopVals gives the weights that combine their outputs. Nothing here
depends on all $M$ experts, which is why one sweep is enough.}
\label{fig:app_flow_inf}
\end{figure}
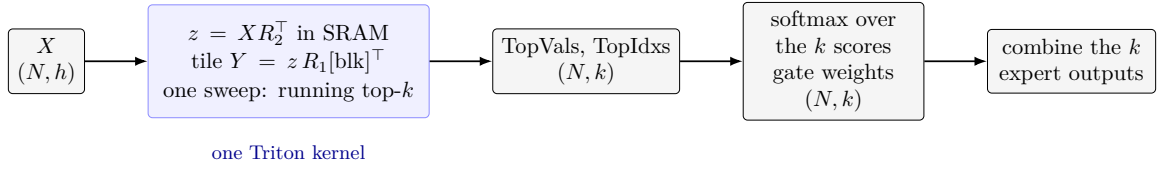

\begin{algorithm}[H]
\caption{Fused low-rank router + Top-$k$ (inference, prefilling)}
\label{alg:lr_inf}
\begin{algorithmic}[1]
\Require Input $X \in \mathbb{R}^{N \times h}$, router weights $R_1 \in \mathbb{R}^{M \times r}$, $R_2 \in \mathbb{R}^{r \times h}$, top-$k$ size $k$, block sizes $B_N, B_M, B_H$
\Ensure $\mathrm{TopVals} \in \mathbb{R}^{N \times k}$, $\mathrm{TopIdxs} \in \mathbb{Z}^{N \times k}$ in HBM
\For{each token block $b \in \{0, \ldots, \lceil N / B_N \rceil - 1\}$ \textbf{in parallel}}
  \LComment{Phase 1: $z = X_{\mathrm{block}} R_2^\top$ accumulated in SRAM}
  \State $z \gets 0$ \Comment{$(B_N, r)$ in \texttt{fp32} registers}
  \For{each hidden tile $i \in \{0, B_H, 2B_H, \ldots\}$}
    \State Load $X_{\mathrm{tile}} \gets X[bB_N{:}(b{+}1)B_N,\, i{:}i{+}B_H]$ from HBM \Comment{$(B_N, B_H)$}
    \State Load $R_2$ tile from HBM \Comment{$(r, B_H)$}
    \State $z \mathrel{+}= X_{\mathrm{tile}} \cdot R_{2,\mathrm{tile}}^\top$ \Comment{tile matmul in SRAM}
  \EndFor
  \State Narrow $z$ to \texttt{bf16} \Comment{tensor-core operands are narrow; the accumulator is \texttt{fp32}}
  \LComment{Phase 2: one sweep of $z R_1^\top$ over $M$ with a running top-$k$}
  \State Allocate $\mathrm{top\_vals} \gets -\infty$, $\mathrm{top\_idxs} \gets 0$ \Comment{$(B_N, k)$ in SRAM}
  \For{each expert tile $m_0 \in \{0, B_M, 2B_M, \ldots\}$}
    \State Load $R_1$ tile from HBM \Comment{$(B_M, r)$, cached in L2 across blocks}
    \State $y_{\mathrm{tile}} \gets z \cdot R_1[m_0{:}m_0{+}B_M, :]^\top$, out-of-range experts set to $-\infty$ \Comment{$(B_N, B_M)$ in SRAM}
    \For{$c = 1, \ldots, k$} \Comment{merge tile logits into the running top-$k$}
      \State $(v_{\max}, j_{\max}) \gets (\max, \arg\max)$ of $y_{\mathrm{tile}}$ along axis $1$
      \State $(v_{\min}, j_{\min}) \gets (\min, \arg\min)$ of $\mathrm{top\_vals}$ along axis $1$
      \State If $v_{\max} > v_{\min}$: write $(v_{\max},\, m_0 + j_{\max})$ into slot $j_{\min}$
      \State Mask $y_{\mathrm{tile}}[:, j_{\max}] \gets -\infty$ \Comment{prevent re-selection}
    \EndFor
  \EndFor
  \State Write $\mathrm{top\_vals}, \mathrm{top\_idxs}$ to HBM
\EndFor
\end{algorithmic}
\end{algorithm}

The kernel reads $X$ once, $R_1$ once and $R_2$ once, and writes $2Nk$ numbers,
so its HBM traffic is $O(Nh + Mr + Nk)$ rather than $O(NM)$. The inner merge loop runs $k$ times per expert tile and is
sequential, which is why the cost grows linearly in $k$ in
Figure~\ref{fig:fuse_benefit}(Left) and the cost is mild within the range of commonly used $k$.

\subsection{Training}
\label{appendix:training_pseudocode}

Training adds the load-balancing loss, and with it a quantity that depends on all
$M$ experts rather than on the $k$ selected ones. Figure~\ref{fig:app_flow} shows
its effect on the forward and on the backward.

\begin{figure}[H]
\centering
\resizebox{\textwidth}{!}{%
\begin{tikzpicture}[
  font=\small,
  bx/.style={draw, rounded corners=2pt, align=center, inner sep=4pt, minimum height=8mm},
  ker/.style={bx, fill=blue!6, draw=blue!45},
  dat/.style={bx, fill=black!4},
  grad/.style={bx, fill=red!8, draw=red!55, text=red!85!black},
  ar/.style={-{Latex[length=2mm]}, thick},
  gar/.style={-{Latex[length=2mm]}, thick, red!70},
  node distance=9mm,
]
\node[dat] (X) {$X$\\$(N,h)$};
\node[ker, right=of X] (z) {$z = X R_2^\top$\\$(N,r)$};
\node[ker, right=of z, minimum height=16mm, text width=34mm] (tile)
  {tile $Y = z\,R_1[\mathrm{blk}]^\top$\\[1pt] sweep 1: $m, Z$, top-$k$\\
   \ch{sweep 2: $P = e^{Y-m}/Z$}};
\node[dat, right=12mm of tile, yshift=9mm] (tv) {TopVals, TopIdxs\\$(N,k)$};
\node[dat, right=12mm of tile, yshift=-9mm] (ps) {\ch{$\mathrm{psum}$}\\\ch{$(M)$}};
\node[dat, right=9mm of tv] (lp) {$\mathcal{L}_{\text{pred}}$};
\node[dat, right=9mm of ps] (lb) {$\mathcal{L}_{\mathrm{bal}}$};
\draw[ar] (X) -- (z);
\draw[ar] (z) -- (tile);
\draw[ar] (tile.east) -- ++(4mm,0) |- (tv.west);
\draw[ar] (tile.east) -- ++(4mm,0) |- (ps.west);
\draw[ar] (tv) -- (lp);
\draw[ar] (ps) -- (lb);
\node[grad, below=22mm of tile, minimum height=16mm, text width=38mm] (bk)
  {\ch{sweep 1: $g$ $(N)$}\\[1pt] sweep 2: $\bar Y$ tile\\
   $= \overline{\mathrm{TopVals}} + \ch{P\odot(\bar p - g)}$};
\node[grad] (gtv) at (tv|-bk.north) {$\overline{\mathrm{TopVals}}$\\$(N,k)$};
\node[grad] (gp)  at (tv|-bk.south) {\ch{$\bar p$}\\\ch{$(M)$}};
\node[grad, left=of bk] (gz) {$\bar z$\\$(N,r)$};
\node[grad, left=of gz] (gX) {$\bar X$ $(N,h)$\\\footnotesize to previous layer};
\node[grad, below=8mm of gz] (gR2) {$\bar R_2$ $(r,h)$};
\node[grad, below=8mm of bk] (gR1) {$\bar R_1$ $(M,r)$};
\draw[gar] (lp.east) -- ++(5mm,0) |- (gtv.east);
\draw[gar] (lb.east) -- ++(11mm,0) |- (gp.east);
\draw[gar] (gtv.west) -- (bk.east|-gtv);
\draw[gar] (gp.west)  -- (bk.east|-gp);
\draw[gar] (bk) -- (gz);
\draw[gar] (bk) -- (gR1);
\draw[gar] (gz) -- (gR2);
\draw[gar] (gz) -- (gX);
\node[left=2mm of X, font=\footnotesize\bfseries, rotate=90, anchor=south] {forward};
\node[left=2mm of gX, font=\footnotesize\bfseries, red!80!black, rotate=90, anchor=south] {backward};
\end{tikzpicture}}
\caption{\textbf{Fused low-rank router for training.} Blue blocks run inside the
kernel, which holds one $B_N \times B_M$ tile of $\ell$ at a time and never forms
the full $N \times M$ matrix. \ch{Purple marks what the Switch loss adds.}
Sweep 2 computes the quantity required for the load-balancing loss:
$\mathrm{psum}_i = \sum_n P_{n,i}$, the softmax probability of expert $i$ summed
over tokens. $\overline{\mathrm{TopVals}}$ is the gradient of the prediction loss
with respect to the $k$ selected scores, and is sparse in $M$. \ch{$\bar p$ is the
gradient of the balancing loss with respect to $\mathrm{psum}$, one number per
expert and dense in $M$.}}
\label{fig:app_flow}
\end{figure}
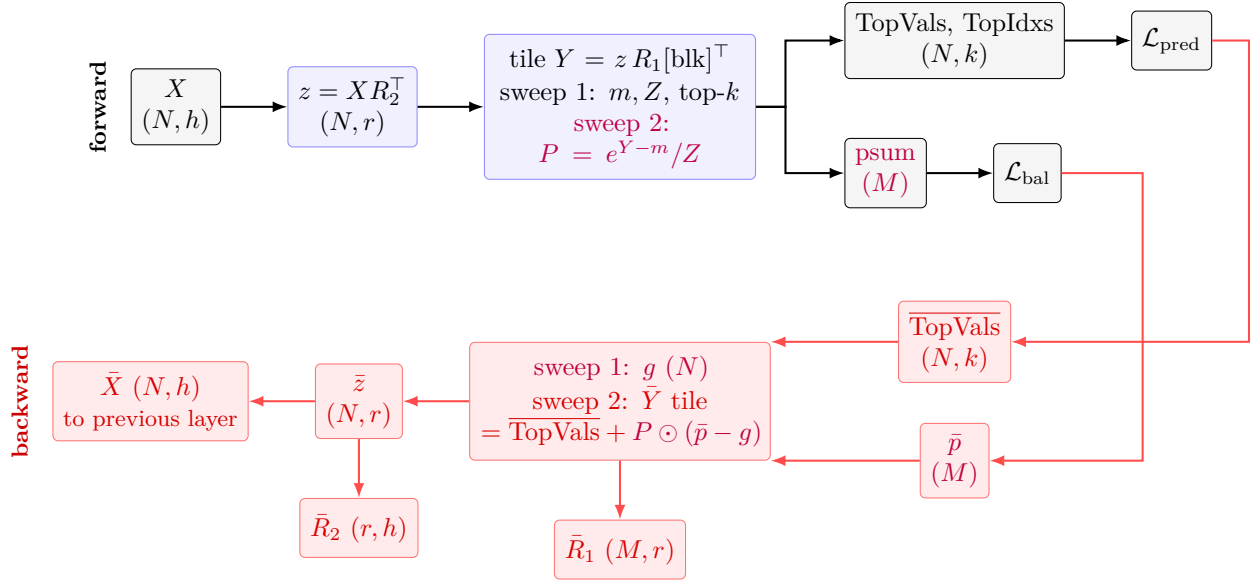

\subsubsection{Forward-Pass Pseudocode}
\label{appendix:fwd_pseudocode}

At forward pass in training, we leverage the techniques in online softmax
of~\citet{milakov2018online} to compute the load-balancing loss. Let $P \in \mathbb{R}^{N \times M}$ be the row-wise
softmax of the logits,
\begin{equation}
  P_{n,j} \;=\; \frac{e^{\,\ell_{n,j}}}{Z_n},
  \qquad
  Z_n \;=\; \sum_{j'=1}^{M} e^{\,\ell_{n,j'}} \in \mathbb{R} .
  \label{eq:app_softmax}
\end{equation}
The Switch load-balancing loss~\citep{fedus2022switch} is
$\mathcal{L}_{\mathrm{bal}} = M \sum_{i=1}^{M} f_i p_i$, where $f_i$ is the fraction
of dispatch slots that went to expert $i$ and carries no gradient, and
$p_i = \frac{1}{N}\sum_{n} P_{n,i}$ is that expert's mean routing probability over
\emph{all} $M$ experts. The only quantity in it that a tiled kernel cannot produce
for free is the column sum
\begin{equation}
  \mathrm{psum}_i \;=\; \sum_{n=1}^{N} P_{n,i} \;\in\; \mathbb{R}^{M},
  \qquad p_i = \mathrm{psum}_i / N,
  \label{eq:app_psum}
\end{equation}
because $P_{n,i}$ needs the normalizer $Z_n$, which is not known until every expert
has been visited. Returning $\mathrm{psum}$ is therefore the whole of what the
kernel must add, and it is $M$ numbers rather than $NM$.

\paragraph{Building the softmax denominator without the $(N, M)$ tensor.}
The kernel maintains the row maximum $m_n = \max_{j \in [M]} \ell_{n,j}$ and the
denominator $Z_n = \sum_{j'} e^{\ell_{n,j'} - m_n}$ across the expert tiles, and
uses $P_{n,j} = e^{\ell_{n,j} - m_n} / Z_n$. Neither $m_n$ nor $Z_n$ is final until
every expert has been seen, so both are built incrementally. After $t$ tiles the
kernel holds $m^{(t)}$, the largest logit so far, and $Z^{(t)}$, the sum over those
experts. When tile $t{+}1$ raises the maximum to $m'$, every term already in
$Z^{(t)}$ carries the factor $e^{-m^{(t)}}$ and now needs $e^{-m'}$. That
correction is one constant shared by all of them, so a single multiply rescales
the whole accumulated sum, and the calculation is shown as~\eqref{eq:app_online}:
\begin{equation}
  Z^{(t+1)} = Z^{(t)}\, e^{\,m^{(t)} - m'} \;+\; \textsc{rowsum}\bigl(e^{\,y_{\mathrm{tile}} - m'}\bigr),
  \qquad m^{(t+1)} = m' .
  \label{eq:app_online}
\end{equation}
Here $e^{\,y_{\mathrm{tile}} - m'}$ is elementwise on the $(B_N, B_M)$ tile and
$\textsc{rowsum}$ adds each row across the $B_M$ experts of this tile, so both sides
of~\eqref{eq:app_online} are vectors of length $B_N$: one running normalizer per
token.

Once $m$ and $Z$ are final, a second sweep regenerates each tile and accumulates
$\textsc{colsum}(e^{\,y_{\mathrm{tile}} - m}/Z)$ into $\mathrm{psum}$, where
$\textsc{colsum}$ reduces along tokens and so returns one number per expert. The
second sweep is needed because the normalizer is only complete at the end of the
first, and because the top-$k$ merge overwrites $y_{\mathrm{tile}}$ with $-\infty$ as
it consumes columns. It costs one extra $z R_1^\top$ matmul, $2NMr$ FLOPs, and no
extra HBM traffic.

In the implementation the first matmul $z = X R_2^\top$ is left to cuBLAS and $z$
is passed into the kernel, so the training kernel fuses the second matmul, the
top-$k$ and the softmax statistics, while the inference kernel of
Algorithm~\ref{alg:lr_inf} fuses all three.

\begin{algorithm}[H]
\caption{Fused low-rank router + Top-$k$ (training, forward). \ch{Purple lines are what the
Switch load-balancing loss adds.}}
\label{alg:lr_fwd}
\begin{algorithmic}[1]
\Require $z = X R_2^\top \in \mathbb{R}^{N \times r}$, router weight $R_1 \in \mathbb{R}^{M \times r}$, top-$k$ size $k$, block sizes $B_N, B_M$
\Ensure $\mathrm{TopVals} \in \mathbb{R}^{N \times k}$, $\mathrm{TopIdxs} \in \mathbb{Z}^{N \times k}$, \ch{$\mathrm{psum} \in \mathbb{R}^{M}$, $m \in \mathbb{R}^{N}$, $Z \in \mathbb{R}^{N}$} in HBM
\For{each token block $b \in \{0, \ldots, \lceil N / B_N \rceil - 1\}$ \textbf{in parallel}}
  \State Load $z_{\mathrm{block}}$ from HBM \Comment{$(B_N, r)$ in SRAM registers}
  \LComment{Sweep 1: tile $z R_1^\top$ over $M$ with running top-$k$ \ch{and running $m, Z$}}
  \State Allocate $\mathrm{top\_vals} \gets -\infty$, $\mathrm{top\_idxs} \gets 0$ \Comment{$(B_N, k)$ in SRAM}
  \State \ch{Allocate $m \gets -\infty$, $Z \gets 0$} \Comment{\ch{$(B_N)$ each, in SRAM}}
  \For{each expert tile $m_0 \in \{0, B_M, 2B_M, \ldots\}$}
    \State Load $R_1$ tile from HBM \Comment{$(B_M, r)$, cached in L2 across blocks}
    \State $y_{\mathrm{tile}} \gets z \cdot R_1^\top$, out-of-range experts set to $-\infty$ \Comment{$(B_N, B_M)$ in SRAM}
    \State \ch{$m' \gets \max\bigl(m,\ \textsc{rowmax}(y_{\mathrm{tile}})\bigr)$}
           \Comment{\ch{new running maximum, $(B_N)$}}
    \State \ch{$Z \gets Z\,e^{\,m - m'} + \textsc{rowsum}\bigl(e^{\,y_{\mathrm{tile}} - m'}\bigr)$;\ \
               $m \gets m'$}
           \Comment{\ch{eq.~\eqref{eq:app_online}}}
    \For{$c = 1, \ldots, k$} \Comment{merge tile logits into running top-$k$}
      \State $(v_{\max}, j_{\max}) \gets (\max, \arg\max)$ of $y_{\mathrm{tile}}$ along axis $1$
      \State $j_{\min} \gets \arg\min$ of $\mathrm{top\_vals}$ along axis $1$
      \State If $v_{\max} > \mathrm{top\_vals}[:, j_{\min}]$: write $(v_{\max}, m_0 + j_{\max})$ into slot $j_{\min}$
      \State Mask $y_{\mathrm{tile}}[:, j_{\max}] \gets -\infty$ \Comment{\ch{destroys $y_{\mathrm{tile}}$}; prevent re-selection}
    \EndFor
  \EndFor
  \LComment{\ch{Sweep 2: normalized column sums; $m$ and $Z$ are still in registers}}
  \For{\ch{each expert tile $m_0 \in \{0, B_M, 2B_M, \ldots\}$}}
    \State \ch{$y_{\mathrm{tile}} \gets z \cdot R_1^\top$;\ \ $P \gets e^{\,y_{\mathrm{tile}} - m} / Z$}
           \Comment{\ch{$(B_N, B_M)$, recomputed}}
    \State \ch{$\textsc{atomicAdd}\bigl(\mathrm{psum}[m_0{:}m_0{+}B_M],\ \textsc{colsum}(P)\bigr)$}
  \EndFor
  \State Write $\mathrm{top\_vals}, \mathrm{top\_idxs}$\ch{, $m$, $Z$} to HBM
\EndFor
\end{algorithmic}
\end{algorithm}

\paragraph{Forming the loss.} The balancing loss itself stays outside the kernel and
costs $O(M)$:
\begin{equation}
  \mathcal{L}_{\mathrm{bal}} \;=\; M \sum_{i=1}^{M} f_i\, p_i
  \;=\; \frac{M}{N} \, f \cdot \mathrm{psum},
  \qquad
  f_i = \frac{1}{Nk}\sum_{n=1}^{N}\sum_{j=1}^{k}\mathbf{1}\bigl[\mathrm{TopIdxs}[n,j] = i\bigr],
  \label{eq:app_lbal}
\end{equation}
with $f$ obtained by a \texttt{bincount} on $\mathrm{TopIdxs}$ and detached, as is
standard. Because the kernel returns $\mathrm{psum}$ rather than a loss value, the
same outputs serve the Switch form of~\eqref{eq:app_lbal} or a different balancing
objective, without touching the kernel.

\subsubsection{Backward-Pass Pseudocode}
\label{appendix:bwd_pseudocode}

Two gradients reach the kernel, and they have different shapes.
$\overline{\mathrm{TopVals}} \in \mathbb{R}^{N \times k}$ arrives from the prediction
loss on the $k$ selected scores; it is sparse in $M$, touching only the selected
experts. $\bar p \in \mathbb{R}^{M}$ arrives from
$\mathcal{L}_{\mathrm{bal}}$ through $\mathrm{psum}$, and by~\eqref{eq:app_lbal} it is
$\bar p_i = (M/N) f_i$, one number per expert and dense in $M$. The kernel must
return $\bar X \in \mathbb{R}^{N \times h}$,
$\bar R_1 \in \mathbb{R}^{M \times r}$ and $\bar R_2 \in \mathbb{R}^{r \times h}$, as
in Figure~\ref{fig:app_flow}. Everything else follows from the chain rule on
$\ell = z R_1^\top$ and $z = X R_2^\top$, so the only new quantity to derive is the
logit gradient $\bar y$.

\paragraph{The logit gradient of the balancing loss.} Fix a token $n$ and an
expert $j \in [M]$. Since $\mathrm{psum}_i = \sum_{n'} P_{n',i}$ and only row $n$ of
$P$ depends on $\ell_n$, the chain rule contracts over the expert index alone:
\begin{equation}
  \frac{\partial \mathcal{L}_{\mathrm{bal}}}{\partial \ell_{n,j}}
  = \sum_{i=1}^{M} \bar p_i \, \frac{\partial P_{n,i}}{\partial \ell_{n,j}} .
  \label{eq:app_chain}
\end{equation}
The softmax Jacobian follows from the quotient rule on
$P_{n,i} = e^{\ell_{n,i}}/Z_n$ with $Z_n = \sum_{j'} e^{\ell_{n,j'}}$. The numerator
depends on $\ell_{n,j}$ only when $i = j$, contributing
$\delta_{ij} e^{\ell_{n,i}}$ with $\delta_{ij} = 1$ if $i = j$ and $0$ otherwise; the
denominator always does, contributing
$\partial Z_n / \partial \ell_{n,j} = e^{\ell_{n,j}}$. Hence
\begin{equation}
  \frac{\partial P_{n,i}}{\partial \ell_{n,j}}
  = \frac{\delta_{ij}\,e^{\ell_{n,i}} Z_n - e^{\ell_{n,i}} e^{\ell_{n,j}}}{Z_n^{2}}
  = \delta_{ij}\frac{e^{\ell_{n,i}}}{Z_n} - \frac{e^{\ell_{n,i}}}{Z_n}\frac{e^{\ell_{n,j}}}{Z_n}
  = P_{n,i}\bigl(\delta_{ij} - P_{n,j}\bigr).
  \label{eq:app_jac}
\end{equation}
Substituting~\eqref{eq:app_jac} into~\eqref{eq:app_chain}, the $\delta_{ij}$ collapses
the first sum to its $i = j$ entry and the second factors out $P_{n,j}$:
\begin{equation}
  \frac{\partial \mathcal{L}_{\mathrm{bal}}}{\partial \ell_{n,j}}
  = \underbrace{\bar p_j P_{n,j}}_{\text{from }\delta_{ij}}
    \;-\; P_{n,j}\sum_{i=1}^{M} \bar p_i P_{n,i}
  = P_{n,j}\bigl(\bar p_j - g_n\bigr),
  \qquad
  g_n \equiv \sum_{i=1}^{M} \bar p_i P_{n,i} \in \mathbb{R}.
  \label{eq:app_kernelgrad}
\end{equation}
So $g \in \mathbb{R}^{N}$ is one scalar per token: the mean of $\bar p$ under that
token's own routing distribution. Because $g_n$ sums over all $M$ experts, it must be
complete before any single expert's gradient can be written, so the backward takes
two sweeps for the same reason the forward does.

\begin{algorithm}[H]
\caption{Low-rank router (training, backward) with activation recompute. \ch{Purple lines are
what the Switch load-balancing loss adds.}}
\label{alg:lr_bwd}
\begin{algorithmic}[1]
\Require $X$, $R_1\in\mathbb{R}^{M\times r}$, $R_2\in\mathbb{R}^{r\times h}$, $\mathrm{TopIdxs}\in\mathbb{Z}^{N\times k}$, $\overline{\mathrm{TopVals}}\in\mathbb{R}^{N\times k}$ (top-$k$ logit gradients), \ch{$\bar p\in\mathbb{R}^{M}$, saved $m, Z \in \mathbb{R}^{N}$}
\Ensure $\bar X\in\mathbb{R}^{N\times h}$, $\bar R_1\in\mathbb{R}^{M\times r}$, $\bar R_2\in\mathbb{R}^{r\times h}$ in HBM
\For{each token block $b$ \textbf{in parallel}}
  \LComment{Phase 1: recompute $z = X_{\mathrm{block}} R_2^\top$ in SRAM (saves $(N, r)$ activation memory)}
  \State $z \gets 0$ \Comment{$(B_N, r)$ in SRAM}
  \For{each hidden tile $i$}
    \State $z \mathrel{+}= X_{\mathrm{tile}} \cdot R_2^\top$
  \EndFor
  \LComment{\ch{Phase 2, sweep 1: the per-token scalar $g_n$ of eq.~\eqref{eq:app_kernelgrad}}}
  \State \ch{$g \gets 0$} \Comment{\ch{$(B_N)$ in SRAM}}
  \For{\ch{each expert tile $m_0 \in \{0, B_M, 2B_M, \ldots\}$}}
    \State \ch{$y_{\mathrm{tile}} \gets z \cdot R_1^\top$;\ \ $P \gets e^{\,y_{\mathrm{tile}} - m}/Z$}
    \State \ch{$g \mathrel{+}= \textsc{rowsum}\bigl(P \odot \bar p[m_0{:}m_0{+}B_M]\bigr)$}
  \EndFor
  \LComment{Phase 2, sweep 2: the logit gradient, contracted into $\bar z$ and $\bar R_1$}
  \State $\bar z \gets 0$ \Comment{$(B_N, r)$ in SRAM}
  \For{each expert tile $m_0 \in \{0, B_M, 2B_M, \ldots\}$}
    \State $\bar y \gets 0$; scatter $\overline{\mathrm{TopVals}}[n, j]$ into column $\mathrm{TopIdxs}[n,j]$ of this tile
           \Comment{sparse: $k$ of $B_M$}
    \State \ch{$y_{\mathrm{tile}} \gets z \cdot R_1^\top$;\ \ $P \gets e^{\,y_{\mathrm{tile}} - m}/Z$}
    \State \ch{$\bar y \mathrel{+}= P \odot \bigl(\bar p[m_0{:}m_0{+}B_M] - g\bigr)$}
           \Comment{\ch{eq.~\eqref{eq:app_kernelgrad}, dense in $M$}}
    \State $\bar z \mathrel{+}= \bar y \cdot R_1[m_0{:}m_0{+}B_M, :]$;\ \
           $\bar R_1[m_0{:}m_0{+}B_M, :] \mathrel{+}= \bar y^\top z$
  \EndFor
  \LComment{Phase 3: $\bar X_{\mathrm{block}} = \bar z \cdot R_2$ tiled over $h$}
  \For{each hidden tile $i$}
    \State Load $R_2$ tile, write $\bar X_{\mathrm{tile}} \gets \bar z \cdot R_2$
  \EndFor
  \State $\bar R_2 \mathrel{+}= \bar z^\top X_{\mathrm{block}}$ \Comment{$(r, B_N) \times (B_N, h) \to (r, h)$}
\EndFor
\end{algorithmic}
\end{algorithm}

\section{Additional Plots}
\label{appendix:additional_plots}

\subsection{Routing entropy}
\label{appendix:hnorm_definition}
We report the routing entropy of the per-expert token counts $c_1, \ldots, c_M$ of
one layer. Writing $p_a = c_a / \sum_b c_b$ for the share of the layer's routed
tokens that expert $a$ receives, the normalized entropy is
$H_{\mathrm{norm}} = \bigl(-\sum_{a} p_a \log p_a\bigr) / \log M$, accumulated over
the full training set. It is $1$ when every expert carries the same load, and falls
towards $0$ as routing concentrates on fewer experts. The counts are read from saved
checkpoints: for each checkpoint we pass the full training set through that fixed
model once and count the tokens each expert receives.

Figure~\ref{fig:memorization_in_training_M1024} repeats
Figure~\ref{fig:memorization_in_training} at $h{=}64$, $s{=}8$, $M{=}1024$, where the
expert pool is more than five times larger, and we find similar results.

\begin{figure}[t]
    \centering
    \includegraphics[width=\linewidth]{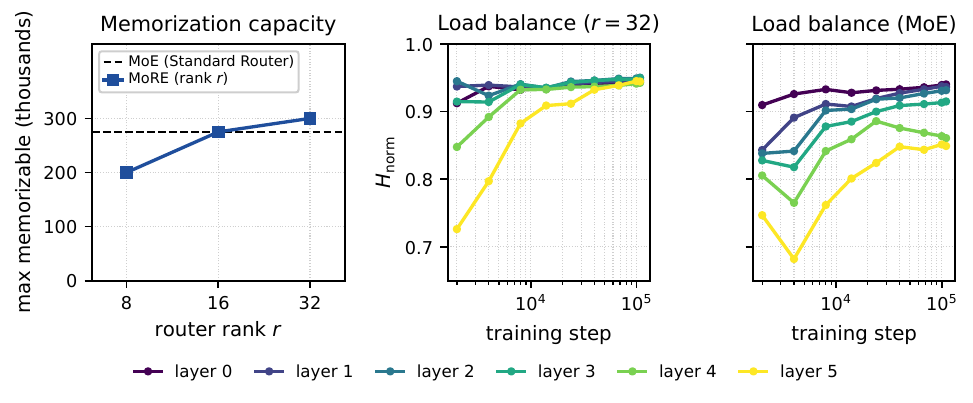}
    \caption{\textbf{Similar results at a larger expert pool.} Phonebook, $h{=}64$,
    $s{=}8$, $M{=}1024$, $k{=}4$, $L{=}6$, panels as in
    Figure~\ref{fig:memorization_in_training}. \textbf{Left:} largest phonebook
    memorized, searched on a $25$K grid. \textbf{Middle, Right:} $H_{\mathrm{norm}}$
    per layer against training step, for $r{=}32$ and for MoE, both trained on the
    $200$K phonebook. This configuration is off the matched-FLOP curve of
    Proposition~\ref{thm:optimal_scaling}: it places several ranks below $\log_2 M$,
    and both arms hold the same experts, so only the router differs.}
    \label{fig:memorization_in_training_M1024}
\end{figure}

\FloatBarrier
\subsection{Adding More Experts help}
\label{appendix:more_experts}

We sweep the size of the expert pool over $M \in \{112, 224, 448, 896, 1792\}$ while holding the rest of the architecture fixed at $h{=}512$, $s{=}56$, $r{=}64$, $k{=}4$ and $L{=}12$. All five configurations train on the same $5$B tokens. Figure~\ref{fig:more_experts} shows that validation perplexity decreases monotonically over the whole range, from $19.10$ at $M{=}112$ to $17.05$ at $M{=}1792$, demonstrating the benefits of using more experts.

\begin{figure}[H]
    \centering
    \includegraphics[width=0.55\linewidth]{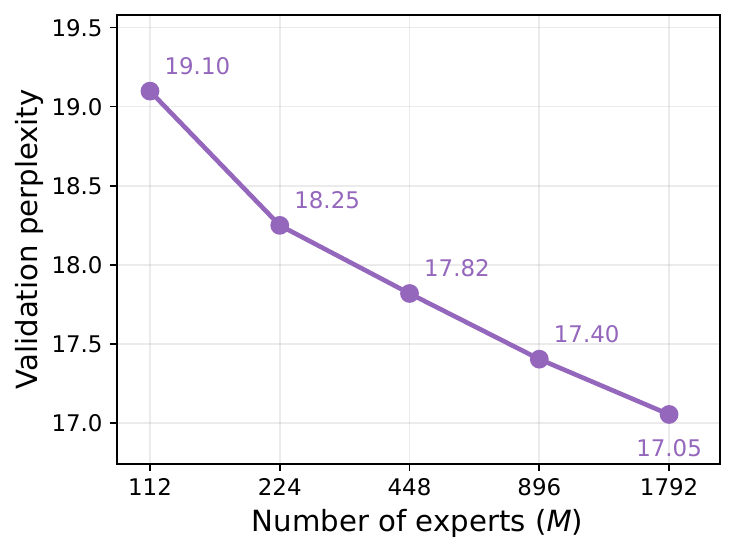}
    \caption{\textbf{Adding more experts on MoRE keeps improving quality.} Validation perplexity of MoRE ($r = 64$) after $5$B training tokens, against the number of experts $M$. The curve is monotone across a $16\times$ increase in the pool size.}
    \label{fig:more_experts}
\end{figure}

\FloatBarrier
\subsection{The perplexity gap grows with the token budget}
\label{appendix:gap_vs_tokens}

In Figure~\ref{fig:gap_vs_tokens_ep_pair}, we find that the perplexity gap between MoRE and MoE grow with token budget, which we think the gap would grow larger if more tokens are trained. This also matches our intuition, since MoRE's benefits of memorizing more information would be more apparent when more tokens are trained. 

\begin{figure}[H]
    \centering
    \includegraphics[width=0.9\linewidth]{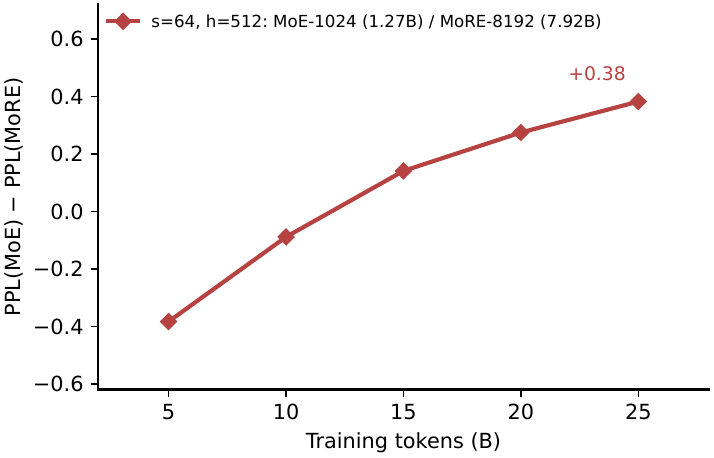}
    \caption{\textbf{The perplexity gap between the two configurations grows with the token budget.} Validation perplexity of MoE-1024 minus that of MoRE-8192 on the $5{,}000$ held-out mixture sequences, against training tokens. Positive values mean MoRE has the lower perplexity. Both models have matched active FLOPs and $1.27$B against $7.92$B total parameters. We plot the first $25$B tokens before the cosine schedule enters its annealing tail.}
    \label{fig:gap_vs_tokens_ep_pair}
\end{figure}

\FloatBarrier

\FloatBarrier
\section{Additional Plots for Efficiency}
\label{appendix:additional_plots_efficiency}

\subsection{Forward pass of a single layer at larger expert widths}
Figure~\ref{fig:prefill_heatmap_appendix_s128_s512} repeats
Figure~\ref{fig:prefill_heatmap} at $s \in \{128, 256, 512\}$ with the same
protocol: inference (prefilling) of the whole layer, $k{=}4$, $r{=}64$,
$N = 262{,}144$ ($B{=}256$, $T{=}1024$), \texttt{bf16}, mean of 30 trials after 10
warm-ups on a single NVIDIA B200. The fused MoRE layer is faster in every cell that
fits in memory. The speedup still grows with $M$ along every row, and it shrinks as
$s$ grows, because wider experts make the expert computation, which both layers
share, a larger part of the layer. At $M = 8{,}192$ it is $1.7$ to $4.2\times$ at
$s{=}128$, $1.6$ to $3.6\times$ at $s{=}256$ and $1.6$ to $3.0\times$ at $s{=}512$.

\begin{figure}[H]
    \centering
    \includegraphics[width=\linewidth]{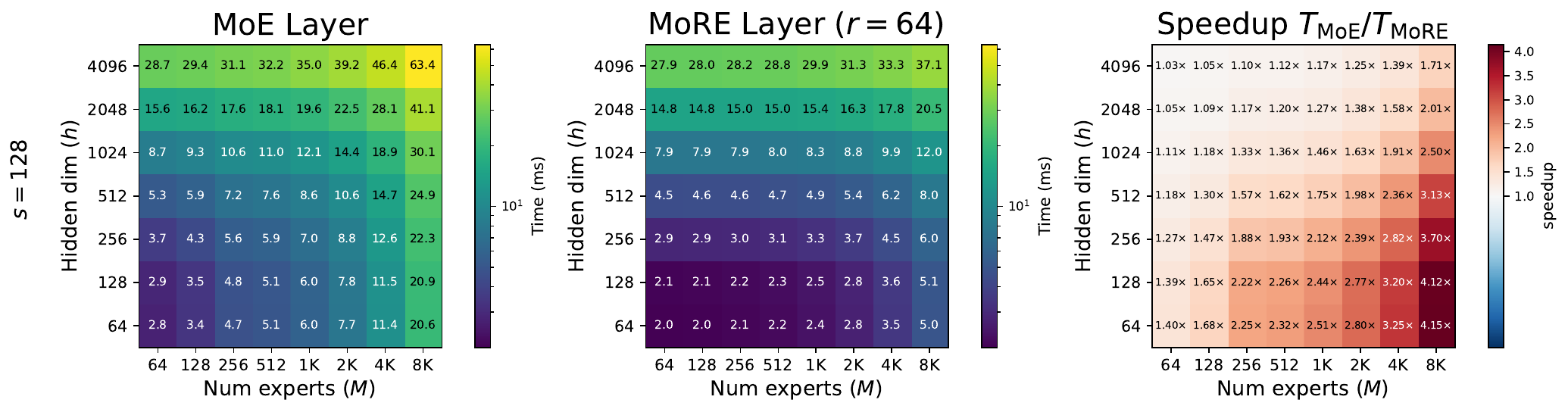}\\[4pt]
    \includegraphics[width=\linewidth]{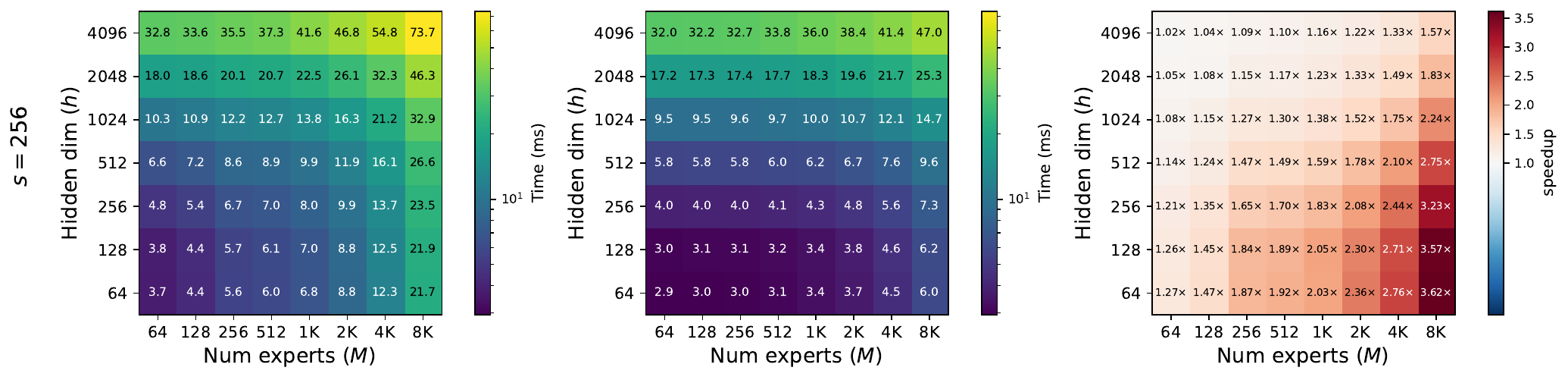}\\[4pt]
    \includegraphics[width=\linewidth]{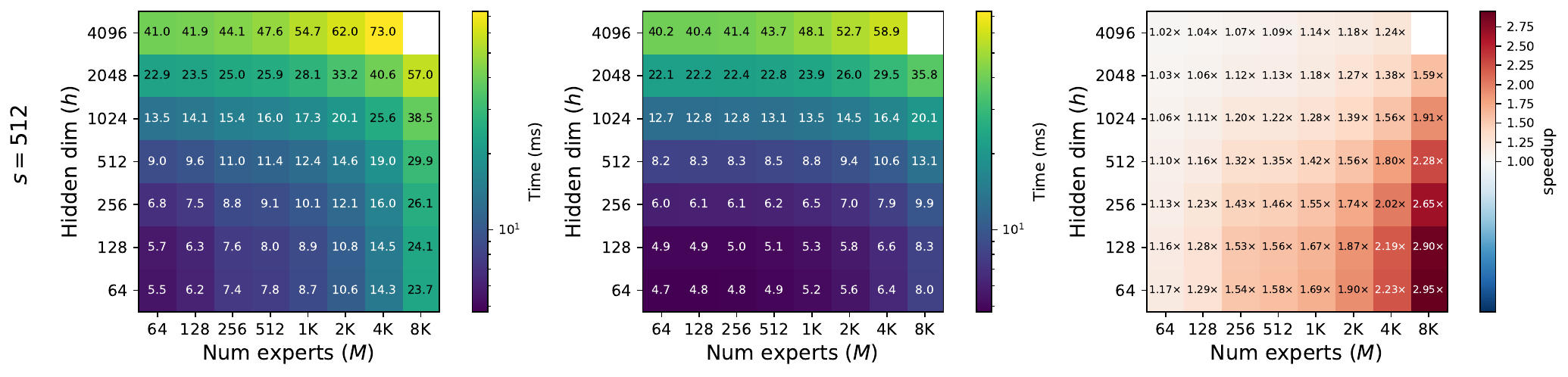}
    \caption{\textbf{The fused MoRE layer stays faster at larger expert widths.}
    Inference (prefilling) of a single layer at $s{=}128$, $256$ and $512$ (rows).
    \textbf{(a)} MoE Layer. \textbf{(b)} MoRE Layer ($r{=}64$).
    \textbf{(c)} speedup $T_{\text{MoE}} / T_{\text{MoRE}}$. The blank cell at
    $s{=}512$ ran out of memory on one B200. $k{=}4$, $r{=}64$, $N = 262{,}144$
    ($B{=}256$, $T{=}1024$), \texttt{bf16}, mean of 30 trials after 10 warm-ups on
    a single NVIDIA B200.}
    \label{fig:prefill_heatmap_appendix_s128_s512}
\end{figure}

\subsection{The same grid on an A100}
\label{appendix:a100_grid}

Figure~\ref{fig:switch_a100} repeats the inference (prefilling) measurement of
Figure~\ref{fig:prefill_heatmap} on a single NVIDIA A100-80GB, over the same
$(h, M)$ grid: hidden dimension
$h \in \{64, 128, 256, 512, 1{,}024, 2{,}048, 4{,}096\}$ against expert count
$M \in \{64, 128, 256, 512, 1{,}024, 2{,}048, 4{,}096, 8{,}192\}$, at $s{=}64$,
$k{=}4$ and $r{=}64$. The fused MoRE layer is faster in every cell, by $1.04$ to
$3.07\times$, and the speedup grows with $M$ along every row, reaching $2.60$ to
$3.07\times$ at $M = 8{,}192$.

\begin{figure}[H]
    \centering
    \includegraphics[width=\linewidth]{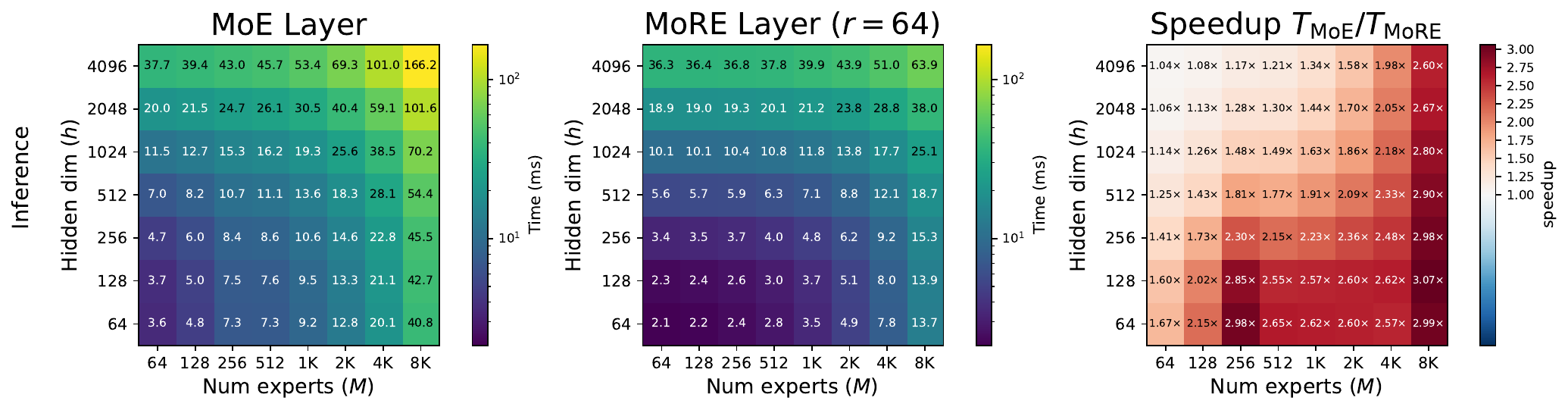}
    \caption{\textbf{The A100 reproduces the B200 result, with the fused MoRE layer
    $2.60$ to $3.07\times$ faster at $M{=}8{,}192$.} Inference (prefilling) of a
    single layer, the two arms at equal $M$. \textbf{(a)} MoE Layer.
    \textbf{(b)} MoRE Layer ($r{=}64$). \textbf{(c)} speedup
    $T_{\text{MoE}} / T_{\text{MoRE}}$. Single NVIDIA A100-80GB, $s{=}64$, $k{=}4$,
    $N = 262{,}144$ ($B{=}256$, $T{=}1024$), \texttt{bf16}, mean of $30$ trials
    after $10$ warm-ups.}
    \label{fig:switch_a100}
\end{figure}

\begin{figure}[H]
    \centering
    \includegraphics[width=0.62\linewidth]{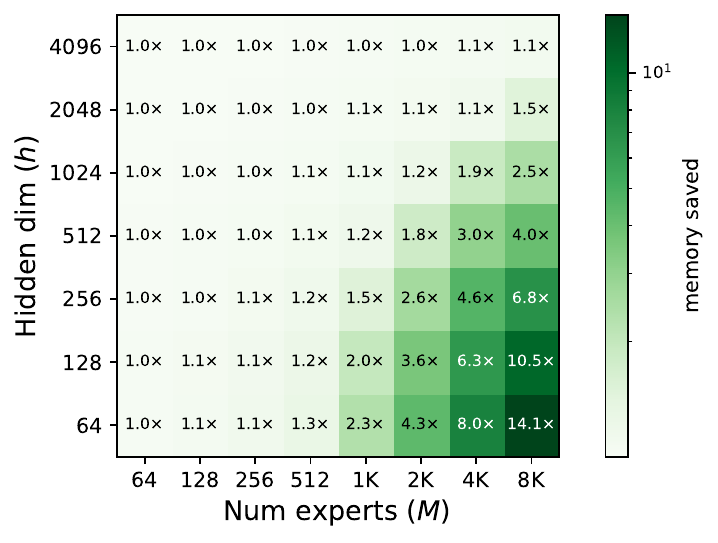}
    \caption{\textbf{Fusing the low-rank router reduces peak memory by up to
    $14\times$ for the whole layer.} Ratio of peak memory over a forward and a
    backward, unfused MoRE against fused MoRE, on the same grid.}
    \label{fig:switch_memory}
\end{figure}

\FloatBarrier
\subsection{Fusing a standard MoE layer}
\label{appendix:fuse_standard_moe}

The kernel of Section~\ref{subsec:fused-kernel} can also be written for a standard router. In Figure~\ref{fig:fuse_standard_moe}, we find the benefit of fusing diminishes as $h$ grows larger. 

\begin{figure}[H]
    \centering
    \includegraphics[width=\linewidth]{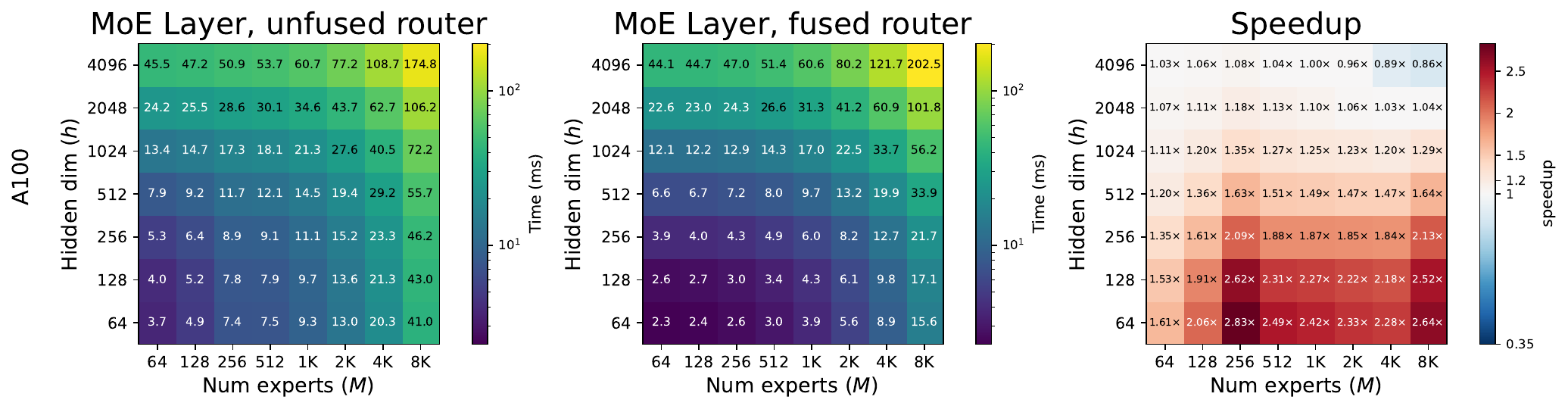}\\[4pt]
    \includegraphics[width=\linewidth]{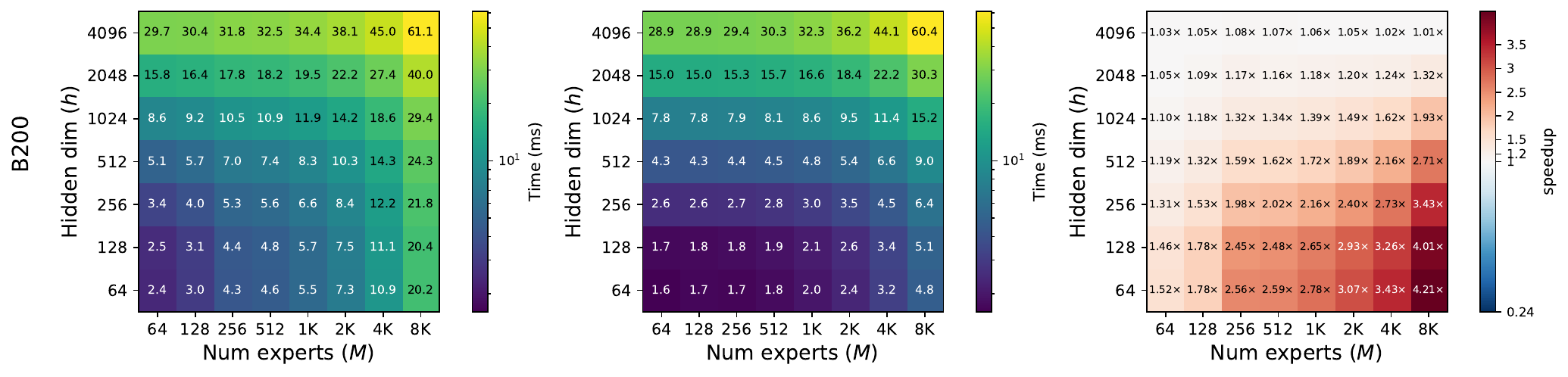}
    \caption{\textbf{On both A100 and B200, fusing Standard MoE does not get much acceleration at
    large $h$.} Inference (prefilling) of a single layer, both arms a standard MoE
    layer at the same $M$. \textbf{(a)} unfused router. \textbf{(b)} fused router.
    \textbf{(c)} speedup. \textbf{(Top)} A100, $0.86$ to $2.83\times$.
    \textbf{(Bottom)} B200, $1.01$ to $4.21\times$. At $h{=}4{,}096$, $M{=}8{,}192$
    fusing takes the layer from $174.8$ to $202.5$\,ms on the A100 and from $61.1$
    to $60.4$\,ms on the B200. $s{=}64$, $k{=}4$, $N = 262{,}144$ ($B{=}256$,
    $T{=}1024$), \texttt{bf16}, mean of $30$ trials after $10$ warm-ups.}
    \label{fig:fuse_standard_moe}
\end{figure}

Figure~\ref{fig:fuse_standard_moe} holds $M$ fixed along each row. Reading one row of
it, $M{=}4{,}096$, against the low-rank router at the same shape gives
Figure~\ref{fig:fuse_saving_h}: fusing takes the same work off both layers, and only
the standard one gives it back as $h$ grows.

\begin{figure}[H]
    \centering
    \includegraphics[width=\linewidth]{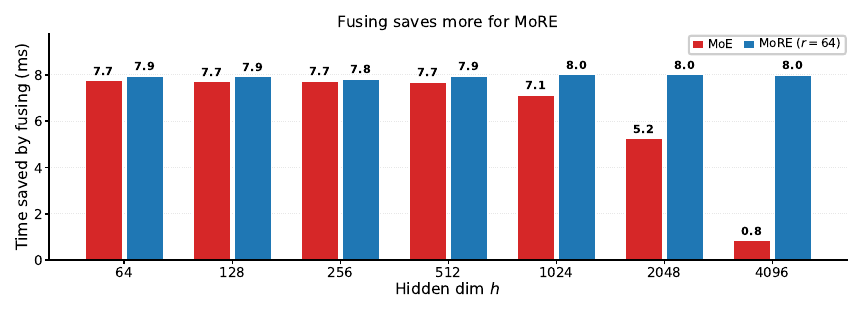}
    \caption{\textbf{Fusing keeps saving time for MoRE as $h$ grows, but stops saving
    time for MoE.} Forward pass of a single layer, the unfused-router time minus the
    fused-router time, against the hidden dimension $h$ at a fixed pool of
    $M{=}4{,}096$ experts. What fusing removes is the round trip of the $(N, M)$
    scores, which depends on $N$ and $M$ but not on $h$. The standard router gives it
    back because its fused kernel reloads a $(B_M, h)$ tile of the router weight for
    every token block. $s{=}64$, $k{=}4$, $r{=}64$, $N = 262{,}144$ ($B{=}256$,
    $T{=}1024$), \texttt{bf16}, mean of $30$ trials after $10$ warm-ups on a single
    NVIDIA B200.}
    \label{fig:fuse_saving_h}
\end{figure}

\FloatBarrier
\subsection{The low-rank router once both layers are fused}
\label{appendix:lowrank_bothfused}

Figure~\ref{fig:lowrank_bothfused} compares MoE and MoRE ($r{=}64$) where both are
fused.

\begin{figure}[H]
    \centering
    \includegraphics[width=\linewidth]{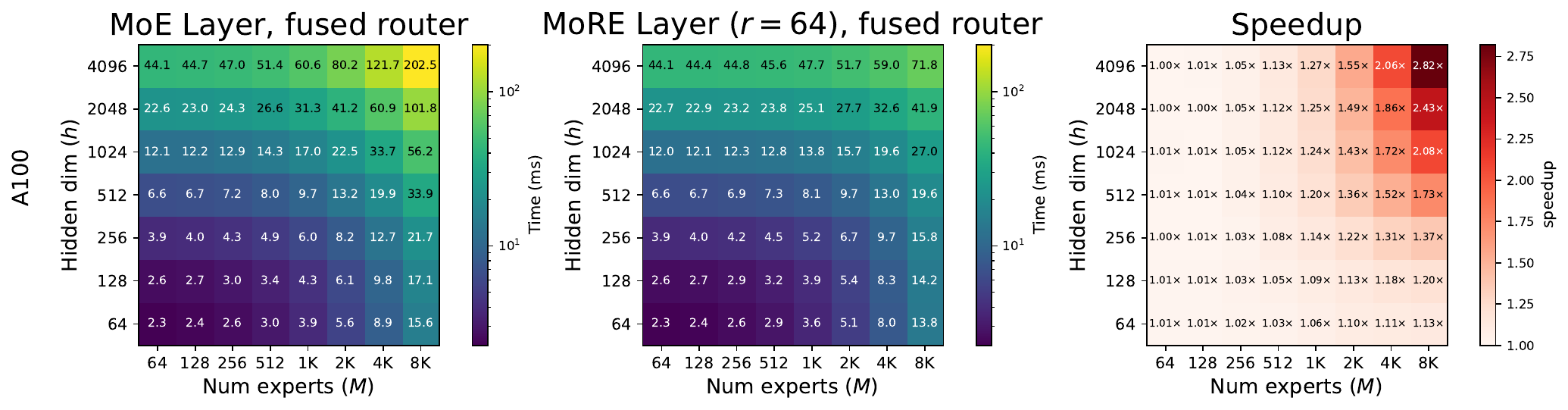}\\[4pt]
    \includegraphics[width=\linewidth]{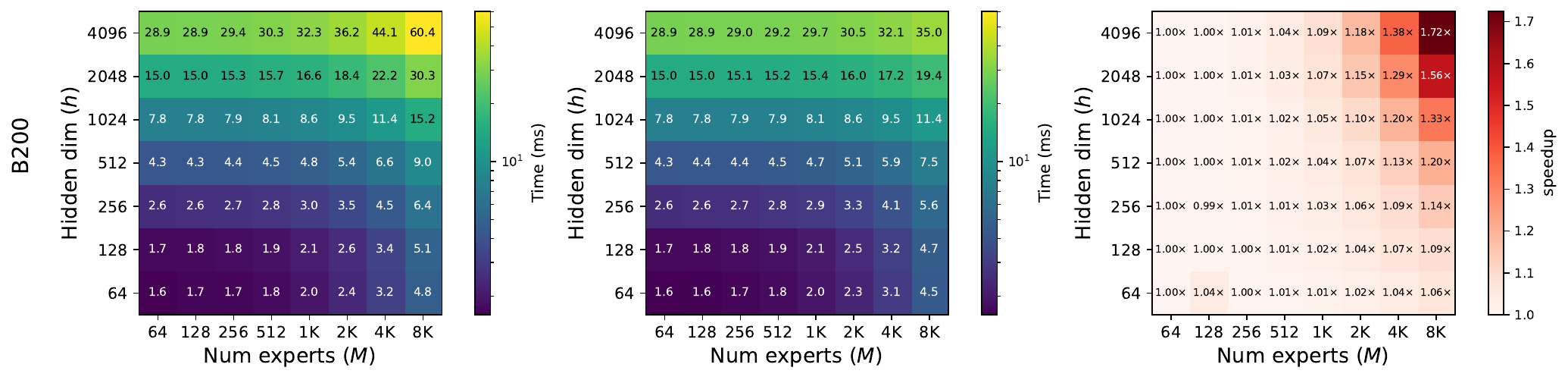}
    \caption{\textbf{Comparing fused MoE and fused MoRE, the low-rank factorization
    gives speedup up to $2.82\times$ on an A100 and $1.72\times$ on a B200.} Inference
    (prefilling) of a single layer, both arms fused, at the same $M$.
    \textbf{(a)} MoE Layer with the fused router. \textbf{(b)} MoRE Layer
    ($r{=}64$) with the fused router. \textbf{(c)} speedup. \textbf{(Top)} A100,
    $1.00$ to $2.82\times$. \textbf{(Bottom)} B200, $0.99$ to $1.72\times$.
    $s{=}64$, $k{=}4$, $N = 262{,}144$ ($B{=}256$, $T{=}1024$), \texttt{bf16}, mean
    of $30$ trials after $10$ warm-ups.}
    \label{fig:lowrank_bothfused}
\end{figure}

\FloatBarrier
\subsection{End-to-End Forward-Pass Breakdown at Large $M$}
\label{appendix:end_to_end_M}

Figure~\ref{fig:end_to_end_M_sweep} measures the end-to-end wall-clock time of a 12-layer transformer with 8 attention heads (FlashAttention-4~\citep{zadouri2026flashattention4algorithmkernelpipelining}), vocabulary size $V{=}50{,}257$, $h{=}512$, $s{=}64$, $k{=}4$, $B{=}64$, $T{=}1024$, sweeping $M \in \{1{,}024, 2{,}048, 4{,}096, 8{,}192\}$.

\begin{figure}[H]
    \centering
    \includegraphics[width=0.49\linewidth]{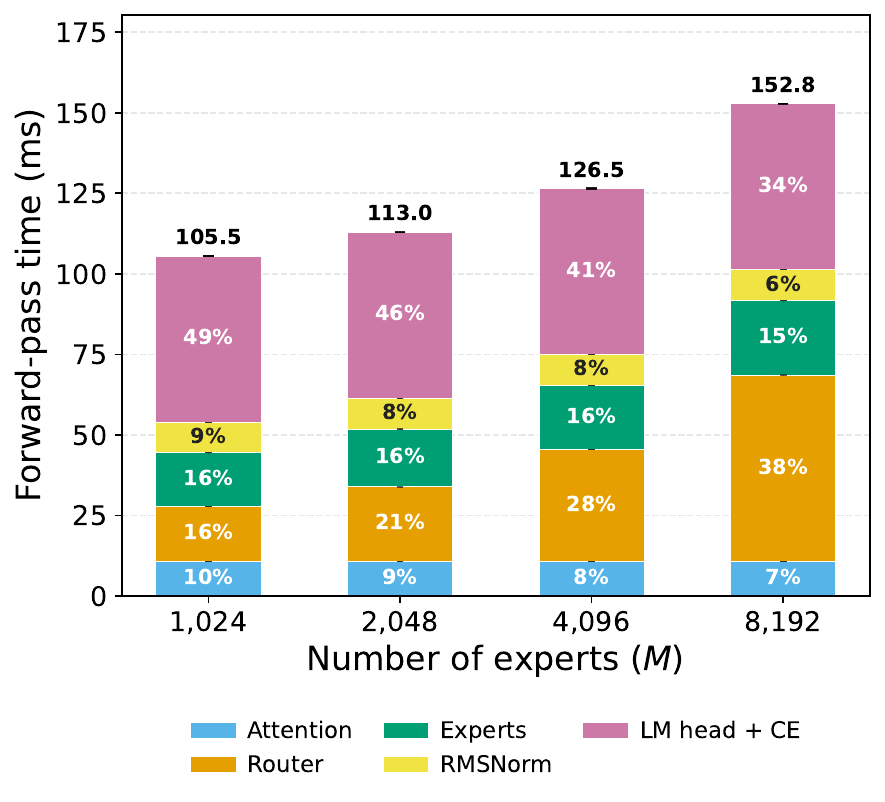}\hfill
    \includegraphics[width=0.49\linewidth]{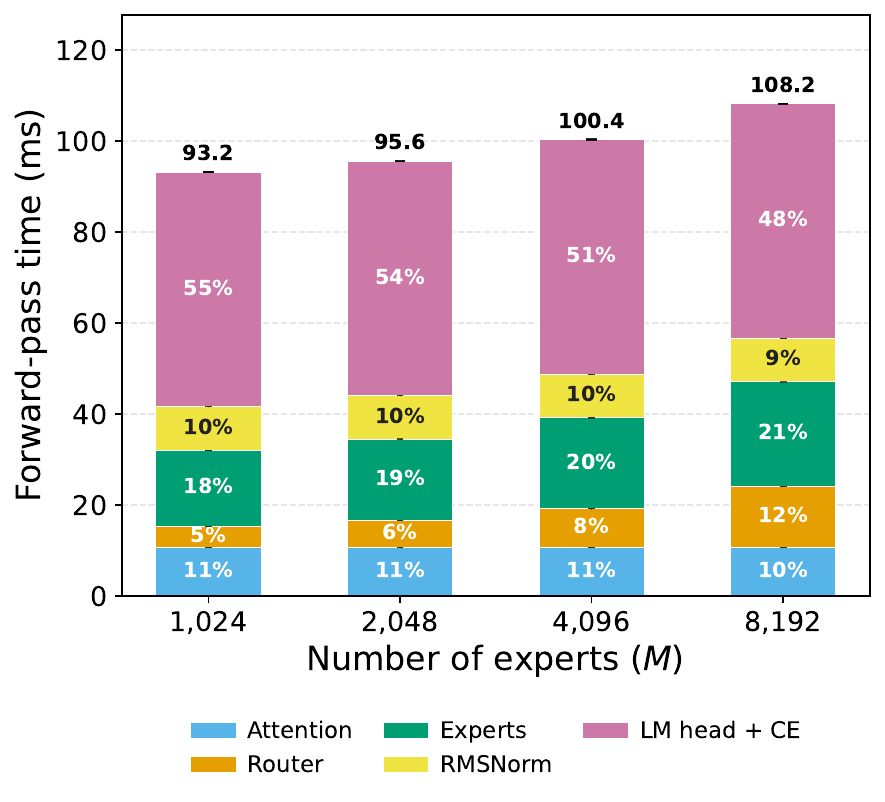}
    \caption{\textbf{End-to-end forward-pass breakdown at $h{=}512$, sweeping $M$.}
    Each bar decomposes the wall-time of a 12-layer transformer forward pass into Attention, Router, Experts (dispatch + grouped MLP), RMSNorm, and the Liger-fused~\citep{hsu2024ligerkernel} LM head + cross-entropy.
    \textbf{(Left)} Standard \textbf{MoE} router (unfused $h{\times}M$ matmul + top-$k$ + softmax + load-balancing-loss as separate kernels): the router slice grows from 16\% of total at $M{=}1{,}024$ to \textbf{38\% at $M{=}8{,}192$}, becoming the largest single component and the dominant driver of the rising end-to-end cost.
    \textbf{(Right)} Fused low-rank \textbf{MoRE} router ($r{=}64$): the router slice stays at 5--12\% across the same sweep, and total wall-time grows much more slowly with $M$. At $M{=}8{,}192$, MoRE is \textbf{30\% faster end-to-end} ($108.2$\,ms vs.\ $152.8$\,ms), with router cost reduced from $57.9$\,ms to $13.4$\,ms.
    Mean of 200 trials with 50 warm-ups, 95\% C.I.\ on a single NVIDIA B200.}
    \label{fig:end_to_end_M_sweep}
\end{figure}
\subsection{Fusion benefit at different batch sizes}
\label{appendix:fuse_batch}

In Figure~\ref{fig:fuse_benefit_batch}, we find that the benefit of fusing grows with the
batch size. The fully fused MoRE router is $1.4\times$ faster than the unfused one at
$B{=}2$ and $6.3\times$ faster at $B{=}256$. At $B{=}2$ both routers take about
$0.1$\,ms. As $B$ grows, most of the unfused router's time goes to the top-$k$ over the
$(N, M)$ score tensor that it stores in HBM, whose size grows with the number of tokens
$N = B \cdot T$, while the fused kernel never stores this tensor.

\begin{figure}[H]
    \centering
    \includegraphics[width=0.55\linewidth]{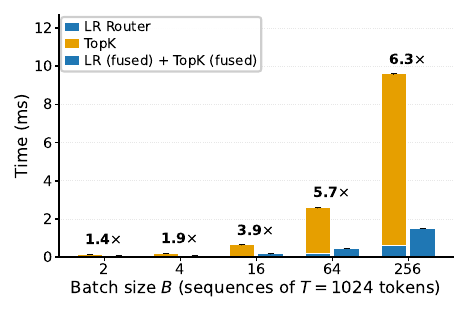}
    \caption{\textbf{The benefit of fusing the MoRE router grows with the batch size.} Forward-pass time of the low-rank router, unfused and fully fused, with the speedup of the fused kernel above each pair. $h{=}2048$, $M{=}4{,}096$, $r{=}16$, $k{=}4$, $T{=}1024$, \texttt{bf16}, mean of $200$ trials after $50$ warm-ups on a single NVIDIA B200.}
    \label{fig:fuse_benefit_batch}
\end{figure}

\FloatBarrier

\section{Scaling Total Parameters}
\label{appendix:scaling}

\begin{proposition}[Optimal scaling under fixed active-FLOPs]
\label{thm:optimal_scaling}
Fix the active-FLOP budget $\tau$, the hidden dimension $h$, and the active expert count $k$. The MoE configuration $(M, k, h, s)$ and the MoRE configuration $(M', k, h, s', r)$ that maximize total MLP parameters under matched $\tau$ are
\[
M^* \;=\; \frac{\tau}{2 C_r h}, \qquad
s^* \;=\; \frac{\tau}{2 C_e k h}, \qquad
T_{\mathrm{MoE}}^* \;=\; \frac{3 \tau^2}{4 C_e C_r\, h k},
\]
\[
M' \;=\; \frac{\tau'}{2 C_r r}, \qquad
s' \;=\; \frac{\tau'}{2 C_e k h}, \qquad
T_{\mathrm{MoRE}}^* \;=\; \frac{3 (\tau')^2}{4 C_e C_r\, r k},
\]
where $\tau' := \tau - C_r h r$. Furthermore, in the regime $M, M' \gg r$ (equivalently $\tau \gg C_r h r$),
\[
\frac{M'}{M^*} \;\to\; \frac{h}{r}, \qquad \frac{T_{\mathrm{MoRE}}^*}{T_{\mathrm{MoE}}^*} \;\to\; \frac{h}{r}.
\]
\end{proposition}

\begin{proof}
We use the same budget-allocation argument in both settings.

\medskip\noindent\textbf{MoE.} Allocate fraction $\alpha \in [0, 1]$ of the budget to the router and $1 - \alpha$ to the experts:
\[
C_r M h \;=\; \alpha \tau, \qquad C_e k h s \;=\; (1 - \alpha) \tau,
\]
so $M = \alpha \tau / (C_r h)$ and $s = (1 - \alpha) \tau / (C_e k h)$. Substituting into $T_{\mathrm{MoE}} = 3 M h s$,
\[
T_{\mathrm{MoE}}(\alpha) \;=\; \frac{3 \tau^2}{C_e C_r\, h k} \cdot \alpha (1 - \alpha),
\]
which is maximized at $\alpha = 1/2$, yielding the stated $M^*, s^*, T_{\mathrm{MoE}}^*$.

\medskip\noindent\textbf{MoRE.} The router cost expands as $C_r (M' + h) r = C_r M' r + C_r h r$, where the second term is a fixed overhead independent of $M'$. Subtracting it leaves a variable budget $\tau' := \tau - C_r h r$, which we split the same way:
\[
C_r M' r \;=\; \alpha \tau', \qquad C_e k h s' \;=\; (1 - \alpha) \tau',
\]
so $M' = \alpha \tau' / (C_r r)$ and $s' = (1 - \alpha) \tau' / (C_e k h)$. Substituting into $T_{\mathrm{MoRE}} = 3 M' h s'$,
\[
T_{\mathrm{MoRE}}(\alpha) \;=\; \frac{3 (\tau')^2}{C_e C_r\, r k} \cdot \alpha (1 - \alpha),
\]
again maximized at $\alpha = 1/2$, yielding the stated $M', s', T_{\mathrm{MoRE}}^*$.

\medskip\noindent\textbf{Limit.} The ratios are
\[
\frac{M'}{M^*} \;=\; \frac{h}{r}\left(1 - \frac{C_r h r}{\tau}\right), \qquad
\frac{T_{\mathrm{MoRE}}^*}{T_{\mathrm{MoE}}^*} \;=\; \frac{h}{r}\left(1 - \frac{C_r h r}{\tau}\right)^2.
\]
Since $\tau = C_e k h s + C_r M h > C_r M h \gg C_r h r$ when $M \gg r$, the correction $C_r h r / \tau \ll 1$, and both ratios converge to $h/r$.
\end{proof}

\section{Theoretical Results}
\label{appendix:theoretical_results}

Section~\ref{sec:theory} states the rank bounds under the normalization $\|u_i\|_2\leq 1$ and $\|R_2x\|_2\leq 1$. The proofs below are given in a slightly more general form with arbitrary norm bounds $B_u$ and $B_z$. In this form, it is convenient to measure margins in normalized units: a normalized margin $\eta$ corresponds to a raw logit margin $\eta B_uB_z$. Setting $B_u=B_z=1$ recovers the statements used in Section~\ref{sec:theory}.

We first prove the lower bound. The argument shows that realizing many distinct top-$k$ assignments forces many separated projected representations to fit inside an $r$-dimensional ball.

\begin{theorem}[Lower bound on required rank, general norm bounds]
\label{thm:lower-general}
Let $\mathcal{F} \subseteq \binom{[M]}{k}$ and let $B_u,B_z>0$. If a rank-$r$ router $R=R_1R_2$ with $\|u_i\|_2\leq B_u$ and $\|R_2x_S\|_2\leq B_z$ realizes every $S\in\mathcal{F}$ at margin at least $\eta B_uB_z$, then
\[
|\mathcal{F}| \leq (1+2/\eta)^r,
\qquad
r \geq \frac{\log|\mathcal{F}|}{\log(1+2/\eta)}.
\]
\end{theorem}

\begin{proof}
The proof is a volume packing argument: distinct assignments force their projections $z_S$ to be far apart in $\mathbb{R}^r$, and we count how many such well-separated points fit in a ball.

For each $S \in \mathcal{F}$, choose an input $x_S$ realizing $S$ with margin $\eta B_uB_z$, and write $z_S = R_2 x_S$. Take any two distinct $S, T \in \mathcal{F}$: since $|S| = |T| = k$, there exist $i \in S \setminus T$ and $j \in T \setminus S$. At $x_S$ expert $i$ wins over $j$, while at $x_T$ expert $j$ wins over $i$, so
\[
\langle u_i - u_j,\, z_S\rangle \;\geq\; \eta B_uB_z, \qquad \langle u_j - u_i,\, z_T\rangle \;\geq\; \eta B_uB_z.
\]
Adding the inequalities and applying Cauchy--Schwarz with $\|u_i - u_j\|_2 \leq 2B_u$ gives $\|z_S - z_T\|_2 \geq \eta B_z$.

Hence $\{z_S : S \in \mathcal{F}\}$ are $|\mathcal{F}|$ points pairwise separated by at least $\eta B_z$, all lying in the radius-$B_z$ ball of $\mathbb{R}^r$. The open balls of radius $\eta B_z/2$ around these points are therefore disjoint and contained in the ball of radius $B_z(1+\eta/2)$. Comparing volumes,
\[
|\mathcal{F}| \cdot \left(\frac{\eta B_z}{2}\right)^r \;\leq\; \left(B_z\left(1 + \frac{\eta}{2}\right)\right)^r,
\]
which rearranges to $|\mathcal{F}| \leq (1 + 2/\eta)^r$.
\end{proof}

\begin{proof}[Proof of Theorem~\ref{thm:lower}]
Apply Theorem~\ref{thm:lower-general} with $B_u=B_z=1$.
\end{proof}

We next prove the upper bound. The construction first realizes the desired assignments in dimension $M$ using the standard basis, then applies a Johnson--Lindenstrauss projection~\citep{johnson1984extensions} to reduce the dimension while preserving the inner products that determine the top-$k$ ordering.

\begin{theorem}[Upper bound on required rank, general norm bounds]
\label{thm:upper-general}
Fix $1 \leq k \leq M - 1$, $\mathcal{F} \subseteq \binom{[M]}{k}$, and norm bounds $B_u,B_z>0$. Set $\rho := \sqrt{1/k + 1/(M - k)} > 1/\sqrt{k}$. For any target normalized margin $\eta \in (0, \rho)$, $h \geq r$, and
\[
r \;\geq\; \frac{C \log(M + |\mathcal{F}|)}{(\rho - \eta)^2}
\]
with universal $C > 0$, there is a rank-$r$ router with $\|u_i\|_2 \leq B_u$ realizing every $S \in \mathcal{F}$ at some input with $\|R_2 x_S\|_2 \leq B_z$ and margin at least $\eta B_uB_z$.
\end{theorem}

\begin{proof}
We first solve the problem in the uncompressed regime $r = M$, then use a Johnson--Lindenstrauss projection to compress down to $r = O(\log(M + |\mathcal{F}|) / (\rho - \eta)^2)$ while preserving the inner products that determine routing.

To begin, let $e_1, \ldots, e_M$ be the standard basis of $\mathbb{R}^M$, which serve as expert vectors. For each $S \in \mathcal{F}$, define
\[
z_S \;:=\; \frac{1}{\rho}\!\left(\frac{\mathbf{1}_S}{k} - \frac{\mathbf{1}_{[M] \setminus S}}{M - k}\right).
\]
A direct computation gives
\[
\|z_S\|_2^2 \;=\; \frac{1}{\rho^2}\!\left(k \cdot \tfrac{1}{k^2} + (M - k) \cdot \tfrac{1}{(M-k)^2}\right) \;=\; \frac{1}{\rho^2}\!\left(\tfrac{1}{k} + \tfrac{1}{M-k}\right) \;=\; 1.
\]
For $i \in S$ and $j \notin S$,
\[
\langle e_i, z_S\rangle \;=\; \frac{1}{k\rho}, \qquad \langle e_j, z_S\rangle \;=\; -\frac{1}{(M - k)\rho}, \qquad \langle e_i, z_S\rangle - \langle e_j, z_S\rangle \;=\; \rho.
\]
So in dimension $M$, the trivial choice of expert vectors $e_i$ realizes every $S \in \mathcal{F}$ with margin exactly $\rho$, which is strictly larger than the target $\eta$.

We now compress this construction via Johnson--Lindenstrauss. Let $\mathcal{V} = \{e_1, \ldots, e_M\} \cup \{z_S : S \in \mathcal{F}\}$, so $|\mathcal{V}| \leq M + |\mathcal{F}|$. By the Johnson--Lindenstrauss lemma, for $r = O(\log|\mathcal{V}| / \varepsilon^2)$ there is a linear map $\Phi : \mathbb{R}^M \to \mathbb{R}^r$ preserving squared Euclidean distances on $\mathcal{V} \cup \{0\}$ multiplicatively to within $(1 \pm \varepsilon)$, i.e.\ $(1-\varepsilon)\|u - v\|_2^2 \leq \|\Phi u - \Phi v\|_2^2 \leq (1+\varepsilon)\|u - v\|_2^2$ for all $u, v \in \mathcal{V} \cup \{0\}$. Since every $u \in \mathcal{V}$ has unit norm, $\|u - v\|_2^2 \leq 4$ for all such pairs, so this is equivalent to additive preservation up to $O(\varepsilon)$; absorbing the constant, we assume preservation up to additive error $\varepsilon$ in what follows. The polarization identity
\[
\langle u, v\rangle \;=\; \tfrac{1}{2}\!\left(\|u\|_2^2 + \|v\|_2^2 - \|u - v\|_2^2\right)
\]
then preserves all pairwise inner products of vectors in $\mathcal{V}$ up to additive error $O(\varepsilon)$, and absorbing the constant we may assume preservation up to additive $\varepsilon$. Choose
\[
\varepsilon \;:=\; c_0\,(\rho - \eta)
\]
for a sufficiently small universal constant $c_0 > 0$, which gives $r = O\!\bigl(\log(M + |\mathcal{F}|) / (\rho - \eta)^2\bigr)$.

To extract the rank-$r$ router from $\Phi$, define normalized compressed vectors
\[
\tilde{u}_i \;:=\; \frac{\Phi e_i}{\|\Phi e_i\|_2}, \qquad \tilde{z}_S \;:=\; \frac{\Phi z_S}{\|\Phi z_S\|_2},
\]
and set $u_i = B_u\tilde{u}_i$, $\bar z_S = B_z\tilde{z}_S$. The norm bounds are immediate. Since normalization changes the preserved inner products by at most another constant multiple of $\varepsilon$, for every $i \in S$ and $j \notin S$,
\[
\langle u_i, \bar z_S\rangle - \langle u_j, \bar z_S\rangle
= B_uB_z\!\left(\langle \tilde{u}_i,\tilde{z}_S\rangle-\langle \tilde{u}_j,\tilde{z}_S\rangle\right)
\geq B_uB_z(\rho-C'\varepsilon)
\geq \eta B_uB_z
\]
for a universal constant $C' > 0$, by the choice of $c_0$. Hence every $S \in \mathcal{F}$ is realized with margin at least $\eta B_uB_z$.

Finally, we lift to input space: since $h \geq r$, set $R_2 = [I_r \mid 0] \in \mathbb{R}^{r \times h}$ and choose $x_S$ whose first $r$ coordinates are $\bar z_S$, so that $R_2 x_S = \bar z_S$. This realizes the desired router with the claimed parameters.
\end{proof}

\begin{proof}[Proof of Theorem~\ref{thm:upper}]
Apply Theorem~\ref{thm:upper-general} with $B_u=B_z=1$ and target margin $\eta=c/\sqrt{k}$ for a sufficiently small universal constant $c>0$. Since $\rho>1/\sqrt{k}$, we have $(\rho-\eta)^2=\Omega(1/k)$, and therefore $r=O(k\log(M+|\mathcal{F}|))$. The constructed margin is $\eta=\Omega(1/\sqrt{k})$, as claimed.
\end{proof}

\begin{proof}[Proof of Corollary~\ref{cor:full-expressivity}]
For $\mathcal{F} = \binom{[M]}{k}$, Stirling's bounds $(M/k)^k \leq |\mathcal{F}| \leq (eM/k)^k$ give
\[
\log |\mathcal{F}| \;=\; \Theta(k \log(M/k)).
\]
Since $M \leq |\mathcal{F}|$ for $1 \leq k \leq M-1$, we also have $\log(M+|\mathcal{F}|) = \Theta(\log |\mathcal{F}|) = \Theta(k \log(M/k))$.

\emph{Upper bound.} By Theorem~\ref{thm:upper}, any rank $r \geq C k \log(M+|\mathcal{F}|)$ admits the desired router at margin $\Omega(1/\sqrt{k})$. Substituting $\log(M+|\mathcal{F}|) = \Theta(k \log(M/k))$ yields rank $r = O(k^2 \log(M/k))$ as a sufficient condition. This margin is resolvable under $p = \Omega(\log k)$ precision.

\emph{Lower bound.} At $p$-bit top-$k$ precision, $\eta = \Omega(2^{-p})$ implies $\log(1 + 2/\eta) = O(p)$. Theorem~\ref{thm:lower} then gives
\[
r \;\geq\; \frac{\log |\mathcal{F}|}{\log(1 + 2/\eta)} \;=\; \Omega\!\left(\frac{k \log(M/k)}{p}\right).
\]
\end{proof}

\paragraph{Proposition~\ref{prop:gaussian-load-balancing} (restated).}
Let $x_1,\ldots,x_n \sim \mathcal{N}(0,I_h)$ be independent, and suppose $M=2^r$ ($r = \Theta(\log_2 M)$) with $h\geq r$. For any $1 \leq k \leq M$, there exists a rank-$r$ top-$k$ router with $M$ experts that selects every expert with probability exactly $k/M$ on each example. Consequently, if $N_a$ is the number of examples routed to expert $a$, then writing $L := \log\frac{2M}{\delta}$, for any $\delta \in (0,1)$ the following holds with probability at least $1-\delta$:
\[
\max_{a\in[M]}\left|N_a-\frac{nk}{M}\right|
\leq
\sqrt{\frac{2nkL}{M}} \;+\; \frac{2}{3}L .
\]

\paragraph{Construction.}
The construction follows the routing step in the Gaussian memorization argument of \citet{jelassi2025mixtureparrotsexpertsimprove}. Let $R_2 \in \mathbb{R}^{r \times h}$ select the first $r$ coordinates, so $z_i := R_2 x_i \sim \mathcal{N}(0,I_r)$ independently for $i \in [n]$. Index the $M=2^r$ experts by sign vectors $a \in \{\pm 1\}^r$ and set the corresponding expert vector to $u_a = a/\sqrt{r}$.

At $k=1$ the winner is available in closed form: the score $\langle u_a, z\rangle = \frac{1}{\sqrt r}\sum_{\ell} a_\ell z_\ell$ is maximized coordinate-wise by $a_\ell = \mathrm{sign}(z_\ell)$, so the selected expert is the sign pattern of $z$, which is uniform on $\{\pm1\}^r$. At $k>1$ no such closed form is available, since the identity of the runner-up experts depends on the magnitudes $|z_\ell|$ and not only on their signs. We therefore argue from a symmetry of the construction rather than from an explicit description of the selected set. The symmetry specializes to the sign-pattern argument at $k=1$, so the proof below also covers the top-$1$ case.

\begin{definition}[Score vector]
\label{def:score-vector}
The \emph{score vector} at a projected input $z \in \mathbb{R}^r$ is the $M$-dimensional vector $s(z) \in \mathbb{R}^M$ that collects the routing scores of all $M$ experts, with entries indexed by the experts themselves:
\begin{equation}
\label{eq:score-vector}
s(z)_a \;:=\; \langle u_a, z\rangle \;=\; \frac{1}{\sqrt r}\sum_{\ell=1}^{r} a_\ell z_\ell ,
\qquad a \in \{\pm1\}^r .
\end{equation}
The router selects the $k$ experts whose entries in $s(z)$ are largest.
\end{definition}

The key point is that under this construction the standard Gaussian has no special preference for one expert over another in expectation. Each expert is a vertex of the hypercube $\{\pm1\}^r$, and negating a subset of the input coordinates is the same as negating the corresponding entries of the expert index, because the two negations cancel in the inner product $\langle a, z\rangle$. Negating input coordinates therefore only relabels the experts. We capture this intuition using the reflection action.

\begin{definition}[Reflection action]
\label{def:reflection-action}
Let $\mathcal{G} := \{\pm1\}^r$, a group under coordinatewise multiplication $\odot$. It acts on $\mathbb{R}^r$ by $\sigma \cdot z := \sigma \odot z$ and on the expert set $\{\pm1\}^r$ by $\pi_\sigma(a) := \sigma \odot a$. For a score vector $s \in \mathbb{R}^M$, write $s \circ \pi_\sigma$ for the relabelled vector with entries $(s\circ\pi_\sigma)_a = s_{\sigma\odot a}$.
\end{definition}

The lemma below states the three properties of this action that the proof uses.

\begin{lemma}[Symmetry properties]
\label{lem:sign-symmetry}
For every $\sigma \in \mathcal{G}$:
\begin{enumerate}
\item[\textup{(G1)}] \emph{(Invariance of the input distribution.)} If $z \sim \mathcal{N}(0,I_r)$ then $\sigma \cdot z$ has the same distribution as $z$.
\item[\textup{(G2)}] \emph{(Equivariance of the router.)} $s(\sigma \cdot z) = s(z)\circ\pi_\sigma$ for all $z \in \mathbb{R}^r$. Reflecting the input relabels the score vector by $\pi_\sigma$ and changes nothing else.
\item[\textup{(G3)}] \emph{(Transitivity.)} The map $\sigma \mapsto \pi_\sigma$ is a group action of $\mathcal{G}$ on the expert set which is simply transitive: for all $a,b \in \{\pm1\}^r$ there is exactly one $\sigma \in \mathcal{G}$ with $\pi_\sigma(a) = b$, namely $\sigma = a \odot b$. Moreover $\pi_\sigma^{-1} = \pi_\sigma$.
\end{enumerate}
\end{lemma}

\begin{proof}
(G1) The coordinates of $z$ are independent and symmetric about $0$, so negating any subset of them preserves the joint distribution.

(G2) By \eqref{eq:score-vector}, $s(\sigma\cdot z)_a = \frac{1}{\sqrt r}\sum_\ell a_\ell\sigma_\ell z_\ell = \frac{1}{\sqrt r}\sum_\ell (\sigma\odot a)_\ell z_\ell = s(z)_{\sigma\odot a} = (s(z)\circ\pi_\sigma)_a$.

(G3) Since $\pi_\sigma\pi_\tau = \pi_{\sigma\odot\tau}$ and $\pi_{\mathbf 1} = \mathrm{id}$, the map is a group action. If $\pi_\sigma(a) = b$ then $\sigma = \sigma\odot a\odot a = b\odot a$, which is unique, and $\pi_{b\odot a}(a) = b\odot a\odot a = b$ confirms that it works. Finally $\pi_\sigma(\pi_\sigma(a)) = \sigma\odot\sigma\odot a = a$.
\end{proof}

Top-$k$ selection reads the score vector only through the ordering of its entries, never through which expert happens to carry which entry. Relabelling the experts therefore relabels the selected set in the same way.

\begin{lemma}[Top-$k$ selection is equivariant and almost surely well defined]
\label{lem:topk-equivariant}
Let $s \in \mathbb{R}^M$ be a score vector whose $M$ entries are pairwise distinct, and let $\mathrm{TopK}(s)$ be the set of the $k$ experts whose entries in $s$ are largest. Relabelling does not create equal entries, so $\mathrm{TopK}(s\circ\pi_\sigma)$ is defined for every $\sigma \in \mathcal{G}$, and the selection is \emph{$\mathcal{G}$-equivariant}:
\[
\mathrm{TopK}(s\circ\pi_\sigma) \;=\; \pi_\sigma\bigl(\mathrm{TopK}(s)\bigr) .
\]
Moreover, for $z \sim \mathcal{N}(0,I_r)$ the entries of $s(z)$ are almost surely pairwise distinct, so $\mathrm{TopK}(s(z))$ is almost surely well defined.
\end{lemma}

\begin{proof}
Relabelling the entries of a vector relabels its order statistics. The entry of $s\circ\pi$ at index $a$ is the entry of $s$ at index $\pi(a)$, so the indices carrying the $k$ largest entries of $s\circ\pi$ are exactly $\pi^{-1}$ applied to those for $s$. With $\pi = \pi_\sigma$ and $\pi_\sigma^{-1} = \pi_\sigma$ from (G3), this is the stated equivariance.

For distinctness, if $a \ne b$ then $u_a - u_b = (a-b)/\sqrt r \ne 0$, so $\{z : s(z)_a = s(z)_b\}$ is a hyperplane through the origin and therefore has Lebesgue measure zero. A union over the finitely many pairs is still null.
\end{proof}

\begin{lemma}[Uniform marginals]
\label{lem:uniform-marginals}
Let $z \sim \mathcal{N}(0,I_r)$ and let $1 \le k \le M$. Then
\[
\Pr\bigl[\,a \in \mathrm{TopK}(s(z))\,\bigr] \;=\; \frac{k}{M}
\qquad\text{for every expert } a \in \{\pm1\}^r .
\]
\end{lemma}

\begin{proof}
By Lemma~\ref{lem:topk-equivariant} we may work on the almost sure event that the entries of $s(z)$ are pairwise distinct. Fix $\sigma \in \mathcal{G}$ and an expert $a$. Using (G1) in the first step, (G2) in the second, and equivariance in the third,
\begin{align*}
\Pr[\,a \in \mathrm{TopK}(s(z))\,]
&= \Pr[\,a \in \mathrm{TopK}(s(\sigma\cdot z))\,]
= \Pr[\,a \in \mathrm{TopK}(s(z)\circ\pi_\sigma)\,] \\
&= \Pr[\,a \in \pi_\sigma(\mathrm{TopK}(s(z)))\,]
= \Pr[\,\pi_\sigma(a) \in \mathrm{TopK}(s(z))\,],
\end{align*}
where the last step uses $\pi_\sigma^{-1} = \pi_\sigma$ from (G3). By transitivity (G3), for any two experts $a$ and $b$ there is a $\sigma$ with $\pi_\sigma(a) = b$, so the function $a \mapsto \Pr[a \in \mathrm{TopK}(s(z))]$ takes the same value at every expert. Since $|\mathrm{TopK}(s(z))| = k$ always, we have $\sum_{a=1}^{M}\mathbf{1}\{a \in \mathrm{TopK}(s(z))\} = k$, and taking expectations gives $\sum_{a=1}^{M} \Pr[a \in \mathrm{TopK}(s(z))] = k$. A constant function on $M$ points summing to $k$ equals $k/M$ at each point.
\end{proof}

\begin{proof}[Proof of Proposition~\ref{prop:gaussian-load-balancing}]
Lemma~\ref{lem:uniform-marginals} gives selection probability $k/M$ for every expert. Fix an expert $a$. The indicators $\mathbf{1}\{a \text{ selected at } x_i\}$, $i \in [n]$, are independent because the $x_i$ are independent, and each is Bernoulli$(k/M)$. Hence $N_a \sim \mathrm{Binomial}(n,k/M)$ with mean $nk/M$ and variance $v := n\frac kM(1-\frac kM)$.

Bernstein's inequality with increments bounded by $1$ gives
\[
\Pr\!\left[\left|N_a - \frac{nk}{M}\right| \geq t\right] \;\leq\; 2\exp\!\left(-\frac{t^2}{2v+\frac23 t}\right).
\]
Setting the right-hand side to $\delta/M$ and solving $t^2 - \frac23 Lt - 2vL = 0$ yields $t = \frac L3 + \sqrt{\frac{L^2}{9}+2vL} \leq \sqrt{2vL} + \frac23 L$, using $\sqrt{x+y}\leq\sqrt x+\sqrt y$. Bounding $v = n\frac kM(1-\frac kM) \leq \frac{nk}{M}$ and taking a union bound over the $M$ experts gives the claimed bound.

This proposition isolates only the routing step. The dependence on $n$ enters through the number of examples assigned to each expert, approximately $nk/M$. The router rank only needs to distinguish the $M$ expert labels, which requires $r=\log_2 M$ in this construction, independently of $k$. The expert MLPs then memorize their assigned examples, as in the memorization upper bound of \citet{jelassi2025mixtureparrotsexpertsimprove}.
\end{proof}

\end{document}